\documentclass{article}

\usepackage{time_2024}
\usepackage{graphicx}
\usepackage{setspace}
\usepackage{multirow}
\usepackage{amsmath}
\usepackage{amsthm}
\usepackage{algpseudocode}
\usepackage{booktabs}
\usepackage{mathtools,amssymb}
\usepackage{subfigure}
\usepackage{booktabs} 
\usepackage{bm,bbm}
\usepackage{multirow}
\usepackage{wrapfig}
\usepackage{hyperref}       
\usepackage{url}            
\usepackage{threeparttable}
\algrenewcommand\algorithmicindent{0.5em}%

\usepackage[utf8]{inputenc} 
\usepackage[T1]{fontenc}    
\usepackage{hyperref}       
\usepackage{url}            
\usepackage{booktabs}       
\usepackage{amsfonts}       
\usepackage{nicefrac}       
\usepackage{microtype}      
\usepackage{xcolor}         
\usepackage{algpseudocode}
\usepackage{natbib}

\usepackage{booktabs}
\usepackage{bm}
\usepackage{makecell}
\usepackage{multirow}
\usepackage{amsmath}
\usepackage{mathrsfs}
\usepackage{graphicx}
\usepackage{natbib}
\usepackage{color}
\usepackage{threeparttable}
\usepackage{subfigure}
\usepackage[ruled,vlined]{algorithm2e}
\usepackage{amsthm}
\usepackage{amsfonts}
\usepackage{marvosym}

\newcommand{\tp}{Tokenization Paradigm }
\newcommand{\myformer}{WQ4TS }

\newcommand{\Rmnum}[1]{\expandafter\@slowromancap\romannumeral #1@}

\newtheorem{definition}{Definition}

\newtheorem{theorem}{Theorem}[section]
\newtheorem{lemma}[theorem]{Lemma}

\title{A Wave is Worth 100 Words: Investigating Cross-Domain Transferability in Time Series}

\author{%
  Xiangkai Ma\thanks{Equal Contribution}, Xiaobin Hong\footnotemark[1], Wenzhong Li\textsuperscript{\Letter}, Sanglu Lu \\
  State Key Laboratory for Novel Software Technology, Nanjing University, Nanjing, China \\
  \texttt{\small \{xiangkai.ma,xiaobinhong\}@smail.nju.edu.cn}\\ 
  \texttt{\small\{sanglu,lwz\}@nju.edu.cn}
}

\begin{document}

\maketitle

\begin{abstract}
Time series analysis is a fundamental data mining task that has made encouraging progress in many real-world scenarios. Supervised training methods based on empirical risk minimization have proven their effectiveness on specific tasks and datasets. However, the acquisition of well-annotated data is costly and a large amount of unlabeled series data is under-utilized. Due to distributional shifts across various domains and different patterns of interest across multiple tasks. The problem of cross-domain multi-task migration of time series remains a significant challenge.
To address these problems, this paper proposes a novel cross-domain pretraining method based on Wave Quantization (termed as WQ4TS), which can be combined with any advanced time series model and applied to multiple downstream tasks.
Specifically, we transfer the time series data from different domains into a common spectral latent space, and enable the model to learn the temporal pattern knowledge of different domains directly from the common space and utilize it for the inference of downstream tasks, thereby mitigating the challenge of heterogeneous cross-domains migration.
The establishment of spectral latent space brings at least three benefits, cross-domain migration capability thus adapting to zero- and few-shot scenarios without relying on priori knowledge of the dataset, general compatible cross-domain migration framework without changing the existing model structure, and robust modeling capability thus achieving SOTA results in multiple downstream tasks.
To demonstrate the effectiveness of the proposed approach, we conduct extensive experiments including three important tasks: forecasting, imputation, and classification. And three common real-world data scenarios are simulated: full-data, few-shot, and zero-shot. The proposed \myformer achieves the best performance on \textbf{87.5\%} of all tasks, and the average improvement of the metrics on all the tasks is up to \textbf{34.7\%}.
\end{abstract}
\section{Introduction} \label{section:introduction}

Time series analysis has broad real-world applications, including imputation of missing data \cite{Karmitsa2022MissingVI}, power consumption detection for production equipment \cite{Wang2022FewShotFA}, and weather forecasting \cite{Schultz2021CanDL}. Traditional end-to-end deep learning models have made significant progress in time series analysis~\cite{Zhou2023OneFA,Wang2022LearningLS,zhou2022film,seyfi2022generating,li2022generative,jeon2022gtgan}, with many popular deep neural network architectures being applied to time series modeling, such as Linear-based~\cite{Zeng2022AreTE,yi2023frequencydomain}, CNN-based~\cite{wu2023timesnet,wang2023micn}, RNN-based~\cite{Shi2015ConvolutionalLN}, Transformer-based~\cite{Nie2023ATS,zhou2022fedformer,haoyietal-informer-2021,Liu2022NonstationaryTR}, and GNN-based models~\cite{Wu2020ConnectingTD}. 
In addition, work on signal processing before the backbone based on the numerical characterization of the series has driven the development of time series forecasting, such as Seasonal-Trend Decomposition \cite{,zhou2022fedformer,wu2021autoformer,wang2023micn}, which improves the efficiency of backbone in representation extraction by capturing the complex patterns from original series in advance. 
Recently, we have witnessed the remarkable success of pre-trained foundation models \cite{Radford2018ImprovingLU,Radford2019LanguageMA,Brown2020LanguageMA,Touvron2023LLaMAOA,Touvron2023Llama2O} in Natural Language Processing (NLP), Computer Vision (CV), and Multimedia (MM). Large-scale Language Models (LLM) in particular have demonstrated impressive performance in various areas, thus many works try to build a large time series model on top of LLMs~\cite{Zhou2023OneFA,Cao2023TEMPOPG,Li2023FrozenLM,Xue2022PromptCastAN,Chang2023LLM4TSAP,Sun2023TESTTP,Jin2023TimeLLMTS,Liu2023UniTimeAL}.  

However, series with different domain information tend to be heterogeneous and their performance deteriorates rapidly when domain shifts occurs, making it difficult to deploy either specialized models trained for specific tasks and datasets \cite{Cai2022TimeSD,He2023DomainAF,Ragab2022ADATIMEAB,Jin2022DomainAF,Xu2022CDTransCT}, in complex and variable real-world scenarios. When the test set distribution and the training data are not the same, the models often fail to exhibit satisfactory inferential capability. Consider the following case, in the absence of sufficient training data, we would like the models to be pre-trained on heterogeneous datasets first, followed by a small amount of data fine-tuning in the target domain, or even inference directly on the target domain task. To overcome these challenges, there has been work looking to utilize limited data in the target domain to enable the model to adapt to the new target domain, which is referred to as the cross-domain adaptation technique.

For real-world time series analysis tasks, an inference model with strong domain adaptation capabilities is crucial. Recently, more attention has been paid to Time Series Cross Domain Migration \cite{Jin2022DomainAF,Ragab2022ADATIMEAB,He2023DomainAF,Cai2022TimeSD,Ragab2021SelfSupervisedAD}. however, existing approaches either require additional domain expertise or a unique design of the model structure to accommodate the challenges of domain migration, which makes it difficult to directly apply the existing domain adaptation research to SOTA models designed for specific tasks, which makes it challenging to deploy domain adaptive models in the real world. deployment in the real world remains challenging. Moreover, considering the impressive achievements of previous research for specific time series analysis tasks, we would like to design a general and compatible cross-domain migration framework that enables existing models to overcome the challenge of data distribution drift without changing the model structure and without relying on priori knowledge of the dataset. This is the starting point of this paper, i.e., \textbf{"a generic and compatible cross-domain migration framework"}.

Some recent studies \cite{Finder2022WaveletFM,Fang2023WhenSM,Zhang2023DiscoveringFB,Guo2023SelfSupervisedSB} applying spectral analysis techniques to time series have prompted us to think about the fact that time series in spectral space can be uniformly characterized by a set of filters with continuous frequencies, as opposed to the high degree of perturbation in the time domain space that prevails in the pattern of sequence changes, which encourages us to establish a shared spectral space and embed the time series data from different domain domains into this shared space, thus mitigating the time series domain drift phenomenon between data, and enabling existing models to overcome the challenge of data distribution migration drift without changing the model structure and without relying on priori knowledge of the dataset.

\begin{table*}[htpb]
  \caption{
  To verify the scalability of the existing models on different sampling rates (rate=10), we design two groups of supervised experiments which training on the \emph{Original} dataset (ETTh1) and the \emph{Resample} dataset (ETTh1), respectively, and show the MSE Loss calculated on the same test set. 
  }\label{tab:introduction_downsample}
  \centering
  \resizebox{0.6\columnwidth}{!}{
  \begin{small}
  \renewcommand{\multirowsetup}{\centering}
  \setlength{\tabcolsep}{1.pt}
  \renewcommand\arraystretch{1.}
  \begin{tabular}{cccccc}
    \toprule
    
    \multicolumn{1}{c}{\multirow{1}{*}{\scalebox{0.8}{Model}}} & 
    \multicolumn{1}{c}{\rotatebox{0}{\scalebox{0.8}{OneFitsAll}}} & 
    \multicolumn{1}{c}{\rotatebox{0}{\scalebox{0.8}{DLinear}}} & 
    \multicolumn{1}{c}{\rotatebox{0}{\scalebox{0.8}{PatchTST}}} & 
    \multicolumn{1}{c}{\rotatebox{0}{\scalebox{0.8}{FEDformer}}} &
    \multicolumn{1}{c}{\rotatebox{0}{\scalebox{0.8}{TimesNet}}} \\
        
    \hline
    \multirow{1}{*}{\rotatebox{0}{\scalebox{0.8}{Original Supervised}}} & 
    {\scalebox{0.78}{0.352}} & 
    {\scalebox{0.78}{0.357}} & 
    {\scalebox{0.78}{0.351}} & 
    {\scalebox{0.78}{0.448}} & 
    {\scalebox{0.78}{0.400}} \\
    
    \hline
    \multirow{1}{*}{\rotatebox{0}{\scalebox{0.8}{Resample Supervised}}} & 
    {\scalebox{0.78}{0.757}} & 
    {\scalebox{0.78}{0.722}} & 
    {\scalebox{0.78}{0.758}} & 
    {\scalebox{0.78}{0.739}} &
    {\scalebox{0.78}{1.149}} \\

    \hline
    \multirow{1}{*}{\rotatebox{0}{\scalebox{0.8}{MSE $\uparrow$ (\%)}}} & 
    {\scalebox{0.78}{115.1}} & 
    {\scalebox{0.78}{102.2}} & 
    {\scalebox{0.78}{116.0}} & 
    {\scalebox{0.78}{65.0}} & 
    {\scalebox{0.78}{187.3}} \\

    \bottomrule
  \end{tabular}
  \end{small}
}
\end{table*}

In time-domain space, the prevalent distributional drift in time series makes cross-domain migration more challenging, and existing studies tend to mitigate the heterogeneity between time series data from different domains by establishing trend-seasonal decomposition \cite{wu2021autoformer,zhou2022fedformer,wang2023micn} or multiscale decomposition \cite{Shabani2022ScaleformerIM}. However, the challenges of cross-domain migration are mainly caused by the differences between data domains:
(1) Firstly, time series data with diverse domain background knowledge tend to exhibit different characteristics, e.g., electricity consumption generally exhibits long-term trends, ECG signals have stable horizontal baselines, and oscillate frequently, and establishing multi-domain connectivity and complex sequence patterns among different datasets is difficult.
(2) Secondly, even in the same domain background, datasets from different sources may exhibit heterogeneous in their time scales, sampling intervals, periodic patterns, etc., which results in the same data instances can also face the challenge of cross-domain migration due to a priori features such as sampling rate, etc., e.g., Table \ref{tab:introduction_downsample} shows the experimental results.
(3) Thirdly, unlike natural language processed by human brain abstraction as a set of ordered values continuously sampled from the real world, the time series itself is a primitive and low-level data form that does not have any universally applicable inherent pattern, encoding the fluctuating changes into rich semantic representation is non-trivial.

\begin{figure*}[t]
\begin{center}
	\centerline{\includegraphics[width=1.0\columnwidth]{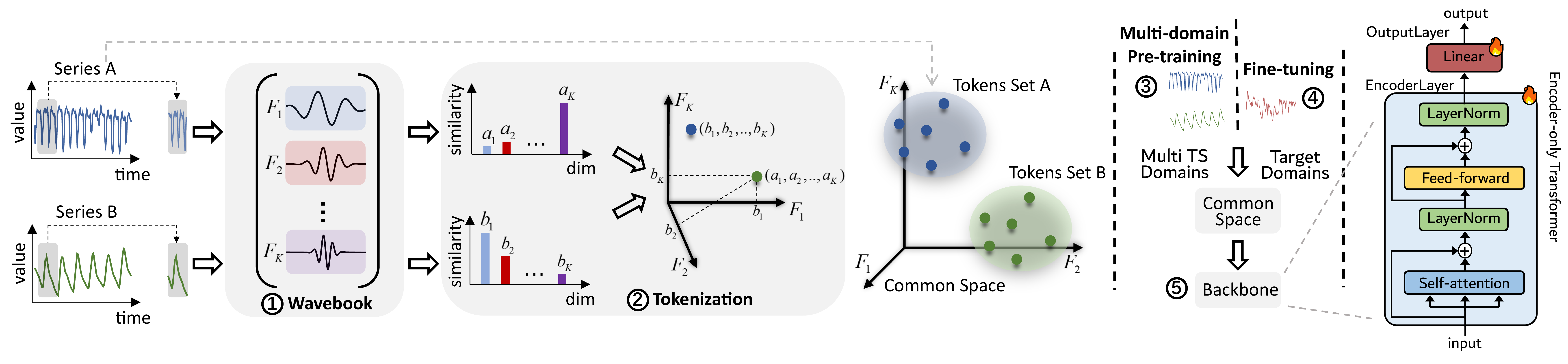}}
	\caption{
Illustration of the proposed \myformer architecture. 
\textcircled{1} The proposed \textbf{Wave Quantize Module} was utilized to establish the common spectral latent space.
\textcircled{2} Based on proposed \textbf{Tokenization Strategy}, the raw series is bijectively projected to the common spectral latent space. The distance between multiple TS domains is effectively reduced, which activates the generalization and migration capabilities of the model. Meanwhile, the generated tokens as fluctuation pattern similarities contain sufficient semantic information
\textcircled{3} Subsequently, the TS pattern knowledge in common spectral latent space will be utilized for multiple downstream tasks, by \textbf{Cross-domain Pre-training}.
\textcircled{4} Due to the Effective feature programming of the embedding technique, we design full parameters \textbf{Fine-tuning for Domain migration}.
\textcircled{5} We design the encoder-only transformer structure as the backbone, which contains a stack of EncoderLayers and OutputLayer.
}\label{fig:intro_2}
\end{center}
\end{figure*}

To address these challenges, we propose a novel framework to build a pre-trained model based on multi-domain time series data, which is illustrated in Fig.~\ref{fig:intro_2}. 
Firstly, we construct the \emph{wavebook} that consists of a set of orthogonal basis functions to form a common latent space from multiple domains.
Based on the wavebook, a \emph{tokenization} mechanism bijectively maps the input time series from diverse domains into tokens in the common spectral latent space following the ``wave as token'' principle, which mitigates the challenge of heterogeneous multiple domains. 
Subsequently, the comprehensive time series pattern knowledge of different domains will be learned by a backbone (i.e., a vanilla Transformer encoder) from the shared common embedding space with a \emph{multi-domain pre-training} process.
The latent representation was formed by calculating the inner product between the original series and each basis function of the wavebook, which implies semantic information of fluctuation pattern similarity of tokens. we can prove that the proposed wavebook and tokenization strategy is interpretable and each token has semantic information, as detailed in the section \ref{section:methods}.
For different TS domains, the above computations have no specific requirements for hyperparameters, thus addressing challenge of heterogeneous.
Therefore, each local segment corresponds to the vector of length $\lambda$, where $\lambda$ is the size of the wavebook. Since the token contains $\lambda$ item fluctuation pattern similarity as semantic information, thus addresses challenge of information pattern.

Since all tokens come from the unified common space, which alleviates the negative migration phenomenon and thus ensures that the model has a considerable advantage in the cross-domain migration phase.
In addition, the proposed \myformer reduces the representation distance between the target and source domains by projecting the raw series from the target domain into the common space, and the model also has the advantage of ``in-modality'', which encouraged us to design full parameters fine-tuning for domain adaptation.

The key contributions of our work are as follows.
\begin{itemize}
  \item This paper presents the concept of the Wave Quantize in time series for the first time, and analyzes the reasons for the slow progress of unsupervised cross-domain migration research in time series. Finally, it analyzes the differences between TS and NLP in terms of data paradigms, which serve as foundational work on establishing general compatible cross-domain migration framework and effective tokenization strategies.
  \item In this paper, we propose a novel \textbf{Wave Quantize} Strategy that advances research related to \textbf{\emph{cross-domain migration}} and common spectral latent space without changing the existing model structure and without relying on priori knowledge of the dataset.
  \item To verify that wave quantize strategy can mitigate negative migration in time series, the experiments consisted of single-domain training and multi-domain pre-training, and showed encouraging results: multi-domain pre-training would be far superior to single-domain pre-training.
  \item We performed $3$ tasks (forecasting, imputation, and classification) under $3$ data-limited scenarios (full data, few-shot, and zero-shot) are experimented to validate the effectiveness of the proposed approach. Specifically, for the forecasting and imputation tasks, extensive experiments on seven mainstream datasets; for the classification task, $35$ most representative datasets were chosen from the UCR dataset.
  In forecasting and imputation tasks, compared to the existing SOTA, the proposed \myformer achieves the best performance on \textbf{87.5\%} of all tasks, moreover, \myformer demonstrates \textbf{25.8\%} and \textbf{44.1\%} performance improvement in few-shot and zero-shot settings, respectively. In the few-shot classification task based on cross-domain migration, the proposed \myformer improves the average accuracy by \textbf{24.9\%} in all seven scenarios.
  These experiments demonstrate the potential of \myformer to be the general compatible cross-domain migration framework in time series.
\end{itemize}

The subsequent structure of this paper is as follows: Section \ref{section:related_work} summarizes the related work about the tokenization strategy, spectrum analysis and cross-domain migration on TS. Section \ref{section:methods} details the proposed wavebook and ``wave as token'' strategy. Section \ref{section:methods_5} shows the model architecture. Section \ref{section:experiments} shows the performance of forecasting, imputation, and classification tasks under three data-limited scenarios: full-data, few-shot, and zero-shot. Finally, experimental results and future research are discussed in Section \ref{section:conclusion}. 




\section{Related Work} \label{section:related_work}
We introduce the related works in terms of LLM for TS, attempts for the unified pre-trained Model, tokenization strategy in TS, and spectrum analysis in TS. 

\subsection{Cross Domain Migration in TS}
Deep neural networks trained on one domain can be poor at generalizing to another domain due to the issue of domain shift, that is, domain shift problem \cite{Wang2020BridgingPA,Zhao2020ARO,Zhang2020UnsupervisedMD,Oza2021UnsupervisedDA}. Domain adaptation (DA) methods attempt to mitigate the harmful effect of domain shift by aligning features extracted across source and target domains \cite{hong2020matchinggan,Wang2020cida}. Existing approaches mainly focus on classification tasks, where a classifier learns a mapping from a learned domain-invariant latent space to a fixed label space using source data. Consequently, the classifier depends only on common features across domains, and can be applied to the target domain \cite{Wilson2018ASO}. Unlike DA, This paper is devoted to design a novel framework to establish cross-domain connectivity between different TS datasets, enabling the knowledge learned by the backbone network from different complex sequence patterns to be migrated to multiple downstream tasks, which is known as Unsupervised Cross-domain Migration (UCM).

Recently, lots of research efforts \cite{Oza2021UnsupervisedDA,Vibashan2021MeGACDAMG} are devoted on unsupervised domain adaptation. This task aims to transfer knowledge learned from a labeled source domain to a different unlabeled target domain, and most approaches focus on aligning distributions of source and target domain and learning domain-invariant feature representations. There are mainly two levels for UDA methods: domain-level \cite{Bousmalis2017Unsupervised,Tzeng2017AdversarialDD} and category-level \cite{Kang2019ContrastiveAN,Du2021CrossDomainGD,Li2021CrossDomainAC}. Domain-level UDA mitigates the distribution divergence between the source and target domain by pulling them into the same distribution at different scale levels. Besides, some works \cite{Du2021CrossDomainGD,Li2021CrossDomainAC} focus on the fine-grained category-level label distribution alignment through an adversarial manner between the feature extractor and two domain-specific classifiers. Finally, One prevailing paradigm aims at minimizing statistical distribution measures to mitigate the distribution shift problem between the source and target domains \cite{Chen2019HoMMHM,Wang2020SelfSupervisedPA,Wu2022MultipleGA}.

In light of successes in related fields, domain adaptation techniques have been introduced to time series tasks \cite{Jin2022DomainAF,Ragab2022ADATIMEAB,He2023DomainAF,Cai2022TimeSD,Ragab2021SelfSupervisedAD}.
To generate accurate input pairs, CDTrans \cite{Xu2022CDTransCT} designs a two-way center-aware labeling algorithm to produce pseudo labels for target samples. Along with the pseudo labels, a weight-sharing triple-branch transformer framework is proposed to apply self-attention and cross-attention for source/target feature learning and source-target domain alignment, respectively.
Recent approach \cite{Jin2022DomainAF} proposes a shared-attention model with domain-adaptive capabilities that predicts future sequences using local representations over different time periods by extracting domain-invariant features and modeling domain-relevant attributes in conjunction with domain-specific features in order to appropriately approximate the data distribution of the respective domain.
AdaTime \cite{Ragab2022ADATIMEAB} develops a bench marking evaluation suite to systematically and fairly evaluate different domain adaptation methods on time series data. Specifically, we standardize the backbone neural network architectures and bench marking datasets, while also exploring more realistic model selection approaches that can work with no labeled data or just a few labeled samples. 
RainCoat \cite{He2023DomainAF} as a model for analyzing both closed-set and generalized-set data for complex time series, addresses feature and label shifts by considering both temporal and frequency features, aligning them across domains, and correcting for misalignments to facilitate the detection of private labels.
Existing model \cite{Cai2022TimeSD} designs intra- and inter-variable sparse attention mechanisms to extract correlation-structured time-series data considering time lags and utilize correlation-structured alignment to guide the transfer of knowledge from the source domain to the target domain.
SLARDA \cite{Ragab2021SelfSupervisedAD} designs a self-supervised learning module that utilizes forecasting as an auxiliary task to improve the transferability of the source features, and proposes a novel autoregressive domain adaptation technique that incorporates temporal dependency of both source and target features during domain alignment.
However, current DA frameworks are often only oriented to a single downstream task, such as classification DA \cite{Xu2022CDTransCT,He2023DomainAF} and prediction DA \cite{Jin2022DomainAF}. A general DA framework for a wide range of downstream tasks is encouraged to generally improve the cross-domain transfer ability of a wide range of SOTA models. Besides, unlike existing methods that require designing shared/shared feature types for different domains and choosing appropriate architectures for time series forecasting models, our approach provides the first end-to-end Unsupervised Crossdomain Migration (UCM) generalized framework for multiple downstream tasks.

\subsection{Tokenization strategy in TS}
Motivated by the successful application of transformers in NLP \cite{Vaswani2017AttentionIA} and pre-trained foundation models \cite{Radford2018ImprovingLU,Brown2020LanguageMA,Touvron2023LLaMAOA,Touvron2023Llama2O}, transformer-based models are viewed as equally promising in time series. 
Due to the characteristics of time series, initial approaches \cite{godfried2020flowdb,Kitaev2020ReformerTE,haoyietal-informer-2021,liu2021pyraformer,wu2021autoformer,woo2022etsformer,zhou2022fedformer} generally followed the \textbf{\emph{Point as Token}} design, that is each sampled time point acts as a separate token. 
There are two limitations of this strategy: firstly, computing global dependencies between any time points leads to high computational complexity; secondly, there is a serious information redundancy in the global dependencies captured by the attention mechanism, considering that the time series is continuously varying, which leads to two neighboring tokens having nearly the same numerical distribution. Therefore, approaches at the initial stage of time series analysis have focused on reducing computational complexity by mitigating information redundancy.

For instance, Reformer \cite{kitaev2020reformer} designed the locally sensitive hashing self-attention to reduce computational complexity. Informer \cite{haoyietal-informer-2021} proposes the ProbSparse self-attention mechanism to efficiently replace the canonical self-attention. Pyraformer \cite{liu2021pyraformer} introduces the pyramidal attention module which reduces computational complexity by constraining the maximum length of the signal traversing paths. Autoformer \cite{wu2021autoformer} designs the Auto-Correlation mechanism based on the series periodicity, which conducts the dependencies discovery and representation aggregation at the sub-series level. ETSformer \cite{woo2022etsformer} proposes the novel exponential smoothing attention and frequency attention to replace the self-attention mechanism in vanilla Transformers, thus improving both accuracy and efficiency. FEDformer \cite{zhou2022fedformer} avoids the high overhead of computation on the time domain by calculating the dependencies between individual bands in the frequency domain, while Fourier Transform has been used to ensure that individual bands have a global view.

Subsequently, the representative PatchTST \cite{Nie2023ATS} proposed a \textbf{\emph{Patch as Token}} strategy and channel-independent design to demonstrate the effectiveness of the transformer in time series. 
Based on this, the recent iTransformer \cite{liu2023itransformer} proposes the \textbf{\emph{Series as Token}} strategy, that is the whole series is regarded as the token, and the attention mechanism is used to capture the dependencies between the sequences of different channels, which is surprisingly intuitive and effective in data domains with a very large number of channels and complex dependency relationships between them. 
The above studies have shown that the reasonable tokenization strategy can drastically improve the performance of Transformer-based models compared to the complex and tedious model structure improvement. 
However, the existing tokenization strategy cannot be well applied to cross-domain migration. Since TS datasets from different domains and sampling settings exhibit diverse periodic patterns, it is difficult to identify a unified sub-series span that matches all TS data domains. existing tokenization strategies are unable to establish a unified framework to adapt to the potential TS data domains, and the pre-training phase is unable to adequately learn cross-domain representational information, which limits the generalization ability and scalability of the model. 
Therefore, The ideal tokenization should be insensitive to the mathematical characteristics of different data domains, which motivates us to propose \textbf{\emph{Wave as Token}} as the general compatible cross-domain migration strategy.

\subsection{LLM for TS}
The adaptation of pre-trained LLMS for time series analysis has attracted attention, exploiting their superior capabilities in sequence representation learning.
OneFitsAll \cite{Zhou2023OneFA} first attempted to apply the pre-trained GPT2 \cite{Radford2019LanguageMA} to the time series downstream tasks, and achieved comparable performance to the state-of-the-art methods by freezing the pre-trained self-attention and feed forward structures to maximize the retention of pre-training information, and only fine-tuning the layer norm. 
In contrast, the recent TimeLLM \cite{Jin2023TimeLLMTS} completely freezes the model parameters of the pre-trained Llama \cite{Touvron2023LLaMAOA}, aligning both time series and natural language modalities by introducing textual prototypes. The task settings and dataset priori information are fed into the pre-trained Llama as hard-prompt by the pre-trained Embedder, thus fully activating the inference capability of the foundation model on the time series forecasting task. 
TEMPO \cite{Cao2023TEMPOPG} designs a prompt pool based on seasonal-trend decomposition to generate specific prompts for each sub-component, and incorporates LoRA \cite{hu2022lora} to achieve efficient fine-tuning of LLM. 
TEST \cite{Sun2023TESTTP} establishes a TS embedding method applicable to LLMs by using orthogonal text embedding vectors as prototypes to constrain the TS embedding space, thus activating the feature extraction capability of LLMs in the time series data domain. 
LLM4TS \cite{Chang2023LLM4TSAP} proposes a two-stage fine-tuning strategy, which firstly enables supervised fine-tuning of LLMs in the time series modality, and subsequently suggests downstream fine-tuning of LLMs in specific tasks, which unleash the flexibility of pre-trained LLMs.

However, existing approaches have not yet realized and attempted to address the challenge of negative migration in cross-domain migration, where UniTime \cite{Liu2023UniTimeAL} has proposed domain instructions to ensure the model recognizes the differences between multiple data domains, yet the problem of negative migration \cite{Hu2019StrategiesFP} still inevitably arises. 
These difficulties motivate us to propose a novel tokenization strategy to establish potential connections between multiple data domains, thus enabling cross-domain migration capability with the original transformer architecture.

\subsection{Spectrum Analysis in TS}
Over the past two decades, spectrum analysis techniques based on Fast Fourier Transform (FFT) and Wavelet Transform (WT) has been widely used in diverse model structures for time series analysis \cite{Finder2022WaveletFM,Fang2023WhenSM,Zhang2023DiscoveringFB,Guo2023SelfSupervisedSB} to improve the performance of learning directly from the time domain. Specifically, transformer-based models use spectral analysis to reduce the complexity of self-attention mechanisms, e.g., Autoformer \cite{wu2021autoformer} captures the periodic pattern information from original series based on FFT and establishes the Auto-Correlation mechanism at the sub-series level for learning dependencies and representation aggregation. FEDformer \cite{zhou2022fedformer} generates a set of mixed frequency components by Fourier analysis and designs a frequency enhanced attention mechanism, which exploited the sparse representation of the spectrogram and achieved linear complexity. Besides, the MLP-based FreTS \cite{yi2023frequencydomain} captures a complete global view of the original signal by operating on the spectral components obtained by frequency-transformation, and overcomes the information bottleneck of MLP on time series by ensuring that MLP focuses on the key part frequency components. In addition, the FourierGNN \cite{yi2023fouriergnn} proposes a novel architecture that uniformly captures inter-series (spatial) dynamics and intra-series (temporal) dependencies by performing matrix multiplication in Fourier space. Finally, researchers in representation learning have noted the significance of establishing consistency constraints in the temporal-frequency space. For example, TF-C \cite{zhang2022self} designs a self-supervised pre-training strategy based on time-frequency consistency, that is temporal- and frequency- representations learned from the same TS samples should be closer in the temporal-frequency space than different TS samples. In addition, CoST \cite{woo2022cost} establishes a more effective connection by mapping the embedded features in the time domain space to the frequency domain.
Based on the advantages of descriptive spectrum analysis, this paper designs the Wave Quantize tokenization strategy.


\section{Methodology} \label{section:methods}
Establishing the generic and compatible cross-domain migration framework between different time series domains has the following challenges: (1) time series data with diverse domain background knowledge tend to exhibit different characteristics; (2) Datasets from different sources may exhibit variations, even in the same domain background; (3) The same data instances can also face the challenge of cross-domain migration due to a priori features such as sampling rate.

The proposed \textbf{\emph{Wave Quantize}} Module solves challenge-1 through the designed common spectral latent space. In addition, the designed strategy ensures that arbitrary series data can be embedded to equal-length groups of tokens (number of tokens equal to series length), thus solving challenge-2. This ensures that no hyperparameter setting tricks are utilized in the model to adapt to potential data domains, even if these domains are completely different from each other in terms of structural features (e.g., channel number, sequence length) are completely different. Finally, we solve challenge-3 by ensuring that each token contains TS pattern information within a localized window through the finite-length basis functions and bridges the difference in the distribution of pattern information from different domains within the $\lambda$-dimensional space formed by basis functions, where $\lambda$ is the size of the proposed wavebook.

Specifically, the basic idea is illustrated in Fig.~\ref{fig:intro_2}. Motivated by VQVAE \cite{Oord2017NeuralDR}, we design a fine-grained tokenization strategy to standardize the original series before the backbone, project input series from diverse domains into a common latent space, and generate a set of tokens. Subsequently, the TS pattern knowledge of different domains will be learned directly by backbone from the common spectral latent space, and used for multiple downstream tasks. Specifically, a set of orthogonal basis functions is designed to form the latent space, where each basis function $\begin{Bmatrix}A_i|1\le i\le \lambda\end{Bmatrix}$ is a finite-length waveform with attenuation, thus ensuring that the energies of the basis functions are confined to a local window. We define this set of orthogonal basis functions as the \emph{"Wavebook"} in Fig.~\ref{fig:intro_2}\textcircled{1}. We perform the inner product operation between the original series and basis function by continuously sliding the window with the step equal to $1$, thus calculating the fluctuation pattern similarity between the basis function and each segment of the input series in the local window. Repeating for all $\lambda$ basis functions so that each segment of the local window corresponds to a feature vector of length $\lambda$. In addition, considering the continuity of the basis functions, when the window length is set as small as possible, each time point is embedded to a vector of length $\lambda$, which contains the pattern information in $\lambda$ dimensions for the local window in which the time point is located, as shown in Fig.~\ref{fig:intro_2}\textcircled{2}.
Moreover, when a fixed generating function is chosen and a set of basis functions with different frequencies is obtained by resampling the generating function continuously. Then, the vector will represent the energy distribution of the corresponding time point on $\lambda$ frequency bands. This makes the proposed tokenization strategy interpretable and each token has semantic information.

\subsection{Framework Overview}  \label{section:methods_2}
We propose a framework called \myformer to build a cross-domain migration time series general model, which is illustrated in Fig.~\ref{fig:intro_2}.
It consists of four parts: wave quantize module, time series tokenization, cross-domain pre-training, and fine-tuning for domain migration.
Wave quantize module construction extracts a set of orthogonal basis functions from multiple TS domains to form the common spectral latent space.
Based on the wavebook, a tokenization mechanism projects the input time series from diverse domains into tokens in the common spectral latent space to mitigate their heterogeneity. 
Subsequently, in the cross-domain pre-training phase, the comprehensive time series pattern knowledge of different domains is extracted from the tokens to train a backbone Transformer encoder to form a time series unified framework.

To guarantee that the model can simultaneously learn latent representations from diverse TS data domains with different statistical features and temporal pattern, we adopt two promising designs: 
1) Encoder-only design. Since TS data often exhibit the complexity of multiple patterns superimposed \cite{wang2023micn,Cao2023TEMPOPG}, we adopt an encoder-only design with excellent generalization capabilities for representation learning, which has been proven competent by the SOTA Transformer-based models~\cite{Nie2023ATS,liu2023itransformer}.
2) channel independent. The multivariate time series sample with $n$ channels was regarded as $n$ separate univariate series,  which is utilized as the exemplary case to simplify the methodology and was shown effective in the literature~\cite{Zhou2023OneFA,Jin2023TimeLLMTS}.
Besides, the specific description of the important symbols involved in the method is shown in Table~\ref{tab:description}.

The overview of the proposed \myformer architecture is illustrated in Fig.~\ref{fig:intro_2}. The time series from multiple TS domains is represented as $X=(x_{1},...,x_{l})\in\mathbb{R}^{l}$ with $l$ time steps, and \myformer is utlized to predict future series $Y=(x_{l+1},...,x_{l+c})\in\mathbb{R}^{c}$ with $c$ time steps. The common space $V^{\lambda}$ will be formed by the proposed $\lambda$-dimensional \textbf{Wave Quantize Module}, which comprises a set of orthogonal basis functions. Subsequently, in \textbf{Time Series Tokenization}, according to proposition \ref{prop:1}, any sub-series from the original series can be transformed into a group of coordinates in $V^{\lambda}$, and the bijection relation is satisfied between the sub-series and the coordinates. Furthermore, proposition \ref{prop:2} gives the sufficient-necessary condition and construction method for orthogonal wavelet bases. In \textbf{Cross-domain Pre-training} phase, the sub-series at the timestep-$j$ are converted to $token_{j}=(p_{1,j},p_{2,j},...,p_{\lambda,j})\in\mathbb{R}^{\lambda}$, from diverse TS domains. Subsequently, simple full-parameter \textbf{Fine-tuning for Domain migration} is designed since the common spectral latent space reduces the distance between the source and target domains.

\begin{table*}[htbp]
  \caption{The specific description of the symbols involved in the method. (Due to space constraints, only the symbols appearing in Eqs.\eqref{eq:0}-\eqref{eq:3} are shown in this table, considering that symbols in the section \ref{section:methods_3} are not necessary for understanding the framework process.)}\label{tab:description}
  \vskip 0.05in
  \centering
  \resizebox{0.8\columnwidth}{!}{
  \begin{threeparttable}
  \begin{small}
  \renewcommand{\multirowsetup}{\centering}
  \setlength{\tabcolsep}{3.8pt}
  \begin{tabular}{c|c}
  \toprule
  \toprule
    Symbol formula & Definition \\
    \toprule


    $X\in\mathbb{R}^{l}$ & the history series, where $l$ as the lookback length \\
    \midrule
    $Y,\bar{Y}\in\mathbb{R}^{c}$ & the future and prediction series, where $c$ is the forecast length \\
    \midrule
    
    $m$ & the precision of the discrete sequence describing the information in $A$ \\
    \midrule
    $A\in\mathbb{R}^{2^m}$ & the amplitude sequence, which discretely inscribe the orthogonal wavelet \\
    \midrule
    $f_c$ & the central frequency of the orthogonal wavelet \\
    \midrule
    $\lambda$ & the size of the wavebook \\
    \midrule
    $S_i\in\mathbb{R},i\in[1,..,\lambda]$ & the set of scale factors utilized to generate the wavebook \\
    \midrule
    $W_i\in\mathbb{R}^{m\cdot s_i}$ & the downsampling coordinate \\
    \midrule
    $A_i=A[W_i]\in\mathbb{R}^{m\cdot s_i}$ & the basis function in wavebook \\
    \midrule
    
    $V^{\lambda}$ & the shared embedding space, where $\lambda$ is the size of wavebook \\
    \midrule
    $p_{i,j}$ & the pattern similarity between the segment of $X$ at timestep$-j$ and $A_i$ \\
    \midrule
    $\lambda,l$ & the dimension and number of embedded token vector \\
    \midrule
    $L$ & the number of stacked of EncoderLayers in Encoder \\
    \midrule
    $token_j\in\mathbb{R}^\lambda $ & the projection coordinates in the $\lambda$-dimensional common space $V^{\lambda}$ \\
    \midrule
    $T_{pos}\in\mathbb{R}^{\lambda\times l}$ & the learnable position encoding \\
    \midrule
    $T^0,T^{L}\in\mathbb{R}^{\lambda\times l}$ & time series tokens that utilized as the input and output to the Encoder \\
    \midrule
    
    $d_k$ & the dimension of latent space \\
    \midrule
    $T_{pos}\in\mathbb{R}^{\lambda\times l}$ & the learnable position encoding tensor \\
    \midrule
    $W^Q,W^K\in\mathbb{R}^{\lambda\times d_k},W^V\in\mathbb{R}^{\lambda\times\lambda}$ & the trainable Linear Layer of self-attention \\
    \midrule
    $Q_L,K_L\in\mathbb{R}^{l\times d_k},V_L\in\mathbb{R}^{l\times\lambda}$ & Query, Key and Value latent-variable \\
    
    \bottomrule
    \end{tabular}
    \end{small}
  \end{threeparttable}
  }
\end{table*}

\subsection{Wave Quantize Module} \label{section:methods_3}
This section presents the theoretical foundations of the wave quantize module, which is the fine-grained tokenization strategy that makes the establishment of the general compatible cross-domain migration framework possible. 
Among it, proposition \ref{prop:1} theoretically proves the bijective projection relation established by the wavebook, and proposition \ref{prop:2} proposes the sufficient-necessary condition for the construction of the wavebook. 
According to these propositions, we construct the wavebook and eventually form the common spectral latent space. For all propositions described in the current section, the thorough proof procedure and the background of the wavebook will be provided.

We introduce the following notations: $F(t)$ is the specific basis function, $H(\omega)$ and $G(\omega)$ refer as the low-pass filter and band-pass filter of $F(t)$, respectively.

\begin{definition} \label{def:1} 
We define the wavebook as a set consisting of several basis functions, as $\{F_{j,k}(t)=2^{j/2}F(2^jt-k),(j,k)\in\mathbb{Z}^2\}$, which constitutes a set of standard orthonormal basis (O.N.B) and forms the finite series space $L^2(\mathbb{R})$ if and only if $F(t)$ is the orthogonal wavelet.
\end{definition}

\begin{definition} \label{prop:1}
Based on the proposed wavebook, there must exist the unique coefficients sequence $\{c_{j,k};(j,k)\in\mathbb{Z}^{2}\}\in l^{2}\left(\mathbb{Z}\right)$. 
For $\forall f(t)\in L^2(\mathbb{R})$, we have $f(t)=\sum_{j\in\mathbb{Z}}\sum_{k\in\mathbb{Z}}c_{j,k}\cdot F_{j,k}(t)$, where $f(t)$ denotes the univariate time series and $\{c_{j,k}\}$ contains the coordinates of $f(t)$ in space $\mathbb{Z}^2$. 
Specifically, the bijection is satisfied between variables $L^2(\mathbb{R})$ and $l^2(\mathbb{Z})$ in two spaces $f(t)$ and $\{c_{j,k}\}$.
\end{definition}

\begin{proof}[\textbf{Proof of Proposition \ref{prop:1}}]
Based on the Definition \ref{def:1}, we have $f(t)=\sum_{j\in\mathbb{Z}}\sum_{k\in\mathbb{Z}}c_{j,k}\cdot F_{j,k}(t)$, where $\left\{F_{j,k}(t), (j,k)\in\mathbb{Z}^2\right\}$ constitutes a standard orthonormal basis for $L^2(\mathbb{R})$, and $\{c_{j,k}\}$ represents coefficients. Besides, since $\{F_{j,k}(t),(j,k)\in\mathbb{Z}^{2}\}$ constitutes a set of orthogonal bases satisfying $\forall(j_1,k_1)\neq(j_2,k_2)$, there has $\left<F_{j_1,k_1},F_{j_2,k_2}\right>=0$.
Further, we have
\begin{equation}
    \begin{split}
    \left\langle f(t),F_{m,n}(t)\right\rangle&=\left\langle\sum_{j\in\mathbb{Z}}\sum_{k\in\mathbb{Z}}c_{j,k}\cdot F_{j,k}(t),F_{m,n}(t)\right\rangle, \\
    &=c_{m,n}\cdot\left\langle F_{m,n}(t),F_{m,n}(t)\right\rangle, \\
    &=c_{m,n}\cdot\left|F_{m,n}(t)\right|^2.
    \end{split}
\end{equation}
Thus, the coefficients sequence can be uniquely determined by $c_{m,n}=\left\langle f(t),F_{m,n}(t)\right\rangle/\left|F_{m,n}(t)\right|^2$, and the bijection relation is satisfied between the variables $f(t)$ and $\{h_n;n\in\mathbb{Z}\}$, which completes the proof of the Proposition \ref{prop:1}.
\end{proof}

\begin{definition} \label{prop:2} 
Let $\left.M(\omega)=\left(\begin{matrix}H(\omega)&G(\omega)\\H(\omega+\pi)&G(\omega+\pi)\end{matrix}\right.\right)$,
and the matrix $M^{H}(\omega)$ is the conjugate transpose matrix of $M(\omega)$.
The sufficient-necessary condition for $F(t)$ as the orthogonal wavelet is that $M(\omega)$ is the Unitary Matrix: $M^H(\omega)M(\omega)=M(\omega)M^H(\omega)=I,a.e.\omega\in\mathbb{R}$.
\end{definition}

\begin{proof}[\textbf{Proof of Proposition \ref{prop:2}}]
Firstly, we introduce and briefly prove the Lemma on the sufficiently-necessary condition for orthonormal system: Defining the function $F(t)\in L^2(\mathbb{R})$, then the set $\{F_{0,n}=2^{0/2}F(2^0t-n)=F(t-n)\}$ forms the orthonormal system of $L^2(\mathbb{R})$, that is $\left\langle F(t-n),F(t-l)\right\rangle=\delta(n-l)$ is the sufficiently-necessary for $\sum_{k\in\mathbb{Z}}\left|\hat{{F}}(\omega+2k\pi)\right|^2=1$, where $\delta(n-l)$ represents $\frac1{2\pi}\int_0^{2\pi}e^{-i(n-l)\omega}d\omega$. Lemma is proved due to
\begin{small}
\begin{equation}
    \begin{split}
    \left\langle F(t-n),F(t-l)\right\rangle&=\frac1{2\pi}\int_{\mathbb{R}}\hat{{F}}(\omega)e^{-in\omega}\cdot\overline{\left(\hat{{F}}(\omega)e^{-il\omega}\right)}d\omega, \\
    &=\frac1{2\pi}\int_0^{2\pi}\sum_{k\in\mathbb{Z}}\left|\hat{{F}}(\omega+2k\pi)\right|^2 e^{-i(n-l)\omega}d\omega.
    \end{split}
\end{equation}
\end{small}

In the Shannon Sampling Theorem, for any signal $f(t)$ defined on $L^2(\mathbb{R})$, if the frequency domain form $\hat{f}(\omega)$ of that signal has a truncation frequency $B$, then that signal can be reconstructed by equally spaced discrete sampling. This sampling interval can be at most $\pi/B$. 
If the function $f(t)$ satisfies the following conditions {$\forall f(t)\in L^2(\mathbb{R}), \hat{{f}}(\omega)=\int_{-\infty}^{+\infty}f(t)e^{-i\omega t}dt,\left|\omega\right|>B$}, we have:
\begin{equation}
    f(t)=\sum_{n\in\mathbb{Z}}f(n\Delta)\frac{sin(\pi/\Delta)(t-n\Delta)}{(\pi/\Delta)(t-n\Delta)}, 0<\Delta\leq\frac\pi B.
\end{equation}

Considering first the case $B=\pi$, and the scale function is defined as $\phi(t)=\frac{\sin(\pi t)}{\pi t}$, we have $f(t)=\sum_{n\in\mathbb{Z}}f(n)\phi(t-n)$.
Further, defining the space $V_j=\{f(t);\hat{{f}}(\omega)=0,|\omega|>2^j\pi\}$ with truncation frequency $B=2^j\pi$ and taking the sampling interval $\Delta=2^{-j}$, we have $f(t)=\sum_{n\in\mathbb{Z}}2^{-j/2}f(2^{-j}n)\phi_{j,n}(t)$, where the scale function $\phi_{j,n}$ as: 
\begin{equation}
    \phi_{j,n}(t)=2^{j/2}\phi(2^jt-n)=\frac{sin\pi(2^jt-n)}{\pi(2^jt-n)}.
\end{equation}
Based on the wavelet function $F_{j,k}$, we define the close-span space $W_j=closespan\{F_{j,k}(t)\textit{=}2^{j/2}F(2^jt-k),(j,k)\in\mathbb{Z}^2\}$, and the close-span space has the following characteristic: 1) Spatial orthogonality $W_j\perp W_{j+1}$; 2) Spatial approximability $L^2(\mathbb{R})=\bigoplus_{j=-\infty}^{+\infty}W_j$.
if $F(t-k)$ is the set of the orthonormal basis of $W_{0}$, then for any $j$ there has the orthonormal basis $\{2^{j/2}F(2^jt-k),(j,k)\in\mathbb{Z}^2\}$ of $W_{j}$. 
Moreover, it is easy to verify: $W_j\perp V_j,V_{j+1}=W_j\oplus V_j$, the construction of orthogonal wavelets is equivalent to finding a set of standard orthogonal bases for $W_{0}$.

Since the scale function $\phi(t)\subseteq V_1$, and there exists a set of orthonormal basis $\{\sqrt{2}\phi(2t-n);n\in\mathbb{Z}\}$ for $V_{1}$, there must exist a unique sequence of coefficients $\{h_n;n\in\mathbb{Z}\}$ such that $\phi(t)=\sqrt{2}\sum_{n\in\mathbb{Z}}h_n\phi(2t-n)$, which is regarded as the scale equation, and since $\phi(2t-n)$ is mutually orthogonal to each other for different $n$, the coefficients are calculated as follows $h_n=\left\langle\phi(t),\sqrt{2}\phi(2t-n)\right\rangle=\sqrt{2}\int_R\phi(t)\overline{\phi}(2t-n)dt$.

In addition, the scale equation are converted to frequency domain form by Fourier transforms
\begin{equation}\label{eqn:H}
    \hat{{\phi}}(\omega)=H(\omega/2)\hat{{\phi}}(\omega/2),~~~H(\omega)=\frac{1}{\sqrt{2}}\sum_{n\in\mathbb{Z}}h_ne^{-in\omega},
\end{equation}
where $H(\omega)$ is referred to as the low-pass filter and hence $h_{n}$ is also referred to as the low-pass filter coefficients.

For the wavelet function $F(t)\subseteq V_1$, there exists $\{g_{n};n\in\mathbb{Z}\}$ such that $F(t)=\sqrt2\sum_{n\in\mathbb{Z}}g_n\phi(2t-n)$, which is regarded as the wavelet equation, and since $\phi(2t-n)$ is orthogonal for different $n$, the coefficients are calculated as $g_n=\left\langle F(t),\sqrt{2}\phi(2t-n)\right\rangle=\sqrt{2}\int_R F(t)\overline{\phi}(2t-n)dt$.

In addition, the wavelet equations can be obtained in frequency domain form by Fourier transformation
\begin{equation}\label{eqn:G}
    \hat{{F}(\omega)}=G(\omega/2)\hat{{\phi}}(\omega/2),~~~G(\omega)=\frac{1}{\sqrt{2}}\sum_{n\in\mathbb{Z}}g_ne^{-in\omega},
\end{equation}
where $G(\omega)$ is referred to as the bandpass filter and $g_n$ is also referred to as the impulse response coefficient.

We define the functions $H(\omega)$ and $G(\omega)$ refer as the low-pass filter and band-pass filter based on $F(t)\in L^2(\mathbb{R})$, which determined by \eqref{eqn:H} and \eqref{eqn:G}, respectively, and introduce the matrix $M(\omega)$.

The function group $\{F_{0,k}=2^{0/2}F(2^0t-k)=F(t-k),k\in\mathbb{Z}\}$ forms the orthonormal basis of $W_{0}$, that is, the sufficient-necessary condition for $F(t)$ as the orthogonal wavelet is that $M(\omega)$ is the Unitary Matrix: $M^H(\omega)M(\omega)=M(\omega)M^H(\omega)=I,a.e.\omega\in\mathbb{R}$, where $M^{H}(\omega)$ is defined to be the conjugate transpose matrix of $M(\omega)$.

By the definition of the Unitary Matrix, $M(\omega)$ is the Unitary Matrix equivalent to
\begin{equation}\label{unimatrix:1}
    \begin{split}
    \left|H(\omega)\right|^2+\left|H(\omega+\pi)\right|^2=1,a.e.\omega\in\mathbb{R}, \\
    \begin{vmatrix}G(\omega)\end{vmatrix}^2+\begin{vmatrix}G(\omega+\pi)\end{vmatrix}^2=1,a.e.\omega\in\mathbb{R}, \\
    H(\omega)\overline{G}(\omega)+H(\omega+\pi)\overline{G}(\omega+\pi)=0,a.e.\omega\in\mathbb{R}.
    \end{split}
\end{equation}

We define the spaces $V_{0}$ and $W_{0}$ as follows
\begin{equation}
    \begin{split}
    V_0=closespan\left\{\phi(t-n),n\in\mathbb{Z}\right\}, \\
    W_0=closespan\left\{F(t-n),n\in\mathbb{Z}\right\}.
    \end{split}
\end{equation}

By the definitions of $H(\omega)$ and $G(\omega)$, Equation \ref{unimatrix:1} is equal to
\begin{equation}
    \begin{split}
    \sum_{k\in\mathbb{Z}}\left|\hat{{\phi}}(\omega+2k\pi)\right|^2=1,a.e.\omega\in\mathbb{R}, \\
    \sum_{k\in\mathbb{Z}}\left|\hat{{F}}(\omega+2k\pi)\right|^2=1,a.e.\omega\in\mathbb{R}, \\
    V_0\perp W_0,a.e.\omega\in\mathbb{R}.
    \end{split}
\end{equation}

Because of Lemma, the first two conditions are equivalent to
\begin{equation}
    \begin{split}
    \left\langle\phi(t-n),\phi(t-l)\right\rangle=\delta(n-l), \\
    \left\langle F(t-n),F(t-l)\right\rangle=\delta(n-l).
    \end{split}
\end{equation}

In summary, the sufficient-necessary condition for $M(\omega)$ as the Unitary Matrix is that $\phi(t-n)$ and $F(t-n)$ form the standard orthogonal system of $L^2(\mathbb{R})$, respectively, and the two spaces formed by $\phi(t-n)$ and $F(t-n)$ are orthogonal. Considering the properties of $V_{j}$ and $W_{j}$, $\{F_{j,k}(t)=2^{j/2}F(2^jt-k),(j,k)\in\mathbb{Z}^2\}$ constitutes the orthonormal basis for $L^2(\mathbb{R})$, and hence $F(t)$ is an orthogonal wavelet.

According to Proposition \ref{prop:2}, by designing a specific functional relationship between $H(\omega)$ and $G(\omega)$, we can guarantee that $M(\omega)$ is the Unitary Matrix, and thus that $F(t)$ is the orthogonal wavelet. For instance, when $G(\omega)=e^{-i\omega}\overline{H}(\omega+\pi)$, it is easy to verify that $M(\omega)$ is the Unitary Matrix, when the frequency domain form of the wavelet function $F(t)$ is that $\hat{F}(\omega)=e^{-i\omega/2}\overline{H}(\pi+\omega/2)\cdot\hat{\phi}(\omega/2)$, where the corresponding impulse response coefficient is $g_n=(-1)^{n-1}\overline{h}_{1-n},n\in\mathbb{Z}$. Therefore, the time domain form of the wavelet function is $F(t)=\sqrt{2}\sum_{n\in\mathbb{Z}}{(-1)^{n-1}\overline{h}_{1-n}\phi(2t-n)}$.
\end{proof}

\subsection{Time Series Tokenization} \label{section:methods_4}
Tokenizing multiple time series domains and establishing their cross-domain connectivities is a challenging task.
We propose a time series tokenization approach to map the input time series from diverse domains into tokens in the common spectral latent space following the ``wave as token'' principle, which is described as follows. 

Based on proposition \ref{prop:2}, the amplitude sequence $A\in\mathbb{R}^{2^{m}}$ is utilized to discretely inscribe the orthogonal wavelet.
Specifically, $A$ defines the amplitude values of $2^{m}$ points that contain information about the waveforms of the specified orthogonal wavelet function, where $m$ serves as the precision of the discrete sequence describing the information in $A$, and the spacing between any two neighboring points is denoted as $step$ which equals to $m/(2^m-1)$

We denote the central frequency of the orthogonal wavelet as $f_{c}$ and the size of the wavebook as $\lambda$. The set of scale factors is calculated by:
\begin{equation}\label{eq:0}
    S=\left\{S_i=(2\cdot f_c\cdot\lambda)/i,i\in[1,2,\ldots,\lambda]\right\},
\end{equation}
where $(m\cdot S_i)\in\mathbb{Z}$. For each scale factor $S_i$, the downsampling coordinate $W_i$ and basis function $A_i\in\mathbb{R}^{m\cdot S_i}$ in wavebook can be sampled by:
\begin{equation}\label{eq:1}
  \begin{split}                
  W_i&=\left\{w_{i,j}=\frac{(2^m-1)\cdot j}{m\cdot s_i},j\in\left[1,...,m\cdot s_i\right]\right\}, \\
  A_i&=A[W_i]=A[w_{i,1},...,w_{i,m\cdot s_i}]\in\mathbb{R}^{m\cdot s_i},    \\
  \end{split}
\end{equation}
where each $A_{i}$ contains waveform information of the specified basis function, which is obtained by the scaling transform on the orthogonal wavelet with the scale factor $S_{i}$. For convenience, we subsequently refer to $A$ as the orthogonal wavelet and $A_i$ as the basis function that makes up the wavebook.

For the time series $X=\left(x_{1},...,x_{l}\right)$ of length $l$ and the basis function $A_{i}=(a_{i,1},...,a_{i,n})$, where $n=m\cdot s_i$ and we define the following operation:
\begin{equation}\label{eq:2}
  \begin{split}                
  Convolve(X,A_i)&=\left(c_{i,1},c_{i,2},...,c_{i,l+n-1}\right), \\
  Difference(X,A_i)&=(d_{i,1},d_{i,2},...,d_{i,l+n-2}), \\
  Recode_{A_i}(X)&=(p_{i,1},p_{i,2},...,p_{i,l}), \\
  \end{split}
\end{equation}
where, $p_{i,j}$, $d_{i,j}$ and $c_{i,j}$ are calculated sequentially based on a given set of $x_k$ and $a_{i,j}$ as follows:
\begin{equation}\label{eq:2.1}
  \begin{split}                
  c_{i,j}&=\sum_{k=1}^jx_k\cdot a_{i,j+1-k},~~~\Delta a_{i,j}=a_{i,j+1}-a_{i,j},\\
  d_{i,j}&=-\sqrt{s_i}\cdot(c_{i,j+1}-c_{i,j})=-\sqrt{s_i}\cdot\sum_{k=1}^{j+1}x_k\Delta a_{i,j+1-k}, \\
  p_{i,j}&=d_{i,j+\frac n2-1}=-\sqrt{s_i}\cdot\sum_{k=1}^{j+\frac n2}x_k\Delta a_{i,j+\frac n2-k}. \\
  \end{split}
\end{equation}

Considering the property defined by the basis function $A_{i}$, that is the fluctuation of $A_{i}$ is concentrated in a finite region (assumed to be the intermediate region).
In Eq. \ref{eq:2}, when $k\to j$, there is $\begin{vmatrix}a_{i,j+(n/2)-k}\end{vmatrix},\begin{vmatrix}a_{i,j+(n/2)-k-1}\end{vmatrix}\rightarrow\begin{vmatrix}a_{i,(n/2)}\end{vmatrix}$, and $\begin{vmatrix}a_{i,1}\end{vmatrix},\begin{vmatrix}a_{i,n}\end{vmatrix}\to0$. Therefore, $p_{i,j}$ is regarded as the fluctuation pattern similarity between the segment of time series $X$ at timestep-$j$ and the basis function $A_{i}$. 
Iterating over all basis functions $\{A_{1},A_{2},...,A_{\lambda}\}$, the set of pattern similarity $token_j=(p_{1,j},p_{2,j},...,p_{\lambda,j})$ can be calculated for any timestep among time series.

Further, the set of orthogonal basis functions can be utilized to form the $\lambda$-dimensional embedding space $V^{\lambda}$. 
With the $token$ described above, the neighborhood segment at timestep-$j$ can be projected as the separate point in $V^{\lambda}$, and $token_{j}\text{=}(p_{1,j},p_{2,j},...,p_{\lambda,j})\in\mathbb{R}^{\lambda}$ is regarded as the projection coordinates of the segment in the $\lambda$-dimensional common space $V^{\lambda}$, where the length of the segment is adapted by the attenuation of the fluctuation distribution of the basis functions. 
Besides, each timestep of the original series $X$ is traversed, and all the projected coordinates are concatenated together as $T^0\text{=}(token_1,token_2,...,token_l)\text{+}T_{pos}$, where $T^0,T_{pos}\in\mathbb{R}^{\lambda \times l}$ and the learnable position encoding $T_{pos}$ indicates the temporal order between sub-series.

\subsection{Cross-domain Pre-training} \label{section:methods_6}
We describe the fine-grained cross-domain pre-training strategy as follows. In each epoch, variables from multiple domains are randomly shuffled with the batch to ensure that the unified general cross-domain model can simultaneously exhibit satisfactory generalization ability across diverse time series domains. 
Specifically, the backbone encoder accepts the tokens from diverse time series domains as input and learns the dependencies between segments via the attention mechanism. Besides, each token vector serves as the fluctuation pattern similarity, thus containing sufficient semantic information. 
Notably, since all tokens come from the unified common embedding space, it alleviates the heterogeneity of TS multi-domains. These advantages ensure that the proposed model is advantageous in multi-domain pre-training phase.

To fully stimulate the inference capability of \myformer in downstream tasks, we design forecasting and classification as multi-tasks for different downstream tasks.
Besides, in the cross-domain pre-training phase, a set of adaptive dynamic weights was designed to balance the gradients from different time series domains, considering differences in generalization across time series domains and the dataset size.
Specifically, we define the aggregate loss function as $Loss=\sum_{i\in\mathbb{Z}}(\alpha_i\cdot Loss_i)$, where $Loss_i$ is defined as the loss value of specific $domain_i$, and multi-task weightings $\{\alpha_i;i\in\mathbb{Z}\}$ are automatically calculated by the learnable strategy \cite{kendall2018multi} which considering the homoscedastic uncertainty of each task.

\subsection{Cross-domain Migration} \label{section:methods_7}
To validate the domain migration capability, \myformer is pre-trained in multiple TS source domains, it contains two natural advantages: (1) the knowledge learned by the model in the pre-training phase has the same modality information as the target domain; (2) the proposed wave quantize module uniformly represents the source and target domains in the common space.
These advantages encourage us to adopt the simple fine-tuning strategy of \textbf{allowing all parameters to fine-tune in the target domain}. The TS samples in the source and target domain are directly projected to the common spectral latent space formed by the wave quantize module. Then it achieves alignment between target and source domains at the representation level and stimulates the cross-domain migration capability of the downstream backbone.
Specifically, \myformer is fine-tuned on partial samples of the target domain, finally predicted on the target domain, in the few-shot classification task, as shown in the table \ref{tab:brief_classification_few};
This implementation ensures that all performance improvements come from the proposed \textbf{Wave Quantize} module, and there is still a huge room for improvement in \myformer, based on the existing Parameter-Efficient Fine-Tuning approaches \cite{peft}.

\section{Architecture of \myformer} \label{section:methods_5}
The architecture of \myformer consists of a \textbf{Tokenization} operation (Section.~\ref{section:methods_4}), an encoder-only Transformer, and an \textbf{OutputLayer}, as illustrated in Fig.~\ref{fig:intro_2}. Specifically, the OutputLayer adaptively selects the model architecture according to different downstream tasks. For example, for forecasting and impuation, the OutputLayer only contains a single linear layer, for classification task, OutputLayer consists of linear layer and Softmax.
The backbone network is a Transformer encoder consisting of stack of $L$ \textbf{EncoderLayers}. The overall architecture is as follows:
\begin{equation}\label{eq:3}
  \begin{split}                
  T^0&=Tokenization(X), \\
  T^{k+1}&=EncoderLayer(T^{k}), k=0,...,L-1, \\
  \overline{Y}&=OutputLayer(T^L).
  \end{split}
\end{equation}

\paragraph{EncoderLayer}
Based on section \ref{section:methods_4}, the time series tokens $T^{k}\in\mathbb{R}^{\lambda \times l}$ are utilized as the input to the multi-head attention mechanism, where $T^{k}$ contains $l$ $\lambda$-dimensional embedded token vectors. 

Specifically, the dependencies between the generated tokens in the common spectral latent space $V^{\lambda}$ can be calculated by the equation $\left(\hat{T}^{k}\right)^\top\text{=}Softmax\left({Q_L\cdot\left(K_L\right)^\top}/{\sqrt{d_k}}\right)\cdotp V_L$, where $Q_L,K_L,V_L=\left(T^{k}\right)^\top\cdot\left[\begin{matrix}W^Q,W^K,W^V\end{matrix}\right]$ are served as query, key, and value latent-variable in attention. 
 
Concretely, $W^Q,W^K\in\mathbb{R}^{{\lambda}\times d_k}$ and $W^V\in\mathbb{R}^{{\lambda}\times {\lambda}}$ denote the trainable linear layer, which projects the $token_j$ into the $d_k$-dimensional latent space. Finally, $\hat{T}^{k}\in\mathbb{R}^{\lambda \times l}$ is the output of the attention mechanism. Besides, each \textbf{EncoderLayer} is also composed of the feed-forward network and layer normalization with residual connections as shown in Fig.~\ref{fig:intro_2}, and generates the latent representation $T^{k+1}\in\mathbb{R}^{\lambda \times l}$ as the input for the next \textbf{EncoderLayer}.

\paragraph{OutputLayer}
The output of the last EncoderLayer is essentially the set of tokens. To make it possible to match the predicted time series format, tokens are first processed by flattening to obtain a 1D series, which is subsequently projected to a specified prediction length via the linear layer. 
Specifically, the \textbf{OutputLayer} contains two components, \textbf{Flatten}: $\mathbb{R}^{\lambda\times l}\mapsto\mathbb{R}^{\lambda l}$ and \textbf{Linear}: $\mathbb{R}^{\lambda l}\mapsto\mathbb{R}^c$, where the output of the OutputLayer is represented as $\overline{Y}\in\mathbb{R}^c$. The detailed process of the master training stage is in Algorithm \ref{alg:stage1}.

\begin{algorithm}[t]
    
    \LinesNumbered
    \KwIn{Lookback series $X=(x_{1},...,x_{l})\in\mathbb{R}^{l}$.}
    \KwOut{Forecasting series $Y=(x_{l+1},...,x_{l+c})\in\mathbb{R}^{c}$.}
    \caption{Master Training Stage of \myformer}\label{alg:stage1}{
    Based on the sufficient-necessary condition introduced in Proposition \ref{prop:2}, the orthogonal wavelet F(t) is first designed to adapte the characteristic of the data domain; \\
    Subsequently, the amplitude sequence $A\in\mathbb{R}^{2^{m}}$ is utilized to discretely inscribe the orthogonal wavelet; \\
    Finally, a set $A_i$ will be sampled from $A$, where each $A_i$ contains waveform information of the specified basis function, which is obtained by the scaling transform on the orthogonal wavelet with the scale factor $S_i$, as shown in \eqref{eq:0}-\eqref{eq:1}; \\
    }
    Randomly initialize the parameter set $\theta_1$;\\
    \For{iteration = 1,2,3, ...}{
    Based on the design of the set $\{A_{1},...,A_{\lambda}\}$, each timestep of the original series $X$ is traversed by Proposition \ref{prop:1}, and all the projected coordinates are concatenated together as $T^0\text{=}(token_1,...,token_l)\text{+}T_{pos}$, where $token_j$ contains the fluctuation pattern similarity between time series $X$ and basis function $\{A_{i}\}$, by \eqref{eq:2}-\eqref{eq:2.1}; \\
    \myformer uses the Encoder obtained by stacking $L$ EncoderLayers as the model backbone. Subsequently the time series tokens $T^0 \in\mathbb{R}^{\lambda \times l}$ are utilized as the input to the multi-head attention mechanism of EncoderLayer, where $\lambda$ and $l$ indicate the dimension and number of embedded token vectors, as follows \eqref{eq:3}; \\
    The output $T^{L}\in\mathbb{R}^{\lambda \times l}$ of the last EncoderLayer is essentially a set of tokens that are first processed by flattening to obtain a 1D series, and then projected to a specified prediction length via the Linear Layer. Ultimately, the output of the output layer is represented as $\overline{Y}\in\mathbb{R}^c$; \\
    Get the forecasting loss $\mathcal{L}_{mse}$ between prediction $\overline{Y}$ and ground truth $Y$, subsequently update parameters $\theta_1$ according to the gradient of $\mathcal{L}_{mse}$; \\
    }
    \Return{$\theta_1$}
\end{algorithm}

\section{Experiments} \label{section:experiments}
In this section, we first introduce the benchmark and baseline which will be utilized in the subsequent experiments, followed by three tasks in each of the three subsections, where each task consists of three different settings: 
1) \textbf{Full-data}: models are trained and predicted on the target domain; 
2) \textbf{Few-shot}: In forecasting and imputation tasks, models are trained and predicted on the target domain, where the dataset is only partially available in the training phase; In the classification task, models are first pre-trained on the source domain, subsequently, are fine-tuned on partial samples of target domain, finally predicted on the target domain; These designs were utilizaed to demonstrate the data efficiency and cross-domain adaptability of the proposed model; 
3) \textbf{Zero-shot}: models are predicted in the target domain directly after pre-training in the single or multiple source domain; 

\begin{table}[htbp]
  \vspace{-5pt}
  \caption{Experiment configuration of \myformer. 
  }\label{tab:model_config}
  \vspace{-10pt}
  \vskip 0.05in
  \centering
  \resizebox{0.8\columnwidth}{!}{
  \begin{threeparttable}
  \begin{small}
  \renewcommand{\multirowsetup}{\centering}
  \setlength{\tabcolsep}{4.3pt}
  \begin{tabular}{c|c|c|c|c|c|c|c|c}
  
    \toprule
    \multirow{2}{*}{Tasks} & \multicolumn{4}{c}{Model Hyper-parameter} & \multicolumn{4}{c}{Training Process} \\
    
    \cmidrule(lr){2-5}\cmidrule(lr){6-9}
    & Layers & $d_{\text{min}}$ $^\dag$ & $d_{\text{max}}$ $^\dag$ & $\lambda^{\circ}$ & LR$^\ast$ & Loss & Batch Size & Epochs\\
    \toprule
    
    Forecasting & 10 & 32 & 512 & 100 & $10^{-4}$ & MSE & 32 & 10 \\
    \midrule
    
    Imputation & 10 & 64 & 128 & 100 & $10^{-4}$ & MSE & 32 & 10 \\
    \midrule

    Classification & 5 & 64 & 128 & 100 & $10^{-3}$ & MSE & 64 & 30 \\
    \bottomrule
    
    \end{tabular}
    \begin{tablenotes}
        \footnotesize
        \item[] $\dag$ $d_{\text{model}}=\min\{\max\{2^{\lceil\log C\rceil}, d_{\text{min}}\},d_{\text{max}}\}$, where $C$ is input series dimension.
        \item[] $\ast$ LR means the initial learning rate.
        \item[] $\circ$ $\lambda$ means the dimension of shared embedding space, that is the size of the wavebook.
  \end{tablenotes}
    \end{small}
  \end{threeparttable}
  }
\vspace{-10pt}
\end{table}

\subsection{Configurations}
We provide the experiment configuration in Table \ref{tab:model_config}. 
All experiments are repeated three times, implemented in PyTorch and conducted on a single Tesla V100 SXM2 32GB GPU. 
Our method is trained with the L2 Loss, using the ADAM optimizer with an initial learning rate of $10^{-4}$, and Batch size is set in $16 \rightarrow 64$. The training process is early stopped after three epochs (patience=3) if there is no loss degradation on the valid set.
The mean square error (MSE) and mean absolute error (MAE) are used as metrics in forecasting and imputation tasks. Besides, the accuracy (Acc), precision (Pre), recall (Rec), and f1-score (F1) are used as metrics in the classification task.
By default, the proposed \textbf{\myformer} contains $5 \rightarrow 10$ \textbf{EncoderLayers}. All the baselines that we reproduced are implemented based on configurations of the original paper or their official code. For a fair comparison, we design the same input embedding and final prediction layer for all base models. 

\subsection{Benchmarks} \label{sec:benchmark}
To evaluate the performance of the proposed method, we extensively experiment with the mainstream time series analysis tasks including long-term forecasting, imputation (i.e., predicting the missing data in a time series), and classification. 
The long-term forecasting, imputation and classification are evaluated with several popular real-world datasets, including: 
\textbf{ETT (ETTh1, ETTh2, ETTm1, and ETTm2)}\footnote{https://github.com/zhouhaoyi/Informer2020} \cite{haoyietal-informer-2021} contains six power load features and oil temperature used for monitoring electricity transformers. ETT involves four subsets. ETTm1 and ETTm2 are recorded at 15-minute intervals, while ETTh1 and ETTh2 are recorded hourly. 
\textbf{Exchange}\footnote{https://github.com/laiguokun/multivariate-time-series-data} \cite{Lai2017ModelingLA} records daily exchange rates of eight different countries ranging from 1990 to 2016.
\textbf{Weather}\footnote{https://www.bgc-jena.mpg.de/wetter/} contains 21 meteorological indicators, such as temperature, humidity, and precipitation, which are recorded every 10 minutes in the year 2020. 
\textbf{Electricity}\footnote{https://archive.ics.uci.edu/dataset/321/electricity} comprises hourly power consumption of 321 clients from 2012 to 2014. 
\textbf{Traffic}\footnote{http://pems.dot.ca.gov} reports the number of vehicles loaded on all 862 roads at each moment in time. 
\textbf{Sunspot}\footnote{https://www.sidc.be/SILSO/newdataset} records observations of sunspots for long-term monitoring, consisting of $73924$ timesteps.
\textbf{River Flow}\footnote{http://www.jenvstat.org/v04/i11} reports th daily river flow, consisting of $23741$ timesteps.
\textbf{Solar Power}\footnote{https://zenodo.org/records/4656032} contains a single long daily time series representing the wind power production in MW recorded every 4 seconds starting from 2019.
\textbf{UCR archive}\footnote{https://www.cs.ucr.edu/\url{~}eamonn/time\_series\_data\_2018} as the well-known time series classification repository, where representative \textbf{35 datasets} are selected to validate the performance of the proposed approach.

\subsection{Baselines} 
We compare the proposed \myformer model with the well-acknowledged and advanced models, which include the CNN-based Models: \textbf{TimesNet} \cite{wu2023timesnet} and \textbf{MICN} \cite{wang2023micn}; the MLP-based model: \textbf{DLinear} \cite{Zeng2022AreTE}; the Transformer-based models: \textbf{Informer} \cite{haoyietal-informer-2021}, \textbf{ETSformer}~\cite{woo2022etsformer}, \textbf{Stationary}~\cite{Liu2022NonstationaryTR}, \textbf{Autoformer} \cite{wu2021autoformer}, \textbf{FEDformer}~\cite{zhou2022fedformer}, and \textbf{PatchTST}~\cite{Nie2023ATS}; 
and a LLM-empowered model: \textbf{OneFitsAll}~\cite{Zhou2023OneFA}.
In the \textbf{forecasting} task, to indicate the generalization capability on different prediction scales, we fixed the lookback length as 336, and the prediction lengths including $\{96, 192, 336, 720\}$. 
In the \textbf{imputation} task, to compare the performance under different proportions of missing data, we randomly mask the time points with a ratio of $\{12.5\%, 25\%, 37.5\%, 50\%\}$, and the lookback length is fixed as 96. 
Besides, the \textbf{classification} task is shown only for accuracy and F1 score, while the full-data task on all 35 datasets from UCR, and the few-shot task on multiple scenarios.


\subsection{Long term forecasting} 
\begin{table*}[htbp]
  \caption{Comparison of the averaged performance from diverse prediction lengths ($\{96,192,336,720\}$) on \textbf{full-data forecasting} task, where ETTh1,h2,m1,m2 are from the same dataset.}\label{tab:brief_forecasting_full}
  \centering
  \resizebox{1.0\columnwidth}{!}{
  \begin{threeparttable}
  \begin{small}
  \renewcommand{\multirowsetup}{\centering}
  \setlength{\tabcolsep}{1pt}
  \begin{tabular}{ccc||cccccccccccccccccc}
    \toprule
    \hline
    
    \multicolumn{1}{c}{\multirow{1}{*}{Models}} & 
    \multicolumn{2}{c}{\rotatebox{0}{\scalebox{0.8}{\textbf{\myformer}}}} &
    \multicolumn{2}{c}{\rotatebox{0}{\scalebox{0.8}{OneFitsAll}}} & 
    \multicolumn{2}{c}{\rotatebox{0}{\scalebox{0.8}{DLinear}}} &
    \multicolumn{2}{c}{\rotatebox{0}{\scalebox{0.8}{PatchTST}}} &
    \multicolumn{2}{c}{\rotatebox{0}{\scalebox{0.8}{TimesNet}}} &
    \multicolumn{2}{c}{\rotatebox{0}{\scalebox{0.8}{FEDformer}}} & 
    \multicolumn{2}{c}{\rotatebox{0}{\scalebox{0.8}{Autoformer}}} & 
    \multicolumn{2}{c}{\rotatebox{0}{\scalebox{0.8}{Stationary}}} & 
    \multicolumn{2}{c}{\rotatebox{0}{\scalebox{0.8}{ETSformer}}} &  
    \multicolumn{2}{c}{\rotatebox{0}{\scalebox{0.8}{Informer}}} \\
    
    \multicolumn{1}{c}{} & 
    \multicolumn{2}{c}{\scalebox{0.8}{(\textbf{Ours})}} & 
    \multicolumn{2}{c}{\scalebox{0.8}{\citeyear{Nie2023ATS}}} &
    \multicolumn{2}{c}{\scalebox{0.8}{\citeyear{wu2023timesnet}}} &
    \multicolumn{2}{c}{\scalebox{0.8}{\citeyear{wang2023micn}}} &
    \multicolumn{2}{c}{\scalebox{0.8}{\citeyear{Zeng2022AreTE}}} & 
    \multicolumn{2}{c}{\scalebox{0.8}{\citeyear{zhou2022fedformer}}} & 
    \multicolumn{2}{c}{\scalebox{0.8}{\citeyear{Liu2022NonstationaryTR}}} & 
    \multicolumn{2}{c}{\scalebox{0.8}{\citeyear{wu2021autoformer}}} & 
    \multicolumn{2}{c}{\scalebox{0.8}{\citeyear{liu2021pyraformer}}} &  
    \multicolumn{2}{c}{\scalebox{0.8}{\citeyear{haoyietal-informer-2021}}} \\
    
    \cline{2-21}
    
    \multicolumn{1}{c}{Metric} & 
    \scalebox{0.78}{MSE} & \scalebox{0.78}{MAE} & 
    \scalebox{0.78}{MSE} & \scalebox{0.78}{MAE} & 
    \scalebox{0.78}{MSE} & \scalebox{0.78}{MAE} & 
    \scalebox{0.78}{MSE} & \scalebox{0.78}{MAE} & 
    \scalebox{0.78}{MSE} & \scalebox{0.78}{MAE} & 
    \scalebox{0.78}{MSE} & \scalebox{0.78}{MAE} & 
    \scalebox{0.78}{MSE} & \scalebox{0.78}{MAE} & 
    \scalebox{0.78}{MSE} & \scalebox{0.78}{MAE} & 
    \scalebox{0.78}{MSE} & \scalebox{0.78}{MAE} & 
    \scalebox{0.78}{MSE} & \scalebox{0.78}{MAE} \\
    
    \hline
    \scalebox{0.78}{ETTm1} &
    {\scalebox{0.78}{0.368}} & \textbf{\scalebox{0.78}{0.369}} &
    {\scalebox{0.78}{0.352}} & {\scalebox{0.78}{0.383}} &
    {\scalebox{0.78}{0.357}} & {\scalebox{0.78}{0.378}} &
    \textbf{\scalebox{0.78}{0.351}} & {\scalebox{0.78}{0.380}} &
    {\scalebox{0.78}{0.400}} & {\scalebox{0.78}{0.406}} &
    {\scalebox{0.78}{0.448}} & {\scalebox{0.78}{0.452}} &
    {\scalebox{0.78}{0.588}} & {\scalebox{0.78}{0.517}} &
    {\scalebox{0.78}{0.481}} & {\scalebox{0.78}{0.456}} &
    {\scalebox{0.78}{0.429}} & {\scalebox{0.78}{0.425}} &
    {\scalebox{0.78}{0.961}} & {\scalebox{0.78}{0.734}} \\

    \hline
    \scalebox{0.78}{ETTm2} &
    \textbf{\scalebox{0.78}{0.247}} & \textbf{\scalebox{0.78}{0.295}} &
    {\scalebox{0.78}{0.266}} & {\scalebox{0.78}{0.326}} &
    {\scalebox{0.78}{0.267}} & {\scalebox{0.78}{0.333}} &
    {\scalebox{0.78}{0.255}} & {\scalebox{0.78}{0.315}} &
    {\scalebox{0.78}{0.291}} & {\scalebox{0.78}{0.333}} &
    {\scalebox{0.78}{0.305}} & {\scalebox{0.78}{0.349}} &
    {\scalebox{0.78}{0.327}} & {\scalebox{0.78}{0.371}} &
    {\scalebox{0.78}{0.306}} & {\scalebox{0.78}{0.347}} &
    {\scalebox{0.78}{0.293}} & {\scalebox{0.78}{0.342}} &
    {\scalebox{0.78}{1.410}} & {\scalebox{0.78}{0.810}} \\

    \hline
    \scalebox{0.78}{ETTh1} &
    \textbf{\scalebox{0.78}{0.407}} & \textbf{\scalebox{0.78}{0.418}} &
    {\scalebox{0.78}{0.427}} & {\scalebox{0.78}{0.426}} &
    {\scalebox{0.78}{0.422}} & {\scalebox{0.78}{0.437}} &
    {\scalebox{0.78}{0.413}} & {\scalebox{0.78}{0.430}} &
    {\scalebox{0.78}{0.458}} & {\scalebox{0.78}{0.450}} &
    {\scalebox{0.78}{0.440}} & {\scalebox{0.78}{0.460}} &
    {\scalebox{0.78}{0.496}} & {\scalebox{0.78}{0.487}} &
    {\scalebox{0.78}{0.570}} & {\scalebox{0.78}{0.537}} &
    {\scalebox{0.78}{0.542}} & {\scalebox{0.78}{0.510}} &
    {\scalebox{0.78}{1.040}} & {\scalebox{0.78}{0.795}} \\
    
    \hline
    \scalebox{0.78}{ETTh2} &
    {\scalebox{0.78}{0.347}} & \textbf{\scalebox{0.78}{0.373}} &
    {\scalebox{0.78}{0.354}} & {\scalebox{0.78}{0.394}} &
    {\scalebox{0.78}{0.431}} & {\scalebox{0.78}{0.446}} &
    \textbf{\scalebox{0.78}{0.330}} & {\scalebox{0.78}{0.379}} &
    {\scalebox{0.78}{0.414}} & {\scalebox{0.78}{0.427}} &
    {\scalebox{0.78}{0.437}} & {\scalebox{0.78}{0.449}} &
    {\scalebox{0.78}{0.450}} & {\scalebox{0.78}{0.459}} &
    {\scalebox{0.78}{0.526}} & {\scalebox{0.78}{0.516}} &
    {\scalebox{0.78}{0.439}} & {\scalebox{0.78}{0.452}} &
    {\scalebox{0.78}{4.431}} & {\scalebox{0.78}{1.729}} \\
    
    \hline
    \scalebox{0.78}{Electricity} &
    \textbf{\scalebox{0.78}{0.154}} & \textbf{\scalebox{0.78}{0.247}} &
    {\scalebox{0.78}{0.167}} & {\scalebox{0.78}{0.263}} &
    {\scalebox{0.78}{0.166}} & {\scalebox{0.78}{0.263}} &
    {\scalebox{0.78}{0.161}} & {\scalebox{0.78}{0.252}} &
    {\scalebox{0.78}{0.192}} & {\scalebox{0.78}{0.295}} &
    {\scalebox{0.78}{0.214}} & {\scalebox{0.78}{0.295}} &
    {\scalebox{0.78}{0.227}} & {\scalebox{0.78}{0.327}} &
    {\scalebox{0.78}{0.193}} & {\scalebox{0.78}{0.338}} &
    {\scalebox{0.78}{0.208}} & {\scalebox{0.78}{0.296}} &
    {\scalebox{0.78}{0.311}} & {\scalebox{0.78}{0.397}} \\
    
    \hline
    \scalebox{0.78}{Traffic} &
    \textbf{\scalebox{0.78}{0.379}} & {\scalebox{0.78}{0.286}} &
    {\scalebox{0.78}{0.414}} & {\scalebox{0.78}{0.294}} &
    {\scalebox{0.78}{0.433}} & {\scalebox{0.78}{0.295}} &
    {\scalebox{0.78}{0.390}} & \textbf{\scalebox{0.78}{0.263}} &
    {\scalebox{0.78}{0.620}} & {\scalebox{0.78}{0.336}} &
    {\scalebox{0.78}{0.610}} & {\scalebox{0.78}{0.376}} &
    {\scalebox{0.78}{0.628}} & {\scalebox{0.78}{0.379}} &
    {\scalebox{0.78}{0.624}} & {\scalebox{0.78}{0.340}} &
    {\scalebox{0.78}{0.621}} & {\scalebox{0.78}{0.396}} &
    {\scalebox{0.78}{0.764}} & {\scalebox{0.78}{0.416}} \\
    
    \hline
    \scalebox{0.78}{Weather} &
    \textbf{\scalebox{0.78}{0.216}} & \textbf{\scalebox{0.78}{0.246}} &
    {\scalebox{0.78}{0.237}} & {\scalebox{0.78}{0.270}} &
    {\scalebox{0.78}{0.248}} & {\scalebox{0.78}{0.300}} &
    {\scalebox{0.78}{0.225}} & {\scalebox{0.78}{0.264}} &
    {\scalebox{0.78}{0.259}} & {\scalebox{0.78}{0.287}} &
    {\scalebox{0.78}{0.309}} & {\scalebox{0.78}{0.360}} &
    {\scalebox{0.78}{0.338}} & {\scalebox{0.78}{0.382}} &
    {\scalebox{0.78}{0.288}} & {\scalebox{0.78}{0.314}} &
    {\scalebox{0.78}{0.271}} & {\scalebox{0.78}{0.334}} &
    {\scalebox{0.78}{0.634}} & {\scalebox{0.78}{0.548}} \\
    
    \hline
    \scalebox{0.78}{Sunspot} &
    \textbf{\scalebox{0.78}{0.395}} & \textbf{\scalebox{0.78}{0.442}} &
    {\scalebox{0.78}{0.445}} & {\scalebox{0.78}{0.477}} &
    {\scalebox{0.78}{0.526}} & {\scalebox{0.78}{0.554}} &
    {\scalebox{0.78}{0.446}} & {\scalebox{0.78}{0.476}} &
    {\scalebox{0.78}{0.450}} & {\scalebox{0.78}{0.478}} &
    {\scalebox{0.78}{0.477}} & {\scalebox{0.78}{0.498}} &
    {\scalebox{0.78}{0.458}} & {\scalebox{0.78}{0.488}} &
    {\scalebox{0.78}{0.462}} & {\scalebox{0.78}{0.496}} &
    {\scalebox{0.78}{0.481}} & {\scalebox{0.78}{0.533}} &
    {\scalebox{0.78}{0.559}} & {\scalebox{0.78}{0.604}} \\
    
    \hline
    \scalebox{0.78}{RiverFlow} &
    \textbf{\scalebox{0.78}{1.004}} & \textbf{\scalebox{0.78}{0.497}} &
    {\scalebox{0.78}{1.218}} & {\scalebox{0.78}{0.551}} &
    {\scalebox{0.78}{1.146}} & {\scalebox{0.78}{0.605}} &
    {\scalebox{0.78}{1.233}} & {\scalebox{0.78}{0.658}} &
    {\scalebox{0.78}{1.247}} & {\scalebox{0.78}{0.651}} &
    {\scalebox{0.78}{1.139}} & {\scalebox{0.78}{0.596}} &
    {\scalebox{0.78}{1.246}} & {\scalebox{0.78}{0.651}} &
    {\scalebox{0.78}{1.137}} & {\scalebox{0.78}{0.574}} &
    {\scalebox{0.78}{1.250}} & {\scalebox{0.78}{0.659}} &
    {\scalebox{0.78}{1.312}} & {\scalebox{0.78}{0.737}} \\
    
    \hline
    \scalebox{0.78}{SolarPower} &
    \textbf{\scalebox{0.78}{0.031}} & \textbf{\scalebox{0.78}{0.063}} &
    {\scalebox{0.78}{0.036}} & {\scalebox{0.78}{0.072}} &
    {\scalebox{0.78}{0.046}} & {\scalebox{0.78}{0.093}} &
    {\scalebox{0.78}{0.043}} & {\scalebox{0.78}{0.086}} &
    {\scalebox{0.78}{0.082}} & {\scalebox{0.78}{0.149}} &
    {\scalebox{0.78}{0.046}} & {\scalebox{0.78}{0.094}} &
    {\scalebox{0.78}{0.093}} & {\scalebox{0.78}{0.167}} &
    {\scalebox{0.78}{0.064}} & {\scalebox{0.78}{0.120}} &
    {\scalebox{0.78}{0.076}} & {\scalebox{0.78}{0.149}} &
    {\scalebox{0.78}{0.078}} & {\scalebox{0.78}{0.147}} \\
    
    
    \hline
    \bottomrule
  \end{tabular}
    \end{small}
  \end{threeparttable}
}
\end{table*}
\begin{table*}[htpb]
  \caption{Comparison of the averaged performance from diverse prediction lengths ($\{96,192,336,720\}$) on \textbf{zero-shot forecasting} task. Where $Source\rightarrow Traget$ indicates that the model is first pre-trained on the single train set of the \emph{SourceDomain}, subsequently, the model parameters are frozen and predicted on the test set of the \emph{TargetDomain}.}\label{tab:brief_forecasting_zero_cross_single}
  \centering
  \resizebox{0.9\columnwidth}{!}{
  \begin{small}
  \renewcommand{\multirowsetup}{\centering}
  \setlength{\tabcolsep}{1.5pt}
  \renewcommand\arraystretch{1}
  \begin{tabular}{ccc||cccccccccccc}
    \toprule
    \hline
    
    \multicolumn{1}{c}{\multirow{2}{*}{\scalebox{1.0}{Scenarios}}} & 
    \multicolumn{2}{c}{\rotatebox{0}{\scalebox{0.8}{\textbf{\myformer}}}} & 
    \multicolumn{2}{c}{\rotatebox{0}{\scalebox{0.8}{OneFitsAll}}} & 
    \multicolumn{2}{c}{\rotatebox{0}{\scalebox{0.8}{DLinear}}} & 
    \multicolumn{2}{c}{\rotatebox{0}{\scalebox{0.8}{PatchTST}}} & 
    \multicolumn{2}{c}{\rotatebox{0}{\scalebox{0.8}{TimesNet}}} & 
    \multicolumn{2}{c}{\rotatebox{0}{\scalebox{0.8}{FEDformer}}} & 
    \multicolumn{2}{c}{\rotatebox{0}{\scalebox{0.8}{Autoformer}}} \\

    
    \cline{2-15} &
    \scalebox{0.78}{MSE} & \scalebox{0.78}{MAE} & 
    \scalebox{0.78}{MSE} & \scalebox{0.78}{MAE} & 
    \scalebox{0.78}{MSE} & \scalebox{0.78}{MAE} & 
    \scalebox{0.78}{MSE} & \scalebox{0.78}{MAE} & 
    \scalebox{0.78}{MSE} & \scalebox{0.78}{MAE} & 
    \scalebox{0.78}{MSE} & \scalebox{0.78}{MAE} & 
    \scalebox{0.78}{MSE} & \scalebox{0.78}{MAE} \\
    \hline
    
    \multirow{1}{*}{\scalebox{0.8}{\shortstack{ETTm2$\rightarrow$ETTm1}}} &
    \textbf{\scalebox{0.78}{0.434}} & \textbf{\scalebox{0.78}{0.429}} & 
    {\scalebox{0.78}{0.790}} & {\scalebox{0.78}{0.579}} & 
    {\scalebox{0.78}{0.516}} & {\scalebox{0.78}{0.473}} & 
    {\scalebox{0.78}{0.596}} & {\scalebox{0.78}{0.508}} & 
    {\scalebox{0.78}{0.857}} & {\scalebox{0.78}{0.599}} & 
    {\scalebox{0.78}{0.718}} & {\scalebox{0.78}{0.564}} & 
    {\scalebox{0.78}{0.722}} & {\scalebox{0.78}{0.566}} \\ 

    \hline
    
    \multirow{1}{*}{\scalebox{0.8}{\shortstack{ETTm1$\rightarrow$ETTm2}}} &
    \textbf{\scalebox{0.78}{0.293}} & \textbf{\scalebox{0.78}{0.326}} & 
    {\scalebox{0.78}{0.342}} & {\scalebox{0.78}{0.369}} & 
    {\scalebox{0.78}{0.360}} & {\scalebox{0.78}{0.410}} & 
    {\scalebox{0.78}{0.325}} & {\scalebox{0.78}{0.361}} & 
    {\scalebox{0.78}{0.357}} & {\scalebox{0.78}{0.384}} & 
    {\scalebox{0.78}{0.321}} & {\scalebox{0.78}{0.360}} & 
    {\scalebox{0.78}{0.325}} & {\scalebox{0.78}{0.365}} \\ 

    \hline
    
    \multirow{1}{*}{\scalebox{0.8}{\shortstack{ETTh2$\rightarrow$ETTh1}}} &
    \textbf{\scalebox{0.78}{0.512}} & \textbf{\scalebox{0.78}{0.493}} & 
    {\scalebox{0.78}{0.780}} & {\scalebox{0.78}{0.604}} & 
    {\scalebox{0.78}{0.609}} & {\scalebox{0.78}{0.532}} & 
    {\scalebox{0.78}{0.616}} & {\scalebox{0.78}{0.537}} & 
    {\scalebox{0.78}{0.920}} & {\scalebox{0.78}{0.635}} & 
    {\scalebox{0.78}{0.746}} & {\scalebox{0.78}{0.598}} & 
    {\scalebox{0.78}{0.735}} & {\scalebox{0.78}{0.593}} \\ 

    \hline
    
    \multirow{1}{*}{\scalebox{0.8}{\shortstack{ETTh1$\rightarrow$ETTh2}}} &
    \textbf{\scalebox{0.78}{0.385}} & \textbf{\scalebox{0.78}{0.405}} & 
    {\scalebox{0.78}{0.420}} & {\scalebox{0.78}{0.430}} & 
    {\scalebox{0.78}{0.478}} & {\scalebox{0.78}{0.483}} & 
    {\scalebox{0.78}{0.416}} & {\scalebox{0.78}{0.444}} & 
    {\scalebox{0.78}{0.443}} & {\scalebox{0.78}{0.442}} & 
    {\scalebox{0.78}{0.444}} & {\scalebox{0.78}{0.463}} & 
    {\scalebox{0.78}{0.445}} & {\scalebox{0.78}{0.459}} \\ 

    \hline
    
    \multirow{1}{*}{\scalebox{0.8}{\shortstack{RiverFlow$\rightarrow$Exchange}}} &
    \textbf{\scalebox{0.78}{0.381}} & \textbf{\scalebox{0.78}{0.424}} & 
    {\scalebox{0.78}{0.464}} & {\scalebox{0.78}{0.491}} & 
    {\scalebox{0.78}{0.585}} & {\scalebox{0.78}{0.537}} & 
    {\scalebox{0.78}{0.421}} & {\scalebox{0.78}{0.458}} & 
    {\scalebox{0.78}{0.497}} & {\scalebox{0.78}{0.508}} & 
    {\scalebox{0.78}{0.942}} & {\scalebox{0.78}{0.765}} & 
    {\scalebox{0.78}{0.845}} & {\scalebox{0.78}{0.739}} \\ 

    \hline
    
    \multirow{1}{*}{\scalebox{0.8}{\shortstack{Sunspot$\rightarrow$Weather}}} &
    \textbf{\scalebox{0.78}{0.254}} & \textbf{\scalebox{0.78}{0.286}} & 
    {\scalebox{0.78}{0.264}} & {\scalebox{0.78}{0.297}} & 
    {\scalebox{0.78}{0.263}} & {\scalebox{0.78}{0.310}} & 
    {\scalebox{0.78}{0.263}} & {\scalebox{0.78}{0.297}} & 
    {\scalebox{0.78}{0.311}} & {\scalebox{0.78}{0.325}} & 
    {\scalebox{0.78}{0.705}} & {\scalebox{0.78}{0.634}} & 
    {\scalebox{0.78}{0.509}} & {\scalebox{0.78}{0.501}} \\ 

    \hline
    \bottomrule
  \end{tabular}
  \end{small}
}
\end{table*}

\begin{table*}[htpb]
  \caption{Comparison of the averaged performance from diverse prediction lengths on \textbf{zero-shot forecasting} task, where $Source\rightarrow Traget$ indicates that the model is first pre-trained uniformly on all train sets from multiple \emph{SourceDomains}, subsequently, the model parameters are frozen and predicted on the test set of the \emph{TargetDomain}.}\label{tab:brief_forecasting_zero_cross_multi}
  \centering
  \resizebox{1.0\columnwidth}{!}{
  \begin{small}
  \renewcommand{\multirowsetup}{\centering}
  \setlength{\tabcolsep}{1pt}
  \renewcommand\arraystretch{1}
  \begin{tabular}{c|cc|cc|cc|cc||cc|cc||cc|cc}
    \toprule
    \hline
    
    \multicolumn{1}{c|}{\rotatebox{0}{\scalebox{0.8}{\textbf{Scenarios}}}} & 
    \multicolumn{8}{c||}{\rotatebox{0}{\scalebox{0.8}{\textbf{Zero-shot}}}} &
    \multicolumn{4}{c||}{\rotatebox{0}{\scalebox{0.8}{\textbf{Full-data}}}} &
    \multicolumn{4}{c}{\rotatebox{0}{\scalebox{0.8}{\textbf{Few-shot}}}} \\

    \hline

    \multicolumn{1}{c|}{\rotatebox{0}{\scalebox{0.8}{\textbf{Models}}}} & 
    \multicolumn{8}{c||}{\rotatebox{0}{\scalebox{0.8}{\textbf{\myformer}}}} &
    \multicolumn{2}{c}{\rotatebox{0}{\scalebox{0.8}{\textbf{OneFitsAll}}}} & 
    \multicolumn{2}{c||}{\rotatebox{0}{\scalebox{0.8}{\textbf{PatchTST}}}} & 
    \multicolumn{2}{c}{\rotatebox{0}{\scalebox{0.8}{\textbf{OneFitsAll}}}} & 
    \multicolumn{2}{c}{\rotatebox{0}{\scalebox{0.8}{\textbf{PatchTST}}}} \\

    \hline

    \multicolumn{1}{c|}{\rotatebox{0}{\textbf{\scalebox{0.8}{Metric}}}} & 
    \scalebox{0.7}{\textbf{MSE}} & \scalebox{0.7}{\textbf{MAE}} & 
    \scalebox{0.7}{\textbf{MSE}} & \scalebox{0.7}{\textbf{MAE}} & 
    \scalebox{0.7}{\textbf{MSE}} & \scalebox{0.7}{\textbf{MAE}} & 
    \scalebox{0.7}{\textbf{MSE}} & \scalebox{0.7}{\textbf{MAE}} & 
    \scalebox{0.7}{\textbf{MSE}} & \scalebox{0.7}{\textbf{MAE}} & 
    \scalebox{0.7}{\textbf{MSE}} & \scalebox{0.7}{\textbf{MAE}} & 
    \scalebox{0.7}{\textbf{MSE}} & \scalebox{0.7}{\textbf{MAE}} & 
    \scalebox{0.7}{\textbf{MSE}} & \scalebox{0.7}{\textbf{MAE}} \\
    
    \hline
    \bottomrule
    
    \multicolumn{1}{c|}{\rotatebox{0}{\scalebox{0.8}{SourceData}}} & 
    \multicolumn{2}{c|}{\rotatebox{0}{\scalebox{0.7}{ETT\{m2,h1,h2\}}}} & 
    \multicolumn{2}{c|}{\rotatebox{0}{\scalebox{0.7}{ETTm2}}} & 
    \multicolumn{2}{c|}{\rotatebox{0}{\scalebox{0.7}{ETTh1}}} & 
    \multicolumn{2}{c||}{\rotatebox{0}{\scalebox{0.7}{ETTh2}}} & 
    \multicolumn{2}{c|}{\rotatebox{0}{\scalebox{0.7}{ETTm1}}} & 
    \multicolumn{2}{c||}{\rotatebox{0}{\scalebox{0.7}{ETTm1}}} & 
    \multicolumn{2}{c|}{\rotatebox{0}{\scalebox{0.7}{ETTm1}}} & 
    \multicolumn{2}{c}{\rotatebox{0}{\scalebox{0.7}{ETTm1}}} \\
    
    \hline

    \multirow{1}{*}{\scalebox{0.8}{\shortstack{SourceData$\rightarrow$ETTm1}}} &
    \textbf{\scalebox{0.78}{0.411}} & \textbf{\scalebox{0.78}{0.416}} & 
    {\scalebox{0.78}{0.434}} & {\scalebox{0.78}{0.429}} & 
    {\scalebox{0.78}{0.682}} & {\scalebox{0.78}{0.742}} & 
    {\scalebox{0.78}{0.742}} & {\scalebox{0.78}{0.802}} & 
    {\scalebox{0.78}{0.352}} & {\scalebox{0.78}{0.383}} & 
    \textbf{\scalebox{0.78}{0.351}} & \textbf{\scalebox{0.78}{0.380}} & 
    {\scalebox{0.78}{0.472}} & {\scalebox{0.78}{0.450}} & 
    {\scalebox{0.78}{0.526}} & {\scalebox{0.78}{0.476}} \\
    
    \hline

    \multicolumn{1}{c|}{\rotatebox{0}{\scalebox{0.8}{SourceData}}} & 
    \multicolumn{2}{c|}{\rotatebox{0}{\scalebox{0.7}{ETT\{m1,h1,h2\}}}} & 
    \multicolumn{2}{c|}{\rotatebox{0}{\scalebox{0.7}{ETTm1}}} & 
    \multicolumn{2}{c|}{\rotatebox{0}{\scalebox{0.7}{ETTh1}}} & 
    \multicolumn{2}{c||}{\rotatebox{0}{\scalebox{0.7}{ETTh2}}} & 
    \multicolumn{2}{c|}{\rotatebox{0}{\scalebox{0.7}{ETTm2}}} & 
    \multicolumn{2}{c||}{\rotatebox{0}{\scalebox{0.7}{ETTm2}}} & 
    \multicolumn{2}{c|}{\rotatebox{0}{\scalebox{0.7}{ETTm2}}} & 
    \multicolumn{2}{c}{\rotatebox{0}{\scalebox{0.7}{ETTm2}}} \\
    
    \hline

    \multirow{1}{*}{\scalebox{0.8}{\shortstack{SourceData$\rightarrow$ETTm2}}} &
    \textbf{\scalebox{0.78}{0.280}} & \textbf{\scalebox{0.78}{0.315}} & 
    {\scalebox{0.78}{0.292}} & {\scalebox{0.78}{0.326}} & 
    {\scalebox{0.78}{0.316}} & {\scalebox{0.78}{0.359}} & 
    {\scalebox{0.78}{0.316}} & {\scalebox{0.78}{0.360}} & 
    {\scalebox{0.78}{0.266}} & {\scalebox{0.78}{0.326}} & 
    \textbf{\scalebox{0.78}{0.255}} & \textbf{\scalebox{0.78}{0.315}} & 
    {\scalebox{0.78}{0.308}} & {\scalebox{0.78}{0.346}} & 
    {\scalebox{0.78}{0.314}} & {\scalebox{0.78}{0.352}} \\
    
    \hline

    \multicolumn{1}{c|}{\rotatebox{0}{\scalebox{0.8}{SourceData}}} & 
    \multicolumn{2}{c|}{\rotatebox{0}{\scalebox{0.7}{ETT\{m1,m2,h2\}}}} & 
    \multicolumn{2}{c|}{\rotatebox{0}{\scalebox{0.7}{ETTm1}}} & 
    \multicolumn{2}{c|}{\rotatebox{0}{\scalebox{0.7}{ETTm2}}} & 
    \multicolumn{2}{c||}{\rotatebox{0}{\scalebox{0.7}{ETTh2}}} & 
    \multicolumn{2}{c|}{\rotatebox{0}{\scalebox{0.7}{ETTh1}}} & 
    \multicolumn{2}{c||}{\rotatebox{0}{\scalebox{0.7}{ETTh1}}} & 
    \multicolumn{2}{c|}{\rotatebox{0}{\scalebox{0.7}{ETTh1}}} & 
    \multicolumn{2}{c}{\rotatebox{0}{\scalebox{0.7}{ETTh1}}} \\
    
    \hline

    \multirow{1}{*}{\scalebox{0.8}{\shortstack{SourceData$\rightarrow$ETTh1}}} &
    \textbf{\scalebox{0.78}{0.461}} & \textbf{\scalebox{0.78}{0.449}} & 
    {\scalebox{0.78}{0.512}} & {\scalebox{0.78}{0.493}} & 
    {\scalebox{0.78}{0.536}} & {\scalebox{0.78}{0.499}} & 
    {\scalebox{0.78}{0.578}} & {\scalebox{0.78}{0.521}} & 
    {\scalebox{0.78}{0.427}} & \textbf{\scalebox{0.78}{0.426}} & 
    \textbf{\scalebox{0.78}{0.413}} & {\scalebox{0.78}{0.430}} & 
    {\scalebox{0.78}{0.693}} & {\scalebox{0.78}{0.568}} & 
    {\scalebox{0.78}{0.712}} & {\scalebox{0.78}{0.580}} \\
    
    \hline
    
    \multicolumn{1}{c|}{\rotatebox{0}{\scalebox{0.8}{SourceData}}} & 
    \multicolumn{2}{c|}{\rotatebox{0}{\scalebox{0.7}{ETT\{m1,m2,h1\}}}} & 
    \multicolumn{2}{c|}{\rotatebox{0}{\scalebox{0.7}{ETTm1}}} & 
    \multicolumn{2}{c|}{\rotatebox{0}{\scalebox{0.7}{ETTm2}}} & 
    \multicolumn{2}{c||}{\rotatebox{0}{\scalebox{0.7}{ETTh1}}} & 
    \multicolumn{2}{c|}{\rotatebox{0}{\scalebox{0.7}{ETTh2}}} & 
    \multicolumn{2}{c||}{\rotatebox{0}{\scalebox{0.7}{ETTh2}}} & 
    \multicolumn{2}{c|}{\rotatebox{0}{\scalebox{0.7}{ETTh2}}} & 
    \multicolumn{2}{c}{\rotatebox{0}{\scalebox{0.7}{ETTh2}}} \\
    
    \hline

    \multirow{1}{*}{\scalebox{0.8}{\shortstack{SourceData$\rightarrow$ETTh2}}} &
    \textbf{\scalebox{0.78}{0.371}} & \textbf{\scalebox{0.78}{0.384}} & 
    {\scalebox{0.78}{0.430}} & {\scalebox{0.78}{0.433}} & 
    {\scalebox{0.78}{0.408}} & {\scalebox{0.78}{0.422}} & 
    {\scalebox{0.78}{0.385}} & {\scalebox{0.78}{0.405}} & 
    {\scalebox{0.78}{0.354}} & {\scalebox{0.78}{0.394}} & 
    \textbf{\scalebox{0.78}{0.330}} & \textbf{\scalebox{0.78}{0.379}} & 
    {\scalebox{0.78}{0.413}} & {\scalebox{0.78}{0.441}} & 
    {\scalebox{0.78}{0.449}} & {\scalebox{0.78}{0.456}} \\
    
    \hline

    \bottomrule
  \end{tabular}
  \end{small}
}
\end{table*}
\begin{table*}[htpb]
  \vspace{-5pt}
  \caption{Comparison of the averaged performance from diverse prediction lengths ($\{96,192,336,720\}$) on \textbf{few-shot forecasting} task, where all samples of trainset are only partially available (5\%) in the training phase.}\label{tab:brief_forecasting_few_5}
  \centering
  \resizebox{0.8\columnwidth}{!}{
  \begin{small}
  \renewcommand{\multirowsetup}{\centering}
  \setlength{\tabcolsep}{1pt}
  \renewcommand\arraystretch{1}
  \begin{tabular}{ccc||cccccccccccc}
    \toprule
    \hline
    
    \multicolumn{1}{c}{\multirow{1}{*}{\scalebox{1.0}{Models}}} & 
    \multicolumn{2}{c}{\rotatebox{0}{\scalebox{0.8}{\textbf{\myformer}}}} & 
    \multicolumn{2}{c}{\rotatebox{0}{\scalebox{0.8}{OneFitsAll}}} & 
    \multicolumn{2}{c}{\rotatebox{0}{\scalebox{0.8}{DLinear}}} & 
    \multicolumn{2}{c}{\rotatebox{0}{\scalebox{0.8}{PatchTST}}} & 
    \multicolumn{2}{c}{\rotatebox{0}{\scalebox{0.8}{TimesNet}}} & 
    \multicolumn{2}{c}{\rotatebox{0}{\scalebox{0.8}{FEDformer}}} & 
    \multicolumn{2}{c}{\rotatebox{0}{\scalebox{0.8}{Autoformer}}} \\

    
    \cline{2-15}
    \multicolumn{1}{c}{\multirow{1}{*}{\scalebox{1.0}{Metric}}} & 
    \scalebox{0.78}{MSE} & \scalebox{0.78}{MAE} & 
    \scalebox{0.78}{MSE} & \scalebox{0.78}{MAE} & 
    \scalebox{0.78}{MSE} & \scalebox{0.78}{MAE} & 
    \scalebox{0.78}{MSE} & \scalebox{0.78}{MAE} & 
    \scalebox{0.78}{MSE} & \scalebox{0.78}{MAE} & 
    \scalebox{0.78}{MSE} & \scalebox{0.78}{MAE} & 
    \scalebox{0.78}{MSE} & \scalebox{0.78}{MAE} \\
    
    \hline
    \scalebox{0.78}{ETTm1} &
    \textbf{\scalebox{0.78}{0.394}} & \textbf{\scalebox{0.78}{0.381}} & 
    {\scalebox{0.78}{0.472}} & {\scalebox{0.78}{0.450}} & 
    {\scalebox{0.78}{0.400}} & {\scalebox{0.78}{0.417}} & 
    {\scalebox{0.78}{0.526}} & {\scalebox{0.78}{0.476}} & 
    {\scalebox{0.78}{0.717}} & {\scalebox{0.78}{0.561}} & 
    {\scalebox{0.78}{0.730}} & {\scalebox{0.78}{0.592}} & 
    {\scalebox{0.78}{0.796}} & {\scalebox{0.78}{0.620}} \\

    \hline
    \scalebox{0.78}{ETTm2} &
    \textbf{\scalebox{0.78}{0.267}} & \textbf{\scalebox{0.78}{0.306}} & 
    {\scalebox{0.78}{0.308}} & {\scalebox{0.78}{0.346}} & 
    {\scalebox{0.78}{0.399}} & {\scalebox{0.78}{0.426}} & 
    {\scalebox{0.78}{0.314}} & {\scalebox{0.78}{0.352}} & 
    {\scalebox{0.78}{0.344}} & {\scalebox{0.78}{0.372}} & 
    {\scalebox{0.78}{0.381}} & {\scalebox{0.78}{0.404}} & 
    {\scalebox{0.78}{0.388}} & {\scalebox{0.78}{0.433}} \\

    \hline
    \scalebox{0.78}{ETTh1} &
    \textbf{\scalebox{0.78}{0.521}} & \textbf{\scalebox{0.78}{0.449}} & 
    {\scalebox{0.78}{0.681}} & {\scalebox{0.78}{0.560}} & 
    {\scalebox{0.78}{0.750}} & {\scalebox{0.78}{0.611}} & 
    {\scalebox{0.78}{0.694}} & {\scalebox{0.78}{0.569}} & 
    {\scalebox{0.78}{0.925}} & {\scalebox{0.78}{0.647}} & 
    {\scalebox{0.78}{0.658}} & {\scalebox{0.78}{0.562}} & 
    {\scalebox{0.78}{0.722}} & {\scalebox{0.78}{0.598}} \\

    \hline
    \scalebox{0.78}{ETTh2} &
    \textbf{\scalebox{0.78}{0.373}} & \textbf{\scalebox{0.78}{0.381}} & 
    {\scalebox{0.78}{0.400}} & {\scalebox{0.78}{0.433}} & 
    {\scalebox{0.78}{0.827}} & {\scalebox{0.78}{0.615}} & 
    {\scalebox{0.78}{0.439}} & {\scalebox{0.78}{0.448}} & 
    {\scalebox{0.78}{0.463}} & {\scalebox{0.78}{0.454}} & 
    {\scalebox{0.78}{0.441}} & {\scalebox{0.78}{0.457}} & 
    {\scalebox{0.78}{0.470}} & {\scalebox{0.78}{0.489}} \\


    \hline
    \scalebox{0.78}{Weather} &
    \textbf{\scalebox{0.78}{0.233}} & \textbf{\scalebox{0.78}{0.262}} & 
    {\scalebox{0.78}{0.263}} & {\scalebox{0.78}{0.301}} & 
    {\scalebox{0.78}{0.263}} & {\scalebox{0.78}{0.308}} & 
    {\scalebox{0.78}{0.269}} & {\scalebox{0.78}{0.303}} & 
    {\scalebox{0.78}{0.298}} & {\scalebox{0.78}{0.318}} & 
    {\scalebox{0.78}{0.309}} & {\scalebox{0.78}{0.353}} & 
    {\scalebox{0.78}{0.310}} & {\scalebox{0.78}{0.353}} \\

    \hline
    \bottomrule
  \end{tabular}
  \end{small}
}
\end{table*}

\paragraph{Experimental setups:}
First, to verify that the proposed wave quantize module and tokenization strategy can fully stimulate the learning ability of the transformer-based backbone as an effective feature program, Table \ref{tab:brief_forecasting_full} and Table \ref{tab:brief_forecasting_zero_cross_single} show the results of the long-time forecasting task under full-data and zero-shot setting. 
Subsequently, to validate that the proposed approach can learn key information from multiple time series domains and efficiently migrate it to previously unseen target domains, we designed diverse adaption approaches for zero-shot learning. 
Specifically, Table \ref{tab:brief_forecasting_zero_cross_multi} shows the performance of the models in the target domain test set directly after pre-training in the single source domain and unified pre-training in multiple source domains, respectively. 
Finally, to illustrate the superior data efficiency of the model, Table \ref{tab:brief_forecasting_few_5} shows the results of the models under the few-shot 5\% settings, respectively. 

\paragraph{Analysis of results:}
In the full-data forecasting task shown in Table \ref{tab:brief_forecasting_full}, the proposed \myformer exhibits the best performance in \textbf{85\%} of the metrics. 
In the zero-shot forecasting task shown in Table \ref{tab:brief_forecasting_zero_cross_single}, the average MSE of the proposed \myformer is reduced by \textbf{26.2\%}, \textbf{19.6\%}, \textbf{14.3\%}, and \textbf{33.4\%} compared to the existing OneFitsAll, DLinear, PatchTST, and TimesNet, respectively.
As shown in Table \ref{tab:brief_forecasting_zero_cross_multi}, the performance of general unsupervised cross-domain migration on the zero-shot forecasting task would be far superior to that of single-domain pre-training, which indicates that our proposed wave quantize strategy could alleviate the negative migration on the time series. 
Besides, the performance of the proposed \myformer on the zero-shot task is much superior over the few-shot task of the existing SOTA models (Average MSE reduced by \textbf{31.4\%}), and archives comparable performance to the full-data results of the existing SOTA models (Average MSE difference is only \textbf{8.1\%}).
In the few-shot (5\%) forecasting task shown in Table \ref{tab:brief_forecasting_few_5}, the average MSE of the proposed \myformer is reduced by \textbf{15.8\%}, \textbf{32.2\%}, \textbf{20.2\%}, and \textbf{34.9\%} compared to the existing OneFitsAll, DLinear, PatchTST, and TimesNet, respectively.

\begin{figure}[t]
\begin{center}
\centerline{\includegraphics[width=0.7\columnwidth]{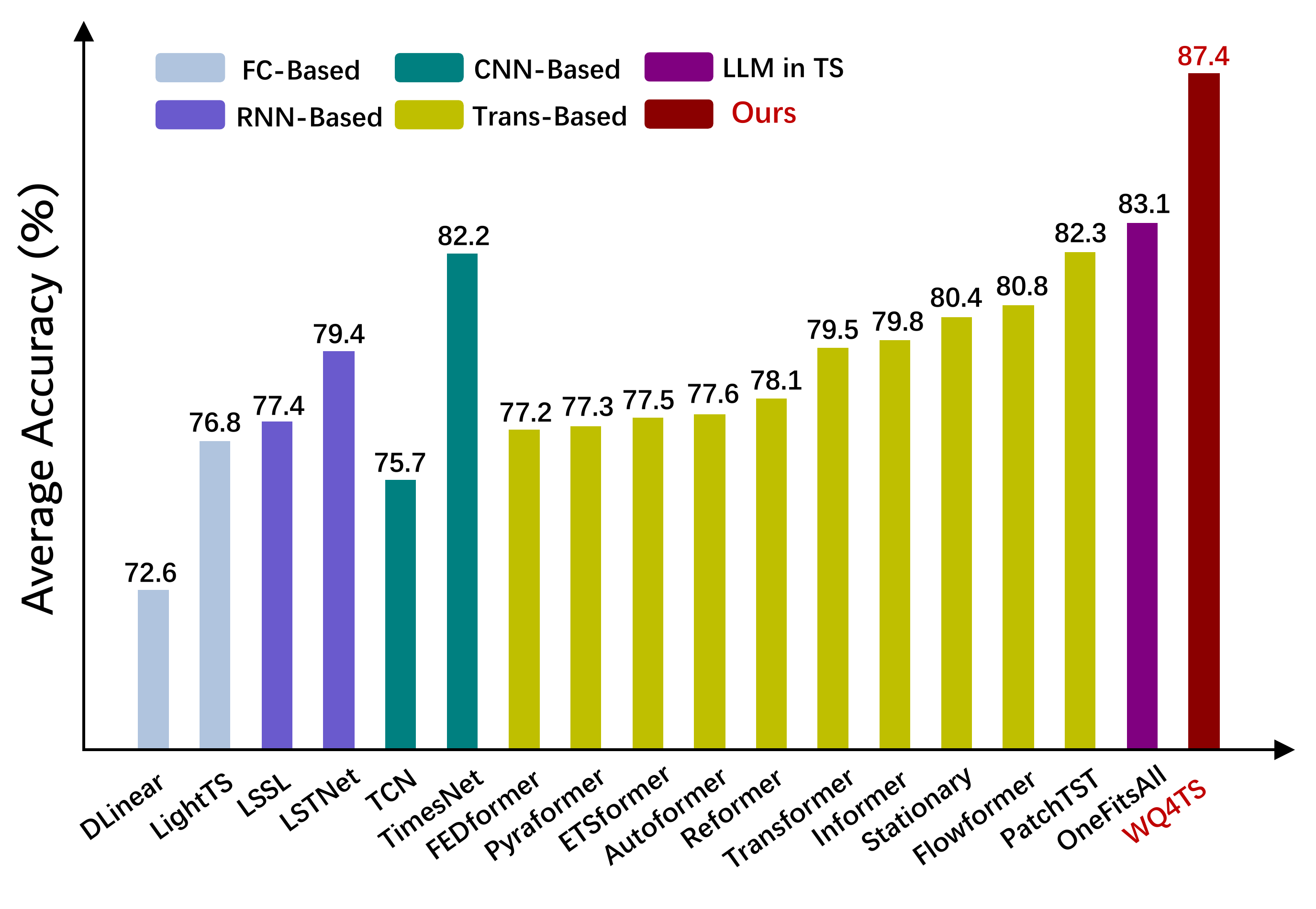}}
\caption{Model comparison in classification. The results are averaged from 35 subsets of UCR. The proposed \myformer achieves the best performance on the classification task under the full-data setting.}
\label{fig:classification_results}
\end{center}
\end{figure}

\subsection{Imputation} 
\begin{table*}[htbp]
  \caption{Comparison of the averaged performance from mask ratios ($\{12.5\%,25\%,37.5\%,50\%\}$) on \textbf{full-data imputation} task.}\label{tab:brief_imputation_full}
  \centering
  \resizebox{1.0\columnwidth}{!}{
  \begin{threeparttable}
  \begin{small}
  \renewcommand{\multirowsetup}{\centering}
  \setlength{\tabcolsep}{0.8pt}
  \begin{tabular}{c|cc|cc|cc|cc|cc|cc|cc|cc|cc|cc}
    \toprule
    \hline
    
    \multicolumn{1}{c}{\multirow{2}{*}{Models}} & 
    \multicolumn{2}{c}{\rotatebox{0}{\scalebox{0.8}{\textbf{\myformer}}}} &
    \multicolumn{2}{c}{\rotatebox{0}{\scalebox{0.8}{OneFitsAll}}} & 
    \multicolumn{2}{c}{\rotatebox{0}{\scalebox{0.8}{TimesNet}}} &
    \multicolumn{2}{c}{\rotatebox{0}{\scalebox{0.8}{PatchTST}}} & 
    \multicolumn{2}{c}{\rotatebox{0}{\scalebox{0.8}{ETSformer}}} &
    \multicolumn{2}{c}{\rotatebox{0}{\scalebox{0.8}{LightTS}}} & 
    \multicolumn{2}{c}{\rotatebox{0}{\scalebox{0.8}{DLinear}}} & 
    \multicolumn{2}{c}{\rotatebox{0}{\scalebox{0.8}{FEDformer}}} & 
    \multicolumn{2}{c}{\rotatebox{0}{\scalebox{0.8}{Stationary}}} & 
    \multicolumn{2}{c}{\rotatebox{0}{\scalebox{0.8}{Autoformer}}} \\
    
    \multicolumn{1}{c}{} & 
    \multicolumn{2}{c}{\scalebox{0.8}{(\textbf{Ours})}} & 
    \multicolumn{2}{c}{\scalebox{0.8}{\citeyearpar{Nie2023ATS}}} &  
    \multicolumn{2}{c}{\scalebox{0.8}{\citeyearpar{wu2023timesnet}}} &
    \multicolumn{2}{c}{\scalebox{0.8}{\citeyearpar{Nie2023ATS}}} &  
    \multicolumn{2}{c}{\scalebox{0.8}{\citeyearpar{wang2023micn}}} &
    \multicolumn{2}{c}{\scalebox{0.8}{\citeyearpar{Zhang2022LessIM}}} &
    \multicolumn{2}{c}{\scalebox{0.8}{\citeyearpar{Zeng2022AreTE}}} & 
    \multicolumn{2}{c}{\scalebox{0.8}{\citeyearpar{zhou2022fedformer}}} & 
    \multicolumn{2}{c}{\scalebox{0.8}{\citeyearpar{Liu2022NonstationaryTR}}} & 
    \multicolumn{2}{c}{\scalebox{0.8}{\citeyearpar{wu2021autoformer}}} \\
    
    \cline{2-21}
    
    \multicolumn{1}{c}{\scalebox{0.76}{MaskRatio}} & 
    \scalebox{0.76}{MSE} & \scalebox{0.76}{MAE} & 
    \scalebox{0.76}{MSE} & \scalebox{0.76}{MAE} & 
    \scalebox{0.76}{MSE} & \scalebox{0.76}{MAE} & 
    \scalebox{0.76}{MSE} & \scalebox{0.76}{MAE} & 
    \scalebox{0.76}{MSE} & \scalebox{0.76}{MAE} & 
    \scalebox{0.76}{MSE} & \scalebox{0.76}{MAE} & 
    \scalebox{0.76}{MSE} & \scalebox{0.76}{MAE} & 
    \scalebox{0.76}{MSE} & \scalebox{0.76}{MAE} & 
    \scalebox{0.76}{MSE} & \scalebox{0.76}{MAE} & 
    \scalebox{0.76}{MSE} & \scalebox{0.76}{MAE} \\
    
    \hline
    \scalebox{0.78}{ETTm1} &
    \textbf{\scalebox{0.78}{0.026}} &\textbf{\scalebox{0.78}{0.099}} & 
    {\scalebox{0.78}{0.028}} &{\scalebox{0.78}{0.105}} & 
    {\scalebox{0.78}{0.027}} &{\scalebox{0.78}{0.107}} & 
    {\scalebox{0.78}{0.047}} &{\scalebox{0.78}{0.140}} & 
    {\scalebox{0.78}{0.120}} &{\scalebox{0.78}{0.253}} & 
    {\scalebox{0.78}{0.104}} &{\scalebox{0.78}{0.218}} & 
    {\scalebox{0.78}{0.093}} &{\scalebox{0.78}{0.206}} & 
    {\scalebox{0.78}{0.062}} &{\scalebox{0.78}{0.177}} & 
    {\scalebox{0.78}{0.036}} &{\scalebox{0.78}{0.126}} & 
    {\scalebox{0.78}{0.051}} &{\scalebox{0.78}{0.150}} \\
    
    \hline
    \scalebox{0.78}{ETTm2} &
    \textbf{\scalebox{0.78}{0.020}} &{\scalebox{0.78}{0.085}} & 
    {\scalebox{0.78}{0.021}} &\textbf{\scalebox{0.78}{0.084}} & 
    {\scalebox{0.78}{0.022}} &{\scalebox{0.78}{0.088}} & 
    {\scalebox{0.78}{0.029}} &{\scalebox{0.78}{0.102}} & 
    {\scalebox{0.78}{0.208}} &{\scalebox{0.78}{0.327}} & 
    {\scalebox{0.78}{0.046}} &{\scalebox{0.78}{0.151}} & 
    {\scalebox{0.78}{0.096}} &{\scalebox{0.78}{0.208}} & 
    {\scalebox{0.78}{0.101}} &{\scalebox{0.78}{0.215}} & 
    {\scalebox{0.78}{0.026}} &{\scalebox{0.78}{0.099}} & 
    {\scalebox{0.78}{0.029}} &{\scalebox{0.78}{0.105}} \\
    
    \hline
    \scalebox{0.78}{ETTh1} &
    \textbf{\scalebox{0.78}{0.064}} &\textbf{\scalebox{0.78}{0.167}} & 
    {\scalebox{0.78}{0.069}} &{\scalebox{0.78}{0.173}} & 
    {\scalebox{0.78}{0.078}} &{\scalebox{0.78}{0.187}} & 
    {\scalebox{0.78}{0.115}} &{\scalebox{0.78}{0.224}} & 
    {\scalebox{0.78}{0.202}} &{\scalebox{0.78}{0.329}} & 
    {\scalebox{0.78}{0.284}} &{\scalebox{0.78}{0.373}} & 
    {\scalebox{0.78}{0.201}} &{\scalebox{0.78}{0.306}} & 
    {\scalebox{0.78}{0.117}} &{\scalebox{0.78}{0.246}} & 
    {\scalebox{0.78}{0.094}} &{\scalebox{0.78}{0.201}} & 
    {\scalebox{0.78}{0.103}} &{\scalebox{0.78}{0.214}} \\
    
    \hline
    \scalebox{0.78}{ETTh2} &
    \textbf{\scalebox{0.78}{0.047}} &\textbf{\scalebox{0.78}{0.138}} & 
    {\scalebox{0.78}{0.048}} &{\scalebox{0.78}{0.141}} & 
    {\scalebox{0.78}{0.049}} &{\scalebox{0.78}{0.146}} & 
    {\scalebox{0.78}{0.065}} &{\scalebox{0.78}{0.163}} & 
    {\scalebox{0.78}{0.367}} &{\scalebox{0.78}{0.436}} & 
    {\scalebox{0.78}{0.119}} &{\scalebox{0.78}{0.250}} & 
    {\scalebox{0.78}{0.142}} &{\scalebox{0.78}{0.259}} & 
    {\scalebox{0.78}{0.163}} &{\scalebox{0.78}{0.279}} & 
    {\scalebox{0.78}{0.053}} &{\scalebox{0.78}{0.152}} & 
    {\scalebox{0.78}{0.055}} &{\scalebox{0.78}{0.156}} \\
    
    \hline
    \scalebox{0.78}{Electricity} &
    \textbf{\scalebox{0.78}{0.053}} &\textbf{\scalebox{0.78}{0.147}} & 
    {\scalebox{0.78}{0.090}} &{\scalebox{0.78}{0.207}} & 
    {\scalebox{0.78}{0.092}} &{\scalebox{0.78}{0.210}} & 
    {\scalebox{0.78}{0.072}} &{\scalebox{0.78}{0.183}} & 
    {\scalebox{0.78}{0.214}} &{\scalebox{0.78}{0.339}} & 
    {\scalebox{0.78}{0.131}} &{\scalebox{0.78}{0.262}} & 
    {\scalebox{0.78}{0.132}} &{\scalebox{0.78}{0.260}} & 
    {\scalebox{0.78}{0.130}} &{\scalebox{0.78}{0.259}} & 
    {\scalebox{0.78}{0.100}} &{\scalebox{0.78}{0.218}} & 
    {\scalebox{0.78}{0.101}} &{\scalebox{0.78}{0.225}} \\
    
    \hline
    \scalebox{0.78}{Weather} &
    \textbf{\scalebox{0.78}{0.028}} &\textbf{\scalebox{0.78}{0.046}} & 
    {\scalebox{0.78}{0.031}} &{\scalebox{0.78}{0.056}} & 
    {\scalebox{0.78}{0.030}} &{\scalebox{0.78}{0.054}} & 
    {\scalebox{0.78}{0.060}} &{\scalebox{0.78}{0.144}} & 
    {\scalebox{0.78}{0.076}} &{\scalebox{0.78}{0.171}} & 
    {\scalebox{0.78}{0.055}} &{\scalebox{0.78}{0.117}} & 
    {\scalebox{0.78}{0.052}} &{\scalebox{0.78}{0.110}} & 
    {\scalebox{0.78}{0.099}} &{\scalebox{0.78}{0.203}} & 
    {\scalebox{0.78}{0.032}} &{\scalebox{0.78}{0.059}} & 
    {\scalebox{0.78}{0.031}} &{\scalebox{0.78}{0.057}} \\
    
    
    \hline
    \bottomrule
  \end{tabular}
    \end{small}
  \end{threeparttable}
   }
\end{table*}
\begin{table*}[htpb]
  \caption{Comparison of the averaged performance from diverse mask ratios ($\{12.5\%,25\%,37.5\%,50\%\}$) on \textbf{zero-shot imputation} task. Where $Source\rightarrow Traget$ indicates that the model is first pre-trained on the train set of the \emph{SourceDomain}, subsequently, the model parameters are frozen and predicted on the test set of the \emph{TargetDomain}.}\label{tab:brief_imputation_zero_cross_single}
  \centering
  \resizebox{0.9\columnwidth}{!}{
  \begin{small}
  \renewcommand{\multirowsetup}{\centering}
  \setlength{\tabcolsep}{1pt}
  \renewcommand\arraystretch{1}
  \begin{tabular}{ccc||cccccccccccc}
    \toprule
    \hline
    
    \multicolumn{1}{c}{\multirow{2}{*}{\scalebox{1.0}{Scenarios}}} & 
    \multicolumn{2}{c}{\rotatebox{0}{\scalebox{0.8}{\textbf{\myformer}}}} & 
    \multicolumn{2}{c}{\rotatebox{0}{\scalebox{0.8}{OneFitsAll}}} & 
    \multicolumn{2}{c}{\rotatebox{0}{\scalebox{0.8}{DLinear}}} & 
    \multicolumn{2}{c}{\rotatebox{0}{\scalebox{0.8}{PatchTST}}} & 
    \multicolumn{2}{c}{\rotatebox{0}{\scalebox{0.8}{TimesNet}}} & 
    \multicolumn{2}{c}{\rotatebox{0}{\scalebox{0.8}{FEDformer}}} & 
    \multicolumn{2}{c}{\rotatebox{0}{\scalebox{0.8}{Autoformer}}} \\

    
    \cline{2-15} &
    \scalebox{0.78}{MSE} & \scalebox{0.78}{MAE} & 
    \scalebox{0.78}{MSE} & \scalebox{0.78}{MAE} & 
    \scalebox{0.78}{MSE} & \scalebox{0.78}{MAE} & 
    \scalebox{0.78}{MSE} & \scalebox{0.78}{MAE} & 
    \scalebox{0.78}{MSE} & \scalebox{0.78}{MAE} & 
    \scalebox{0.78}{MSE} & \scalebox{0.78}{MAE} & 
    \scalebox{0.78}{MSE} & \scalebox{0.78}{MAE} \\
    \hline
    
    \multirow{1}{*}{\scalebox{0.8}{\shortstack{ETTm2$\rightarrow$ETTm1}}} &
    \textbf{\scalebox{0.78}{0.050}} & \textbf{\scalebox{0.78}{0.144}} & 
    {\scalebox{0.78}{0.767}} & {\scalebox{0.78}{0.549}} & 
    {\scalebox{0.78}{0.203}} & {\scalebox{0.78}{0.295}} & 
    {\scalebox{0.78}{0.099}} & {\scalebox{0.78}{0.191}} & 
    {\scalebox{0.78}{0.118}} & {\scalebox{0.78}{0.205}} & 
    {\scalebox{0.78}{0.762}} & {\scalebox{0.78}{0.655}} & 
    {\scalebox{0.78}{0.507}} & {\scalebox{0.78}{0.501}} \\ 

    \hline
    
    \multirow{1}{*}{\scalebox{0.8}{\shortstack{ETTm1$\rightarrow$ETTm2}}} &
    \textbf{\scalebox{0.78}{0.029}} & \textbf{\scalebox{0.78}{0.098}} & 
    {\scalebox{0.78}{0.145}} & {\scalebox{0.78}{0.256}} & 
    {\scalebox{0.78}{0.114}} & {\scalebox{0.78}{0.224}} & 
    {\scalebox{0.78}{0.058}} & {\scalebox{0.78}{0.149}} & 
    {\scalebox{0.78}{0.093}} & {\scalebox{0.78}{0.216}} & 
    {\scalebox{0.78}{2.140}} & {\scalebox{0.78}{1.113}} & 
    {\scalebox{0.78}{1.342}} & {\scalebox{0.78}{0.842}} \\ 

    \hline
    
    \multirow{1}{*}{\scalebox{0.8}{\shortstack{ETTm1$\rightarrow$ETTh1}}} &
    \textbf{\scalebox{0.78}{0.176}} & \textbf{\scalebox{0.78}{0.274}} & 
    {\scalebox{0.78}{0.854}} & {\scalebox{0.78}{0.602}} & 
    {\scalebox{0.78}{0.397}} & {\scalebox{0.78}{0.424}} & 
    {\scalebox{0.78}{0.313}} & {\scalebox{0.78}{0.366}} & 
    {\scalebox{0.78}{0.327}} & {\scalebox{0.78}{0.395}} & 
    {\scalebox{0.78}{1.074}} & {\scalebox{0.78}{0.787}} & 
    {\scalebox{0.78}{0.956}} & {\scalebox{0.78}{0.725}} \\ 

    \hline
    
    \multirow{1}{*}{\scalebox{0.8}{\shortstack{ETTm1$\rightarrow$ETTh2}}} &
    \textbf{\scalebox{0.78}{0.064}} & \textbf{\scalebox{0.78}{0.160}} & 
    {\scalebox{0.78}{0.245}} & {\scalebox{0.78}{0.333}} & 
    {\scalebox{0.78}{0.160}} & {\scalebox{0.78}{0.277}} & 
    {\scalebox{0.78}{0.079}} & {\scalebox{0.78}{0.184}} & 
    {\scalebox{0.78}{0.109}} & {\scalebox{0.78}{0.238}} & 
    {\scalebox{0.78}{2.796}} & {\scalebox{0.78}{1.266}} & 
    {\scalebox{0.78}{2.473}} & {\scalebox{0.78}{1.206}} \\ 

    \hline
    
    \multirow{1}{*}{\scalebox{0.8}{\shortstack{ETTm1$\rightarrow$Exchange}}} &
    \textbf{\scalebox{0.78}{0.003}} & \textbf{\scalebox{0.78}{0.031}} & 
    {\scalebox{0.78}{0.027}} & {\scalebox{0.78}{0.117}} & 
    {\scalebox{0.78}{0.358}} & {\scalebox{0.78}{0.437}} & 
    {\scalebox{0.78}{0.006}} & {\scalebox{0.78}{0.044}} & 
    {\scalebox{0.78}{0.045}} & {\scalebox{0.78}{0.150}} & 
    {\scalebox{0.78}{3.107}} & {\scalebox{0.78}{1.440}} & 
    {\scalebox{0.78}{2.904}} & {\scalebox{0.78}{1.382}} \\ 

    \hline
    
    \multirow{1}{*}{\scalebox{0.8}{\shortstack{ETTm1$\rightarrow$Weather}}} &
    \textbf{\scalebox{0.78}{0.030}} & \textbf{\scalebox{0.78}{0.043}} & 
    {\scalebox{0.78}{0.103}} & {\scalebox{0.78}{0.160}} & 
    {\scalebox{0.78}{0.174}} & {\scalebox{0.78}{0.283}} & 
    {\scalebox{0.78}{0.065}} & {\scalebox{0.78}{0.099}} & 
    {\scalebox{0.78}{0.132}} & {\scalebox{0.78}{0.188}} & 
    {\scalebox{0.78}{0.999}} & {\scalebox{0.78}{0.779}} & 
    {\scalebox{0.78}{1.000}} & {\scalebox{0.78}{0.788}} \\

    \hline
    \bottomrule
  \end{tabular}
  \end{small}
}
\end{table*}

\paragraph{Experimental setups:}
The imputation task, that is predicting the masked portion of the original series based on the unmasked portion. Similarly, to the forecasting task, Table \ref{tab:brief_imputation_full} and Table \ref{tab:brief_imputation_zero_cross_single} show the experimental results for the imputation task under the full-data and zero-shot settings, respectively.

\paragraph{Analysis of results:}
In the full-data imputation task shown in Table \ref{tab:brief_imputation_full}, the proposed \myformer exhibits the best performance in \textbf{91.6\%} of the metrics. 
In the zero-shot imputation task shown in Table \ref{tab:brief_imputation_zero_cross_single}, the average MSE values of the proposed \myformer are reduced by \textbf{83.6\%}, \textbf{74.9\%}, \textbf{43.2\%}, and \textbf{57.3\%} compared to the existing OneFitsAll, DLinear, PatchTST, and TimesNet, respectively.

\subsection{Classification}

\begin{table*}[htpb]
  \caption{Comparison of the accuracy and F1-score on \textbf{few-shot classification} task. Where \emph{Scenario-i: Source-i}$\rightarrow$\emph{Traget-i} ($i\in[0,..,7]$) indicates that the model is first pre-trained in the \emph{SourceDomain}, subsequently, the parameters are fine-tuned in partial (5\%/10\%) samples of the \emph{TargetDomain} and finally predicted on the \emph{TargetDomain}.}\label{tab:brief_classification_few}
  \centering
  \resizebox{0.9\columnwidth}{!}{
  \begin{small}
  \renewcommand{\multirowsetup}{\centering}
  \setlength{\tabcolsep}{1pt}
  \renewcommand\arraystretch{1.0}
  \begin{tabular}{c|c|cc|cc|cc|cc|cc|cc|cc}
    \toprule
    \hline
    
    \multicolumn{2}{c}{\multirow{1}{*}{\scalebox{0.8}{Scenarios}}} & 
    \multicolumn{2}{c}{\rotatebox{0}{\scalebox{0.8}{Scenario-1}}} & 
    \multicolumn{2}{c}{\rotatebox{0}{\scalebox{0.8}{Scenario-2}}} & 
    \multicolumn{2}{c}{\rotatebox{0}{\scalebox{0.8}{Scenario-3}}} & 
    \multicolumn{2}{c}{\rotatebox{0}{\scalebox{0.8}{Scenario-4}}} & 
    \multicolumn{2}{c}{\rotatebox{0}{\scalebox{0.8}{Scenario-5}}} & 
    \multicolumn{2}{c}{\rotatebox{0}{\scalebox{0.8}{Scenario-6}}} & 
    \multicolumn{2}{c}{\rotatebox{0}{\scalebox{0.8}{Scenario-7}}} \\

    \hline
    \scalebox{0.78}{Task} & \scalebox{0.78}{Models} &
    \scalebox{0.78}{ACC} & \scalebox{0.78}{F1} & 
    \scalebox{0.78}{ACC} & \scalebox{0.78}{F1} & 
    \scalebox{0.78}{ACC} & \scalebox{0.78}{F1} & 
    \scalebox{0.78}{ACC} & \scalebox{0.78}{F1} & 
    \scalebox{0.78}{ACC} & \scalebox{0.78}{F1} & 
    \scalebox{0.78}{ACC} & \scalebox{0.78}{F1} & 
    \scalebox{0.78}{ACC} & \scalebox{0.78}{F1} \\
    
    \hline

    \multirow{3}{*}[-0.5ex]{\rotatebox{90}{\scalebox{0.78}{\shortstack{Full Data}}}}
    & \scalebox{0.78}{TimesNet} & 
    \textbf{\scalebox{0.78}{85.22}} & \textbf{\scalebox{0.78}{82.99}} & 
    \textbf{\scalebox{0.78}{97.96}} & \textbf{\scalebox{0.78}{97.64}} & 
    {\scalebox{0.78}{76.37}} & {\scalebox{0.78}{76.61}} & 
    {\scalebox{0.78}{78.10}} & {\scalebox{0.78}{75.26}} & 
    \textbf{\scalebox{0.78}{97.18}} & \textbf{\scalebox{0.78}{97.18}} & 
    {\scalebox{0.78}{95.00}} & {\scalebox{0.78}{95.00}} &
    {\scalebox{0.78}{86.34}} & {\scalebox{0.78}{71.11}} \\
    & \scalebox{0.78}{PatchTST} & 
    {\scalebox{0.78}{78.69}} & {\scalebox{0.78}{75.39}} & 
    {\scalebox{0.78}{95.63}} & {\scalebox{0.78}{94.80}} &
    {\scalebox{0.78}{69.55}} & {\scalebox{0.78}{72.08}} & 
    {\scalebox{0.78}{83.13}} & {\scalebox{0.78}{78.10}} & 
    {\scalebox{0.78}{96.60}} & {\scalebox{0.78}{96.61}} & 
    \textbf{\scalebox{0.78}{98.89}} & \textbf{\scalebox{0.78}{98.89}} &
    \textbf{\scalebox{0.78}{89.27}} & \textbf{\scalebox{0.78}{84.07}} \\
    & \scalebox{0.78}{OneFitsAll} & 
    {\scalebox{0.78}{79.73}} & {\scalebox{0.78}{77.20}} & 
    {\scalebox{0.78}{96.79}} & {\scalebox{0.78}{96.32}} & 
    \textbf{\scalebox{0.78}{79.37}} & \textbf{\scalebox{0.78}{80.62}} & 
    {\scalebox{0.78}{83.03}} & {\scalebox{0.78}{79.50}} & 
    {\scalebox{0.78}{96.99}} & {\scalebox{0.78}{96.99}} & 
    {\scalebox{0.78}{98.33}} & {\scalebox{0.78}{98.34}} &
    {\scalebox{0.78}{88.78}} & {\scalebox{0.78}{83.12}} \\

    \hline
    \multirow{5}{*}[-1ex]{\rotatebox{90}{\scalebox{0.78}{\shortstack{Few-shot \\ (10\%)}}}}
    & \scalebox{0.78}{Random init.} & 
    {\scalebox{0.78}{75.26}} & {\scalebox{0.78}{70.60}} & 
    {\scalebox{0.78}{66.79}} & {\scalebox{0.78}{66.36}} & 
    {\scalebox{0.78}{58.51}} & {\scalebox{0.78}{46.53}} & 
    {\scalebox{0.78}{47.39}} & {\scalebox{0.78}{29.49}} & 
    {\scalebox{0.78}{78.92}} & {\scalebox{0.78}{79.07}} & 
    {\scalebox{0.78}{50.00}} & {\scalebox{0.78}{33.33}} &
    {\scalebox{0.78}{41.46}} & {\scalebox{0.78}{18.58}} \\
    \cline{2-16}
    & \scalebox{0.78}{TimesNet} & 
    {\scalebox{0.78}{77.32}} & {\scalebox{0.78}{74.44}} & 
    {\scalebox{0.78}{91.25}} & {\scalebox{0.78}{89.43}} & 
    {\scalebox{0.78}{64.23}} & {\scalebox{0.78}{66.07}} & 
    {\scalebox{0.78}{83.13}} & {\scalebox{0.78}{66.52}} & 
    {\scalebox{0.78}{86.59}} & {\scalebox{0.78}{86.69}} & 
    {\scalebox{0.78}{83.89}} & {\scalebox{0.78}{83.83}} &
    {\scalebox{0.78}{85.85}} & {\scalebox{0.78}{72.08}} \\
    & \scalebox{0.78}{PatchTST} & 
    {\scalebox{0.78}{78.69}} & {\scalebox{0.78}{74.57}} & 
    {\scalebox{0.78}{80.76}} & {\scalebox{0.78}{75.36}} & 
    {\scalebox{0.78}{63.56}} & {\scalebox{0.78}{64.40}} & 
    {\scalebox{0.78}{48.19}} & {\scalebox{0.78}{45.63}} & 
    {\scalebox{0.78}{87.37}} & {\scalebox{0.78}{87.59}} & 
    {\scalebox{0.78}{90.00}} & {\scalebox{0.78}{89.96}} &
    {\scalebox{0.78}{87.80}} & {\scalebox{0.78}{80.67}} \\
    & \scalebox{0.78}{OneFitsAll} & 
    {\scalebox{0.78}{77.66}} & {\scalebox{0.78}{74.34}} & 
    {\scalebox{0.78}{78.13}} & {\scalebox{0.78}{71.24}} & 
    {\scalebox{0.78}{64.56}} & {\scalebox{0.78}{64.03}} & 
    {\scalebox{0.78}{83.13}} & {\scalebox{0.78}{67.34}} & 
    {\scalebox{0.78}{94.66}} & {\scalebox{0.78}{94.68}} & 
    {\scalebox{0.78}{97.23}} & {\scalebox{0.78}{97.23}} &
    {\scalebox{0.78}{87.32}} & {\scalebox{0.78}{71.62}} \\
    & \scalebox{0.78}{\myformer} & 
    \textbf{\scalebox{0.78}{86.25}} & \textbf{\scalebox{0.78}{84.65}} & 
    \textbf{\scalebox{0.78}{94.17}} & \textbf{\scalebox{0.78}{93.62}} & 
    \textbf{\scalebox{0.78}{91.51}} & \textbf{\scalebox{0.78}{91.38}} & 
    \textbf{\scalebox{0.78}{93.17}} & \textbf{\scalebox{0.78}{91.03}} & 
    \textbf{\scalebox{0.78}{96.31}} & \textbf{\scalebox{0.78}{96.31}} & 
    \textbf{\scalebox{0.78}{98.89}} & \textbf{\scalebox{0.78}{98.89}} &
    \textbf{\scalebox{0.78}{89.27}} & \textbf{\scalebox{0.78}{84.61}} \\

    \hline
    \multirow{5}{*}[-1ex]{\rotatebox{90}{\scalebox{0.78}{\shortstack{Few-shot \\ (5\%)}}}}
    & \scalebox{0.78}{Random init.} & 
    {\scalebox{0.78}{68.38}} & {\scalebox{0.78}{48.30}} & 
    {\scalebox{0.78}{27.41}} & {\scalebox{0.78}{28.16}} & 
    {\scalebox{0.78}{47.09}} & {\scalebox{0.78}{37.94}} & 
    {\scalebox{0.78}{35.74}} & {\scalebox{0.78}{24.63}} & 
    {\scalebox{0.78}{70.86}} & {\scalebox{0.78}{71.21}} & 
    {\scalebox{0.78}{50.00}} & {\scalebox{0.78}{33.33}} &
    {\scalebox{0.78}{36.59}} & {\scalebox{0.78}{16.39}} \\
    \cline{2-16}
    & \scalebox{0.78}{TimesNet} & 
    {\scalebox{0.78}{78.01}} & {\scalebox{0.78}{73.61}} & 
    {\scalebox{0.78}{27.41}} & {\scalebox{0.78}{28.16}} & 
    {\scalebox{0.78}{60.23}} & {\scalebox{0.78}{57.26}} & 
    {\scalebox{0.78}{35.74}} & {\scalebox{0.78}{26.61}} & 
    {\scalebox{0.78}{73.28}} & {\scalebox{0.78}{74.01}} & 
    {\scalebox{0.78}{78.89}} & {\scalebox{0.78}{78.85}} &
    {\scalebox{0.78}{85.85}} & {\scalebox{0.78}{63.18}} \\
    & \scalebox{0.78}{PatchTST} & 
    {\scalebox{0.78}{78.35}} & {\scalebox{0.78}{74.52}} & 
    {\scalebox{0.78}{81.05}} & {\scalebox{0.78}{75.44}} & 
    {\scalebox{0.78}{57.07}} & {\scalebox{0.78}{42.98}} & 
    {\scalebox{0.78}{37.35}} & {\scalebox{0.78}{30.04}} & 
    {\scalebox{0.78}{60.93}} & {\scalebox{0.78}{61.27}} & 
    {\scalebox{0.78}{82.78}} & {\scalebox{0.78}{82.73}} &
    {\scalebox{0.78}{87.32}} & {\scalebox{0.78}{79.84}} \\
    & \scalebox{0.78}{OneFitsAll} & 
    {\scalebox{0.78}{76.98}} & {\scalebox{0.78}{72.85}} & 
    {\scalebox{0.78}{27.41}} & {\scalebox{0.78}{28.16}} & 
    {\scalebox{0.78}{57.07}} & {\scalebox{0.78}{42.98}} & 
    {\scalebox{0.78}{35.74}} & {\scalebox{0.78}{24.63}} & 
    {\scalebox{0.78}{83.19}} & {\scalebox{0.78}{83.65}} & 
    {\scalebox{0.78}{90.00}} & {\scalebox{0.78}{89.99}} &
    {\scalebox{0.78}{84.39}} & {\scalebox{0.78}{58.62}} \\
    & \scalebox{0.78}{\myformer} & 
    \textbf{\scalebox{0.78}{84.54}} & \textbf{\scalebox{0.78}{82.11}} & 
    \textbf{\scalebox{0.78}{91.55}} & \textbf{\scalebox{0.78}{89.76}} & 
    \textbf{\scalebox{0.78}{85.03}} & \textbf{\scalebox{0.78}{86.09}} & 
    \textbf{\scalebox{0.78}{83.85}} & \textbf{\scalebox{0.78}{81.93}} & 
    \textbf{\scalebox{0.78}{94.85}} & \textbf{\scalebox{0.78}{94.85}} & 
    \textbf{\scalebox{0.78}{95.00}} & \textbf{\scalebox{0.78}{95.00}} &
    \textbf{\scalebox{0.78}{87.80}} & \textbf{\scalebox{0.78}{82.48}} \\
    
    \hline
    \bottomrule
  \end{tabular}
  \end{small}
}
\end{table*}

\begin{table*}[htpb]
  \caption{Comparison of the ablation experiment, where \emph{w/o WQ} indicates the model without the proposed \emph{wave quantize} module. To guarantee fairness, \emph{w/o WQ} and \emph{\myformer} have the same model structure and parameter size.}\label{tab:ablation_brief}
  \centering
  \resizebox{0.8\columnwidth}{!}{
  \begin{small}
  \renewcommand{\multirowsetup}{\centering}
  \setlength{\tabcolsep}{1.pt}
  \renewcommand\arraystretch{1.}
  \begin{tabular}{c|cccc|cccc|cccc|cccc}
    \hline
    \hline
    \multicolumn{1}{c|}{\multirow{3}{*}{\scalebox{0.8}{Variant}}} & 
    \multicolumn{4}{c|}{\multirow{1}{*}{\rotatebox{0}{\scalebox{0.8}{full datal}}}} & 
    \multicolumn{4}{c|}{\multirow{1}{*}{\rotatebox{0}{\scalebox{0.8}{few shot (10\%)}}}} & 
    \multicolumn{4}{c|}{\multirow{1}{*}{\rotatebox{0}{\scalebox{0.8}{zero shot (single)}}}} & 
    \multicolumn{4}{c}{\multirow{1}{*}{\rotatebox{0}{\scalebox{0.8}{zero shot (multi)}}}} \\
    \cline{2-17}
    
    &
    \multicolumn{2}{c}{\multirow{1}{*}{\rotatebox{0}{\scalebox{0.8}{ETTh1}}}} & 
    \multicolumn{2}{c|}{\multirow{1}{*}{\rotatebox{0}{\scalebox{0.8}{ETTm1}}}} & 
    \multicolumn{2}{c}{\multirow{1}{*}{\rotatebox{0}{\scalebox{0.8}{ETTh1}}}} & 
    \multicolumn{2}{c|}{\multirow{1}{*}{\rotatebox{0}{\scalebox{0.8}{ETTm1}}}} & 
    \multicolumn{2}{c}{\multirow{1}{*}{\rotatebox{0}{\scalebox{0.8}{ETTh1}}}} & 
    \multicolumn{2}{c|}{\multirow{1}{*}{\rotatebox{0}{\scalebox{0.8}{ETTm1}}}} & 
    \multicolumn{2}{c}{\multirow{1}{*}{\rotatebox{0}{\scalebox{0.8}{ETTh1}}}} & 
    \multicolumn{2}{c}{\multirow{1}{*}{\rotatebox{0}{\scalebox{0.8}{ETTm1}}}} \\
    \cline{2-17}
    
    &
    \scalebox{0.78}{MSE} & \scalebox{0.78}{MAE} & 
    \scalebox{0.78}{MSE} & \scalebox{0.78}{MAE} & 
    \scalebox{0.78}{MSE} & \scalebox{0.78}{MAE} & 
    \scalebox{0.78}{MSE} & \scalebox{0.78}{MAE} & 
    \scalebox{0.78}{MSE} & \scalebox{0.78}{MAE} & 
    \scalebox{0.78}{MSE} & \scalebox{0.78}{MAE} & 
    \scalebox{0.78}{MSE} & \scalebox{0.78}{MAE} & 
    \scalebox{0.78}{MSE} & \scalebox{0.78}{MAE} \\

    \hline

    \multirow{1}{*}{\rotatebox{0}{\scalebox{0.95}{w/o TP}}} & 
    {\scalebox{0.78}{0.582}} & {\scalebox{0.78}{0.542}} &
    {\scalebox{0.78}{0.513}} & {\scalebox{0.78}{0.480}} &
    {\scalebox{0.78}{0.643}} & {\scalebox{0.78}{0.581}} &
    {\scalebox{0.78}{0.524}} & {\scalebox{0.78}{0.489}} &
    {\scalebox{0.78}{0.704}} & {\scalebox{0.78}{0.628}} &
    {\scalebox{0.78}{0.599}} & {\scalebox{0.78}{0.548}} &
    {\scalebox{0.78}{0.677}} & {\scalebox{0.78}{0.606}} &
    {\scalebox{0.78}{0.623}} & {\scalebox{0.78}{0.564}} \\

    \hline
    
    \multirow{1}{*}{\rotatebox{0}{\scalebox{0.95}{\myformer}}} & 
    {\scalebox{0.78}{0.407}} & {\scalebox{0.78}{0.418}} &
    {\scalebox{0.78}{0.368}} & {\scalebox{0.78}{0.369}} &
    {\scalebox{0.78}{0.477}} & {\scalebox{0.78}{0.448}} &
    {\scalebox{0.78}{0.383}} & {\scalebox{0.78}{0.380}} &
    {\scalebox{0.78}{0.512}} & {\scalebox{0.78}{0.493}} &
    {\scalebox{0.78}{0.434}} & {\scalebox{0.78}{0.429}} &
    {\scalebox{0.78}{0.461}} & {\scalebox{0.78}{0.449}} &
    {\scalebox{0.78}{0.411}} & {\scalebox{0.78}{0.416}} \\
    
    \hline
    
    \hline
    \multirow{1}{*}{\rotatebox{0}{\scalebox{0.8}{loss $\downarrow$ (\%)}}} &
    {\scalebox{0.78}{42.9}} & {\scalebox{0.78}{29.6}} & 
    {\scalebox{0.78}{39.4}} & {\scalebox{0.78}{30.1}} & 
    {\scalebox{0.78}{34.8}} & {\scalebox{0.78}{29.7}} & 
    {\scalebox{0.78}{36.8}} & {\scalebox{0.78}{28.7}} & 
    {\scalebox{0.78}{37.5}} & {\scalebox{0.78}{27.4}} & 
    {\scalebox{0.78}{38.0}} & {\scalebox{0.78}{27.7}} & 
    {\scalebox{0.78}{46.9}} & {\scalebox{0.78}{34.9}} & 
    {\scalebox{0.78}{51.6}} & {\scalebox{0.78}{35.6}} \\
    
    \hline
    
    \hline
    \hline
  \end{tabular}
  \end{small}
}
\end{table*}
\begin{table}[htbp]
  \caption{The detailed information of Scenarios in Table \ref{tab:brief_classification_few}. All of the sub-datasets involved in seven  scenarios are derived from the UCR dataset presented in the section \ref{sec:benchmark}.}\label{tab:seven_scenarios}
  \vskip 0.05in
  \centering
  \resizebox{0.7\columnwidth}{!}{
  \begin{threeparttable}
  \begin{small}
  \renewcommand{\multirowsetup}{\centering}
  \setlength{\tabcolsep}{3.8pt}
  \begin{tabular}{c|c|c}
  \hline
  \hline
    Cross-domain & Source domain & Target domain \\
    
    \hline
    Scenario-1 & DistalPhalanxTW & ProximalPhalanxOutlineCorrect \\
    \hline
    Scenario-2 & SonyAIBORobotSurface2 & Chinatown \\
    \hline
    Scenario-3 & SonyAIBORobotSurface2 & SonyAIBORobotSurface1 \\
    \hline
    Scenario-4 & PigArtPressure & InsectEPGRegularTrain \\
    \hline
    Scenario-5 & Earthquakes & ItalyPowerDemand \\
    \hline
    Scenario-6 & BeetleFly & PowerCons \\
    \hline
    Scenario-7 & EOGHorizontalSignal & ProximalPhalanxOutlineAgeGroup \\

    \hline
    \hline
    \end{tabular}
    \end{small}
  \end{threeparttable}
  }
\vspace{-10pt}
\end{table}

\paragraph{Experimental setups:}
First, Fig. \ref{fig:classification_results} shows the average accuracy of the existing baselines and the proposed \myformer on all 35 classification datasets, where each coloration represents multiple models established by the specific backbone. 
In addition, Table \ref{tab:brief_classification_few} demonstrates the cross-domain migration capability of the proposed \myformer, where the upper and lower parts represent the accuracy and F1 score of all existing baselines in the full-data and few-shot (5\% and 10\%) settings, respectively.
Specifically, each \emph{scenario-i} represents a specific tuple of source and target domain, where each domain contains trainset and testset. Further, in the full data setting, the model is trained and tested on the target trainset and target testset. In the few-shot setting, the model is first pre-trained and tested on the source trainset and source testset, and the all-parameters will be fine-tuned in portion data (5\% and 10\%) of the target trainset. The detailed information of the 7 scenarios shown therein is given in the Table \ref{tab:seven_scenarios}.

\paragraph{Analysis of results:}
Fig. \ref{fig:classification_results} demonstrates the average accuracy of the proposed approach and the existing models on all 35 classification datasets, with the proposed \myformer showing the best performance. 
Table \ref{tab:brief_classification_few} demonstrates the unsupervised cross-domain migration capability of the proposed method. Besides, it is noteworthy that the proposed approach archives comparable or superior performance in few-shot settings over state-of-the-art models under full-data settings, in all seven scenarios. Concretely, the average accuracy of the proposed \myformer is increased by \textbf{23.9\%}, \textbf{19.6\%}, \textbf{26.1\%}, and \textbf{30.0\%} compared to the existing OneFitsAll, PatchTST, TimesNet, and MICN respectively.

\subsection{Ablation Study}
To elaborate on the property of our proposed \myformer, we conduct detailed ablations on model architecture.
As shown in Table \ref{tab:ablation_brief}, we find that removing the \emph{wave quantize} module in \myformer will cause significant performance degradation. These results may come from that the proposed feature program will improve the generalization capability of models to learn representation from complex series and migrate it to never-before-seen domains. 
Specifically, for fairness purposes, the ablation model \emph{w/o WQ} is designed to have the same Encoder and OutputLayer structure as \myformer, using only 1D Convolution in place of the proposed \emph{wave quantize} module, which can be expressed as $\textit{Conv}:\mathbb{R}^{1\times l}\mapsto\mathbb{R}^{\lambda\times l}$.
From Table \ref{tab:ablation_brief}, we can find that the performance of \emph{w/o WQ} degenerates \textbf{35.5\%} in the full-data task, degenerates \textbf{32.5\%} in the few-shot task, degenerates \textbf{32.7\%} and \textbf{42.3\%} in the zero-shot task under single-domain and multi-domain respectively. Similar results are found in other datasets, which indicate the advantages of our design.

\begin{table*}[htpb]
\caption{Improvements of \textbf{Wave Quantize Module (WQ)} over different models with prediction lengths $F = 336$, and fixed lookback length $T = 336$.}\label{tab:boosting}
\centering
\resizebox{0.9\columnwidth}{!}{
\begin{small}
\renewcommand{\multirowsetup}{\centering}
\setlength{\tabcolsep}{1.pt}
\renewcommand\arraystretch{1.1}
\begin{tabular}{c|c|cc|cc|cc|cc|cc|cc}
\toprule
\hline

& Models & 
\multicolumn{2}{c}{\scalebox{1.0}{OneFitsAll}} & 
\multicolumn{2}{c}{\scalebox{1.0}{DLinear}} & 
\multicolumn{2}{c}{\scalebox{1.0}{PatchTST}} & 
\multicolumn{2}{c}{\scalebox{1.0}{TimesNet}} & 
\multicolumn{2}{c}{\scalebox{1.0}{FEDformer}} & 
\multicolumn{2}{c}{\scalebox{1.0}{Autoformer}} \\
\cline{3-14}

& Metric 
& MSE & MAE & MSE & MAE & MSE & MAE & MSE & MAE & MSE & MAE & MSE & MAE\\
\hline

\multirow{3}{*}{\scalebox{1.0}{\shortstack{ETTm2\\ $\downarrow$ \\ETTm1}}}
& Original & 
0.790 & 0.579 & 
0.516 & 0.473 & 
0.596 & 0.508 & 
0.857 & 0.599 & 
0.718 & 0.564 & 
0.722 & 0.566 \\
& + \textbf{WQ} & 
\textbf{0.621} & \textbf{0.497} & 
\textbf{0.425} & \textbf{0.419} & 
\textbf{0.490} & \textbf{0.461} & 
\textbf{0.583} & \textbf{0.490} & 
\textbf{0.526} & \textbf{0.470} & 
\textbf{0.518} & \textbf{0.469} \\
\cline{2-14}
& \textbf{Improve} & 
\textcolor{red}{\textbf{21.39\%}} & \textcolor{red}{\textbf{14.16\%}} & 
\textcolor{red}{\textbf{17.64\%}} & \textcolor{red}{\textbf{11.42\%}} & 
\textcolor{red}{\textbf{17.79\%}} & \textcolor{red}{\textbf{9.25\%}} & 
\textcolor{red}{\textbf{31.97\%}} & \textcolor{red}{\textbf{18.21\%}} & 
\textcolor{red}{\textbf{26.74\%}} & \textcolor{red}{\textbf{16.57\%}} & 
\textcolor{red}{\textbf{28.26\%}} & \textcolor{red}{\textbf{17.14\%}} \\
\cline{1-14}

\multirow{3}{*}{\scalebox{1.0}{\shortstack{ETTm1\\ $\downarrow$ \\ETTm2}}}
& Original & 
0.342 & 0.369 & 
0.360 & 0.410 & 
0.325 & 0.361 & 
0.357 & 0.384 & 
0.321 & 0.360 & 
0.325 & 0.365 \\
& + \textbf{WQ} & 
\textbf{0.335} & \textbf{0.352} & 
\textbf{0.316} & \textbf{0.335} & 
\textbf{0.287} & \textbf{0.319} & 
\textbf{0.321} & \textbf{0.342} & 
\textbf{0.289} & \textbf{0.315} & 
\textbf{0.293} & \textbf{0.316} \\
\cline{2-14}
& \textbf{Improve} & 
\textcolor{red}{\textbf{2.05\%}} & \textcolor{red}{\textbf{4.61\%}} & 
\textcolor{red}{\textbf{12.30\%}} & \textcolor{red}{\textbf{18.29\%}} & 
\textcolor{red}{\textbf{11.69\%}} & \textcolor{red}{\textbf{11.63\%}} & 
\textcolor{red}{\textbf{10.08\%}} & \textcolor{red}{\textbf{10.94\%}} & 
\textcolor{red}{\textbf{9.97\%}} & \textcolor{red}{\textbf{12.50\%}} & 
\textcolor{red}{\textbf{9.85\%}} & \textcolor{red}{\textbf{13.42\%}} \\
\cline{1-14}

\multirow{3}{*}{\scalebox{1.0}{\shortstack{ETTh2\\ $\downarrow$ \\ETTh1}}}
& Original & 
0.780 & 0.604 & 
0.609 & 0.532 & 
0.616 & 0.537 & 
0.920 & 0.635 & 
0.746 & 0.598 & 
0.735 & 0.593 \\
& + \textbf{WQ} & 
\textbf{0.679} & \textbf{0.541} & 
\textbf{0.483} & \textbf{0.469} & 
\textbf{0.517} & \textbf{0.487} & 
\textbf{0.649} & \textbf{0.525} & 
\textbf{0.534} & \textbf{0.520} & 
\textbf{0.515} & \textbf{0.508} \\
\cline{2-14}
& \textbf{Improve} & 
\textcolor{red}{\textbf{12.94\%}} & \textcolor{red}{\textbf{10.43\%}} & 
\textcolor{red}{\textbf{20.69\%}} & \textcolor{red}{\textbf{11.84\%}} & 
\textcolor{red}{\textbf{16.07\%}} & \textcolor{red}{\textbf{9.31\%}} & 
\textcolor{red}{\textbf{29.46\%}} & \textcolor{red}{\textbf{17.32\%}} & 
\textcolor{red}{\textbf{28.42\%}} & \textcolor{red}{\textbf{13.04\%}} & 
\textcolor{red}{\textbf{29.93\%}} & \textcolor{red}{\textbf{14.33\%}} \\
\cline{1-14}

\multirow{3}{*}{\scalebox{1.0}{\shortstack{ETTh1\\ $\downarrow$ \\ETTh2}}}
& Original & 
0.420 & 0.430 & 
0.478 & 0.483 & 
0.416 & 0.444 & 
0.443 & 0.442 & 
0.444 & 0.463 & 
0.445 & 0.459 \\
& + \textbf{WQ} & 
\textbf{0.408} & \textbf{0.412} & 
\textbf{0.419} & \textbf{0.428} & 
\textbf{0.382} & \textbf{0.408} & 
\textbf{0.413} & \textbf{0.420} & 
\textbf{0.421} & \textbf{0.437} & 
\textbf{0.425} & \textbf{0.439} \\
\cline{2-14}
& \textbf{Improve} & 
\textcolor{red}{\textbf{2.86\%}} & \textcolor{red}{\textbf{4.19\%}} & 
\textcolor{red}{\textbf{12.34\%}} & \textcolor{red}{\textbf{11.39\%}} & 
\textcolor{red}{\textbf{8.17\%}} & \textcolor{red}{\textbf{8.10\%}} & 
\textcolor{red}{\textbf{6.77\%}} & \textcolor{red}{\textbf{4.98\%}} & 
\textcolor{red}{\textbf{5.18\%}} & \textcolor{red}{\textbf{5.62\%}} & 
\textcolor{red}{\textbf{4.49\%}} & \textcolor{red}{\textbf{4.36\%}} \\
\cline{1-14}

\multirow{3}{*}{\scalebox{1.0}{\shortstack{RiverFlow\\ $\downarrow$ \\Exchange}}}
& Original & 
0.464 & 0.491 & 
0.585 & 0.537 & 
0.421 & 0.458 & 
0.497 & 0.508 & 
0.942 & 0.765 & 
0.845 & 0.739 \\
& + \textbf{WQ} & 
\textbf{0.451} & \textbf{0.476} & 
\textbf{0.431} & \textbf{0.460} & 
\textbf{0.405} & \textbf{0.427} & 
\textbf{0.477} & \textbf{0.491} & 
\textbf{0.460} & \textbf{0.481} & 
\textbf{0.496} & \textbf{0.509} \\
\cline{2-14}
& \textbf{Improve} & 
\textcolor{red}{\textbf{2.86\%}} & \textcolor{red}{\textbf{3.05\%}} & 
\textcolor{red}{\textbf{26.32\%}} & \textcolor{red}{\textbf{14.34\%}} & 
\textcolor{red}{\textbf{3.80\%}} & \textcolor{red}{\textbf{6.71\%}} & 
\textcolor{red}{\textbf{4.02\%}} & \textcolor{red}{\textbf{3.35\%}} & 
\textcolor{red}{\textbf{51.17\%}} & \textcolor{red}{\textbf{37.12\%}} & 
\textcolor{red}{\textbf{41.30\%}} & \textcolor{red}{\textbf{31.20\%}} \\
\cline{1-14}


\hline
\bottomrule
\end{tabular}
\end{small}
}
\vspace{-10pt}
\end{table*}

\subsection{Tokenization Paradigm Strategy Analysis}
In this subsection, we investigate the generalizability of the tokenization paradigm strategy of \textbf{Wave Quantize} by plugging it into other different kinds of models. \textbf{Model selection and experimental setting.} To achieve this objective, we conduct experiments across a spectrum of representative time series forecasting model structures, including (1) Transformer-based methods: PatchTST \cite{Nie2023ATS}, FEDformer \cite{zhou2022fedformer} and Autoformer \cite{wu2021autoformer}; (2) Linear-based methods: DLinear \cite{Zeng2022AreTE} (3) TCN-based methods: TimesNet \cite{wu2023timesnet}; (4) LLM-based methods: OneFitsAll \cite{Zhou2023OneFA}. 

We standardize the input length to 336, and similarly, the prediction length is uniformly set to 336. Subsequently, comparative experiments were conducted on five datasets: ETTh1, ETTh2, ETTm1, ETTm2, RiverFlow, Exchange, Sunspot, and Weather. Specifically, we sequentially replaced the wave quantize components with each model. The comparative analysis was performed to assess the predictive performance before and after the incorporation of \textbf{Wave Quantize}, comparing the original models with the augmented counterparts.

In Table~\ref{tab:boosting}, it is apparent that the incorporation of the Wave Quantize structure leads to a notably substantial enhancement in the predictive performance of various models, even with the introduction of only a single pre-standardized layer.

Specifically, DLinear demonstrates an average MSE reduction of \textbf{17.79\%} across five datasets, other models are OneFitsAll: \textbf{8.41\%}, PatchTST: \textbf{11.65\%}, TimesNet: \textbf{16.72\%}, FEDformer: \textbf{24.64\%}, and Autoformer: \textbf{22.67\%}. 
Particularly noteworthy is the performance enhancement observed in the classical FEDformer model, where the MSE experiences a remarkable decrease of \textbf{51.17\%} and \textbf{26.74\%} on \emph{RiverFlow}$\rightarrow$\emph{Exchange} and \emph{ETTm2}$\rightarrow$\emph{ETTm1}, respectively, a result that is profoundly surprising. This unequivocally substantiates the generality of the Wave Quantize structure.

\section{Conclusion and future work}\label{section:conclusion}
This paper presents the concept of the tokenization paradigm in time series for the first time and proposes the novel wavebook tokenization that advances research related to multi-domain unified pre-train and cross-domain adaptation without changing the model structure, which will lay the groundwork for the large time series model.
To verify that wavebook strategy can mitigate negative migration, the experiments showed encouraging results: unified multi-domain pre-training would be far superior to single-domain pre-training.
In addition, modeling multi-period features from time series will be an important research direction, especially the period features across channels. For example, capturing correlation representations between multiple underlying periodic patterns across channels is essential for multivariate time series tasks.

\bibliography{ref}
\newpage
\appendix

\newpage
\appendix
\onecolumn


\section{Background}\label{appendix:background}

\paragraph{Function space}
We define the space $L^2(\mathbb{R})$ to describe all functions defined on $x\in[-\infty,+\infty]$ with finite energy
\begin{equation}
    L^2(\mathbb{R})=\left\{f(t)\colon\int_{-\infty}^{+\infty}\left|f(t)\right|^2<+\infty\right\}.
\end{equation}

\paragraph{Wavelet generating function and wavelet function}
If $\psi(t)\in L^2(\mathbb{R})$ satisfies 
\begin{equation}
    C_\psi=\int_{-\infty}^{+\infty}\frac{\left|\hat{{\psi}}(\omega)\right|^2}{\left|\omega\right|}d\omega<+\infty,
\end{equation}
then $\psi(t)$ is called the wavelet generating function, where $C_{\psi}$ is called the tolerance parameter. For any $\forall a\neq0,b\in\mathbb{R}$, we define the continuous wavelet function as follows:
\begin{equation}\label{eqn:1}
    \psi_{(a,b)}(t)=\frac{1}{\sqrt{|a|}}\psi(\frac{t-b}{a}).
\end{equation}

In addition, the energy of the wavelet generating function A will only be distributed over the finite interval, the distribution of the energy decaying rapidly to converge to zero as time approaches infinity, and there exists a horizontal line such that the wavelet function integrates up and down the area of the line for this line. Thus wavelets have two characteristics, the attenuation feature and the fluctuation feature:
\begin{equation}
  \begin{split}
    \int_{-\infty}^{+\infty}\left|\psi(t)\right|^2dt<+&\infty, \int_{-\infty}^{+\infty}\left|\psi_{(a,b)}(t)\right|^2dt<+\infty, \\
    \int_{-\infty}^{+\infty}\psi(t)dt=0&\text, \int_{-\infty}^{+\infty}\psi_{(a,b)}(t)dt=0. 
  \end{split}
\end{equation}

\paragraph{Wavelet transform}
For $\forall f(t)\in L^2(\mathbb{R})$, we have
\begin{equation}\label{eqn:2}
    W_f(a,b)=\int_{-\infty}^{+\infty}f(t)\overline{\psi}_{(a,b)}(t)dt,
\end{equation}
which is considered to be the wavelet transform of the signal $f(t)$.

\paragraph{Orthogonal wavelet}
Taking $a=2^{-j},b=2^{-j}k$ in \eqref{eqn:1}, we get a set of wavelets
\begin{equation}
    \left\{\psi_{j,k}(t)=2^{j/2}\psi(2^jt-k),(j,k)\in\mathbb{Z}^2\right\},
\end{equation}
which forms the orthonormal basis (O.N.B) of $L^2(\mathbb{R})$, that is $\psi(t)$ as the orthogonal wavelet. Thus, any signal $f(t)$ defined in $L^2(\mathbb{R})$ can be described as a linear representation of this set of orthonormal basis:
\begin{equation}
    f(t)=\sum_{j=-\infty}^{+\infty}\sum_{k=-\infty}^{+\infty}\alpha_{j,k}\cdot\psi_{j,k}(t).
\end{equation}

Since $\left\{\psi_{j,k}(t),(j,k)\in\mathbb{Z}^2\right\}$ are the mutually orthonormal basis, the coefficients of the above linear representation can be calculated by
\begin{equation}
    \alpha_{j,k}=\int_{-\infty}^{+\infty}f(t)\overline{\psi}_{j,k}(t)dt=W_{f}(2^{-j},2^{-j}k),
\end{equation}
and the signal $f(t)$ has the following complete representation
\begin{equation}\label{eqn:3}
    f(t)=\sum_{j=-\infty}^{+\infty}\sum_{k=-\infty}^{+\infty}W_{f}(2^{-j},2^{-j}k)\cdot\psi_{j,k}(t).
\end{equation}
To satisfy the efficient computation of GPUs in the Pytorch environment, we would like to discretize some values to carve a continuous piece of function. Based on Shannon Sampling Theorem, substituting \eqref{eqn:3} into \eqref{eqn:2} yields
\begin{equation}\label{eqn:4}
    W_f(a,b)=\sum_{j=-\infty}^{+\infty}\sum_{k=-\infty}^{+\infty}W_f(2^{-j},2^{-j}k)\times\int_R\psi_{j,k}(t)\overline{\psi}_{(a,b)}(t)dt.
\end{equation}
Since the part within the integral is only related to $a,b,j,k$, the process of \eqref{eqn:4} can be calculated by discrete sampling. This means that when doing the wavelet transform on a continuous signal, the information has coalesced into discrete sampling points, which provides rationalization for subsequent calculations.

\paragraph{Scale function}
In the Shannon Sampling Theorem, for any signal $f(t)$ defined on $L^2(\mathbb{R})$, if the frequency domain form $\hat{f}(\omega)$ of that signal has a truncation frequency $B$, then that signal can be reconstructed by equally spaced discrete sampling. This sampling interval can be at most $\frac\pi B$. If the function $f(t)$ satisfies the following conditions
\begin{equation}
  \begin{split}
    \forall f(t)\in L^2(\mathbb{R})&, \\
    \hat{{f}}(\omega)=\int_{-\infty}^{+\infty}f(t)e^{-i\omega t}dt&=0,\left|\omega\right|>B,
  \end{split}
\end{equation}
then we have
\begin{equation}\label{eqn:5}
    f(t)=\sum_{n\in\mathbb{Z}}f(n\Delta)\frac{sin(\pi/\Delta)(t-n\Delta)}{(\pi/\Delta)(t-n\Delta)}, 0<\Delta\leq\frac\pi B.
\end{equation}

Considering first the case $B=\pi$ in \eqref{eqn:5}, we have
\begin{equation}
    f(t)=\sum_{n\in\mathbb{Z}}f(n)\frac{sin(\pi(t-n))}{\pi(t-n)}=\sum_{n=-\infty}^{+\infty}f(n)\phi(t-n),
\end{equation}
where the scale function is defined as $\phi(t)=\frac{\sin(\pi t)}{\pi t}$.

Further, define the space $V_0=\left\{f(t);\hat{{f}}(\omega)=0,|\omega|>\pi\right\}$ with truncation frequency $B=\pi$i, where $\left\{\phi_{0,n}=2^{0/2}\phi(2^0t-n)=\phi(t-n);n\in\mathbb{Z}\right\}$ constitutes a set of orthonormal basis (O.N.B) which forms $V_{0}$. Specifically, we have
\begin{equation}
  \begin{split}
    \left\langle\phi(t-n),\phi(t-m)\right\rangle&=\int_{-\infty}^{+\infty}\phi(t-n)\overline{\phi}(t-m)dt, \\
    &=\frac1{2\pi}\int_{-\infty}^{+\infty}(\hat{\phi}(\omega)e^{-in\omega})(\hat{\phi}(\omega)e^{-im\omega})d\omega, \\
    &=\frac{1}{2\pi}\int_{-\infty}^{+\infty}\left|\hat{\phi}(\omega)\right|^{2}e^{-i(n-m)\omega}d\omega, \\
    &=\frac1{2\pi}\int_{-\pi}^{+\pi}e^{-i(n-m)\omega}d\omega.
  \end{split}
\end{equation}
We introduce the function $\delta$ for refined expression, that is 
\begin{equation}
    \forall m,n\in\mathbb{Z},\left<\phi(t-n),\phi(t-m)\right>=\delta(n-m).
\end{equation}

Defining the space $V_j=\left\{f(t);\hat{{f}}(\omega)=0,|\omega|>2^j\pi\right\}$ with truncation frequency $B=2^j\pi$ and taking the sampling interval $\Delta=2^{-j}$, we have 
\begin{equation}
  \begin{split}
    f(t)&=\sum_{n=-\infty}^{+\infty}f(2^{-j}n)\frac{sin\pi(2^jt-n)}{\pi(2^jt-n)}, \\
    &=\sum_{n=-\infty}^{+\infty}2^{-j/2}f(2^{-j}n)\cdot2^{j/2}\phi(2^{j}t-n), \\
    &=\sum_{n=-\infty}^{+\infty}2^{-j/2}f(2^{-j}n)\phi_{j,n}(t),
  \end{split}
\end{equation}
where
\begin{equation}
    \phi_{j,n}(t)=2^{j/2}\phi(2^jt-n)=\frac{sin\pi(2^jt-n)}{\pi(2^jt-n)}.
\end{equation}

Besides, we have
\begin{equation}
  \begin{split}
    \left\langle\phi_{j,n}(t),\phi_{j,m}(t)\right\rangle&=\int_{-\infty}^{+\infty}\phi_{j,n}(t)\overline{\phi}_{j,m}(t)dt, \\
    &=\frac1{2\pi\cdot2^j}\int_{-\infty}^{+\infty}\left|\hat{{\phi}}(2^{-j}\omega)\right|^2e^{-i(n-m)\omega}d\omega, \\
    &=\frac1{2\pi\cdot2^j}\int_{-2^j\pi}^{2^j\pi}e^{-i(n-m)\omega}d\omega, \\
    &=\delta(n-m).
  \end{split}
\end{equation}
Therefore, $\phi_{j,n}$ constitutes a set of orthonormal basis of $V_j$, i.e., $V_j$ is a linear subspace tensored by $\phi_{j,n}$.

\paragraph{Orthogonal wavelet and close-span space} \label{section:VW}
Based on the wavelet function $\psi_{j,k}$, we define the close-span space as:
\begin{equation}
    W_j=closespan\left\{\psi_{j,k}(t)2^{j/2}\psi(2^jt-k),(j,k)\in\mathbb{Z}^2\right\},
\end{equation}
and the close-span space has the following characteristics: 1) Spatial orthogonality $W_j\perp W_{j+1}$, since $\psi_{j,k}$ and $\psi_{j+1,k}$ are mutually orthogonal to each other; 2) Spatial approximability $L^2(\mathbb{R})=\bigoplus_{j=-\infty}^{+\infty}W_j$. Since orthogonal wavelets form a standard set of orthogonal bases of space $L^2(\mathbb{R})$, and the close-span space will approximate space $L^2(\mathbb{R})$ by the direct-sum operation; 3) The transfer relation of neighboring spaces: $g(t)\in W_j\Leftrightarrow g(2t)\in W_{_{j+1}}$, which is known by the definition of close-span space. Based on the above characteristics, by calculating the basis functions of $W_{0}$, a set of orthogonal wavelets can be obtained, where each wavelet $\psi_{j,k}$ can be formed into a close-span space $W_{j}$ and the direct-sum of all the closure spaces can be approximated to $L^2(\mathbb{R})$.

The formulaic description of the above idea is that, if $\psi(t-k)$ is the set of the orthonormal basis of $W_{0}$, then for any $j$ there has $\left\{2^{j/2}\psi(2^jt-k),(j,k)\in\mathbb{Z}^2\right\}$ as the orthonormal basis of $W_{j}$. Moreover, it is easy to verify:
\begin{equation}\label{eqn:6}
    W_j\perp V_j,V_{j+1}=W_j\oplus V_j.
\end{equation}
Thus, the construction of orthogonal wavelets is equivalent to finding a set of standard orthogonal bases for $W_{0}$.

\paragraph{Scale equation and low-pass filter}
Since the scale function $\phi(x)\in V_0\subseteq V_1$, and there exists a set of orthonormal basis $\left\{\sqrt{2}\phi(2t-n);n\in\mathbb{Z}\right\}$ for $V_{1}$, there must exist a unique sequence of coefficients $\left\{h_n;n\in\mathbb{Z}\right\}\in l^2(\mathbb{Z})$ such that
\begin{equation}\label{eqn:7}
    \phi(t)=\sqrt{2}\sum_{n\in\mathbb{Z}}h_n\phi(2t-n),
\end{equation}
which is regarded as the scale equation, and since $\phi(2t-n)$ is mutually orthogonal to each other for different $n$, the coefficients are calculated as follows
\begin{equation}
    h_n=\left\langle\phi(t),\sqrt{2}\phi(2t-n)\right\rangle=\sqrt{2}\int_R\phi(t)\overline{\phi}(2t-n)dt.
\end{equation}

In addition, the scale equation are converted to frequency domain form by Fourier transforms
\begin{equation}\label{eqn:H_app}
    \hat{{\phi}}(\omega)=H(\omega/2)\hat{{\phi}}(\omega/2),H(\omega)=\frac{1}{\sqrt{2}}\sum_{n\in\mathbb{Z}}h_ne^{-in\omega},
\end{equation}
where $H(\omega)$ is referred to as the low-pass filter and hence $h_{n}$ is also referred to as the low-pass filter coefficients.

\paragraph{Wavelet equation and band-pass filter}
For the wavelet function $\psi(x)\in W_0\subseteq V_1$, there exists $\left\{g_{n};n\in\mathbb{Z}\right\}\in l^{2}$ such that
\begin{equation}\label{eqn:8}
    \psi(t)=\sqrt2\sum_{n\in\mathbb{Z}}g_n\phi(2t-n),
\end{equation}
which is regarded as the wavelet equation, and since $\phi(2t-n)$ is orthogonal for different $n$, the coefficients are calculated as follows
\begin{equation}
    g_n=\left\langle\psi(t),\sqrt{2}\phi(2t-n)\right\rangle=\sqrt{2}\int_R\psi(t)\overline{\phi}(2t-n)dt.
\end{equation}
In addition, the wavelet equations can be obtained in frequency domain form by Fourier transformation
\begin{equation}\label{eqn:G_app}
    \hat{{\psi}(\omega)}=G(\omega/2)\hat{{\phi}}(\omega/2),G(\omega)=\frac{1}{\sqrt{2}}\sum_{n\in\mathbb{Z}}g_ne^{-in\omega},
\end{equation}
where $G(\omega)$ is referred to as the bandpass filter and $g_n$ is also referred to as the impulse response coefficient.

\section{Proof of results in Section \ref{section:methods_3}}\label{appendix:proof}

\begin{proof}[\textbf{Proof of Proposition \ref{prop:1}}]
    If $\psi(t)$ is the orthogonal wavelet, we have that $\left\{\psi_{j,k}(t)=2^{j/2}\psi(2^jt-k),(j,k)\in\mathbb{Z}^2\right\}$ constitutes a standard orthonormal basis for $L^2(\mathbb{R})$. Then there must exist the unique coefficients sequence $\left\{c_{j,k};(j,k)\in\mathbb{Z}^{2}\right\}\in l^{2}\left(\mathbb{Z}\right)$. For $\forall f(t)\in L^2(\mathbb{R})$, we have
    \begin{equation}
        f(t)=\sum_{j\in\mathbb{Z}}\sum_{k\in\mathbb{Z}}c_{j,k}\cdot\psi_{j,k}(t).
    \end{equation}
    Specifically, the bijection relation is satisfied between variables $L^2(\mathbb{R})$ and $l^2(\mathbb{Z})$ in two spaces $f(t)$ and $\{h_n;n\in\mathbb{Z}\}$.

    Since $\left\{\psi_{j,k}(t),(j,k)\in\mathbb{Z}^{2}\right\}$ constitutes a set of orthogonal bases satisfying
    \begin{equation}
        \forall(j_1,k_1)\neq(j_2,k_2),\left<\psi_{j_1,k_1},\psi_{j_2,k_2}\right>=0.
    \end{equation}
    Further, we have
    \begin{equation}
        \begin{split}
        \left\langle f(t),\psi_{m,n}(t)\right\rangle&=\left\langle\sum_{j\in\mathbb{Z}}\sum_{k\in\mathbb{Z}}c_{j,k}\cdot\psi_{j,k}(t),\psi_{m,n}(t)\right\rangle, \\
        &=\sum_{j\in\mathbb{Z}}\sum_{k\in\mathbb{Z}}c_{j,k}\cdot\left\langle\psi_{j,k}(t),\psi_{m,n}(t)\right\rangle, \\
        &=c_{m,n}\cdot\left\langle\psi_{m,n}(t),\psi_{m,n}(t)\right\rangle, \\
        &=c_{m,n}\cdot\left|\psi_{m,n}(t)\right|^2.
        \end{split}
    \end{equation}
    Thus, the coefficients sequence can be uniquely determined by $c_{m,n}=\left\langle f(t),\psi_{m,n}(t)\right\rangle/\left|\psi_{m,n}(t)\right|^2$, and the bijection relation is satisfied between the variables $f(t)$ and $\{h_n;n\in\mathbb{Z}\}$, which completes the proof.
\end{proof}

\begin{lemma}\label{lemma:1}
    The sufficiently-necessary condition for orthonormal system: Defining the function $f(x)\in L^2(\mathbb{R})$, then the set
    \begin{equation}
        \left\{f_{0,n}=2^{0/2}f(2^0t-n)=f(t-n);n\in\mathbb{Z}\right\}
    \end{equation}
    forms the orthonormal system of $L^2(\mathbb{R})$, that is
    \begin{equation}
        \left\langle f(t-n),f(t-l)\right\rangle=\delta(n-l),
    \end{equation}
    is sufficiently-necessary for
    \begin{equation}
        \sum_{k\in\mathbb{Z}}\left|\hat{{f}}(\omega+2k\pi)\right|^2=1\quad a.e.\omega\in\mathbb{R}.
    \end{equation}
    In fact, Lemma \ref{lemma:1} is proved due to
    \begin{equation}
        \begin{split}
        \left\langle f(t-n),f(t-l)\right\rangle&=\frac1{2\pi}\int_{\mathbb{R}}\hat{{f}}(\omega)e^{-in\omega}\cdot\overline{\left(\hat{{f}}(\omega)e^{-il\omega}\right)}d\omega, \\
        &=\frac1{2\pi}\int_0^{2\pi}\sum_{k\in\mathbb{Z}}\left|\hat{{f}}(\omega+2k\pi)\right|^2\cdot e^{-i(n-l)\omega}d\omega.
        \end{split}
    \end{equation}
\end{lemma}

\begin{proof}[\textbf{Proof of Proposition \ref{prop:2}}]
    We define the functions $H(\omega)$ and $G(\omega)$ refer as the low-pass filter and band-pass filter based on $\psi(t)\in L^2(\mathbb{R})$, which determined by \eqref{eqn:H_app} and \eqref{eqn:G_app}, respectively, and introduce the definition of matrix $M(\omega)$ as follows
    \begin{equation}
        \left.M(\omega)=\left(\begin{matrix}H(\omega)&G(\omega)\\H(\omega+\pi)&G(\omega+\pi)\end{matrix}\right.\right).
    \end{equation}
    The function group $\left\{\psi_{0,k}=2^{0/2}\psi(2^0t-k)=\psi(t-k),k\in\mathbb{Z}\right\}$ forms the orthonormal basis of $W_{0}$, that is, the sufficient-necessary condition for $\psi(x)$ as the orthogonal wavelet is that $M(\omega)$ is the Unitary Matrix: $M^H(\omega)M(\omega)=M(\omega)M^H(\omega)=I,a.e.\omega\in\mathbb{R}$, where $M^{H}(\omega)$ is defined to be the conjugate transpose matrix of $M(\omega)$.

    By the definition of the Unitary Matrix, $M(\omega)$ is the Unitary Matrix equivalent to
    \begin{equation}\label{unimatrix:1_app}
        \begin{split}
        \left|H(\omega)\right|^2+\left|H(\omega+\pi)\right|^2=1,a.e.\omega\in\mathbb{R}, \\
        \begin{vmatrix}G(\omega)\end{vmatrix}^2+\begin{vmatrix}G(\omega+\pi)\end{vmatrix}^2=1,a.e.\omega\in\mathbb{R}, \\
        H(\omega)\overline{G}(\omega)+H(\omega+\pi)\overline{G}(\omega+\pi)=0,a.e.\omega\in\mathbb{R}.
        \end{split}
    \end{equation}
    We define the spaces $V_{0}$ and $W_{0}$ as follows
    \begin{equation}
        \begin{split}
        V_0=closespan\left\{\phi(t-n),n\in\mathbb{Z}\right\}, \\
        W_0=closespan\left\{\psi(t-n),n\in\mathbb{Z}\right\}.
      \end{split}
    \end{equation}
    By the definitions of $H(\omega)$ and $G(\omega)$, Equation \ref{unimatrix:1_app} is equivalent to
    \begin{equation}
        \begin{split}
        \sum_{k\in\mathbb{Z}}\left|\hat{{\phi}}(\omega+2k\pi)\right|^2=1,a.e.\omega\in\mathbb{R}, \\
        \sum_{k\in\mathbb{Z}}\left|\hat{{\psi}}(\omega+2k\pi)\right|^2=1,a.e.\omega\in\mathbb{R}, \\
        V_0\perp W_0,a.e.\omega\in\mathbb{R}.
        \end{split}
    \end{equation}
    Because of Lemma \ref{lemma:1}, the first two conditions are equivalent to
    \begin{equation}
        \begin{split}
        \left\langle\phi(t-n),\phi(t-l)\right\rangle=\delta(n-l), \\
        \left\langle\psi(t-n),\psi(t-l)\right\rangle=\delta(n-l).
        \end{split}
    \end{equation}

    In summary, the sufficient-necessary condition for $M(\omega)$ as the Unitary Matrix is that $\phi(t-n)$ and $\psi(t-n)$ form the standard orthogonal system of $L^2(\mathbb{R})$, respectively, and the two spaces formed by $\phi(t-n)$ and $\psi(t-n)$ are orthogonal. Considering the properties of $V_{j}$ and $W_{j}$ in section \ref{section:VW}, $\left\{\psi_{j,k}(t)=2^{j/2}\psi(2^jt-k),(j,k)\in\mathbb{Z}^2\right\}$ constitutes the orthonormal basis for $L^2(\mathbb{R})$, and hence $\psi(t)$ is an orthogonal wavelet.

    According to Proposition \ref{prop:2}, by designing a specific functional relationship between $H(\omega)$ and $G(\omega)$, we can guarantee that $M(\omega)$ is the Unitary Matrix, and thus that $\psi(t)$ is the orthogonal wavelet. For instance, when $G(\omega)=e^{-i\omega}\overline{H}(\omega+\pi)$, it is easy to verify that $M(\omega)$ is the Unitary Matrix, when the frequency domain form of the wavelet function $\psi(t)$ is:
    \begin{equation}
        \hat{\psi}(\omega)=e^{-i\omega/2}\overline{H}(\pi+\omega/2)\cdot\hat{\phi}(\omega/2),
    \end{equation}
    where the corresponding impulse response coefficient is
    \begin{equation}
        g_n=(-1)^{n-1}\overline{h}_{1-n},n\in\mathbb{Z},
    \end{equation}
    Therefore, the time domain form of the wavelet function $\psi(t)$ is
    \begin{equation}
        \psi(t)=\sqrt{2}\sum_{n\in\mathbb{Z}}{(-1)^{n-1}\overline{h}_{1-n}\phi(2t-n)}.
    \end{equation}
\end{proof}

\section{Related Work of Tokenization}\label{appendix:related}
In the initial phase, the transformer-based model utlized the ``point as token'' tokenization strategy, which caused to two limitations: high computational complexity and serious information redundancy, Therefore, existing approaches have focused on reducing computational complexity by mitigating information redundancy.

For instance, Reformer \cite{kitaev2020reformer} designed the locally sensitive hashing self-attention to reduce computational complexity. Informer \cite{haoyietal-informer-2021} proposes the ProbSparse self-attention mechanism to efficiently replace the canonical self-attention. Pyraformer \cite{liu2021pyraformer} introduces the pyramidal attention module which reduces computational complexity by constraining the maximum length of the signal traversing paths. Autoformer \cite{wu2021autoformer} designs the Auto-Correlation mechanism based on the series periodicity, which conducts the dependencies discovery and representation aggregation at the sub-series level. ETSformer \cite{woo2022etsformer} proposes the novel exponential smoothing attention and frequency attention to replace the self-attention mechanism in vanilla Transformers, thus improving both accuracy and efficiency. FEDformer \cite{zhou2022fedformer} avoids the high overhead of computation on the time domain by calculating the dependencies between individual bands in the frequency domain, while Fourier Transform has been used to ensure that individual bands have a global view.

\section{More analysis of Tokenization Strategy}\label{appendix:tokenization}
\subsection{Ideal tokenization generation strategy}\label{appendix:tokenization_1}
Since TS datasets from different domains and sampling settings exhibit diverse periodic patterns, it is difficult to identify a unified sub-series span that matches all TS data domains. existing tokenization strategies are unable to establish a unified framework to adapt to the potential TS data domains, and the pre-training phase is unable to adequately learn cross-domain representational information, which limits the generalization ability and scalability of the model. Therefore, The ideal tokenization generation strategies should be insensitive to the mathematical characteristics of different data domains. However, the patch span depends on the periodic pattern of the raw series

Thinking about the approach in NLP: the word is a high-level input data abstracted and generalized by the human brain, and existing subword algorithms \cite{Kudo2018SentencePieceAS,Sennrich2015NeuralMT} can disassemble the complex and lengthy word into a set of mutually independent and complementary units, which ensures that each token input to attention contains similar-size semantic information. This ensures that semantic information, which is essential for inference, is evenly distributed across subwords. However, the distribution of pattern information in the TS data is completely uncertain (e.g., randomly occurring anomalous oscillatory waveforms contain more information), which increases the inference difficulty of attention under the \textbf{\emph{Point as token}} strategy. Therefore, the ideal tokenization strategy should play a role similar to feature programming, such that the generated tokens contain some pattern information from the original series (as semantic information in NLP). At the same time, it needs to ensure that the pattern information obtained from different TS domains is essentially isomorphic or similarly distributed, which is regarded as the key to activate the model cross-domains adaptability.

\subsection{Challenges addressed by wave tokenization strategy}\label{appendix:tokenization_2}
Establishing the unified \tp between different TS domains has the following challenges: (1) TS data with diverse domain background knowledge tend to exhibit different characteristics; (2) Datasets from different sources may exhibit variations, even in the same domain background; (3) The same data instances can also face the challenge of cross-domain adaptation due to a priori features such as sampling rate.

The proposed \textbf{\emph{Wave as Token}} strategy solves \emph{challenge-1} through the designed shared embedding space. In addition, the designed strategy ensures that arbitrary series data can be recoded to equal-length groups of tokens (number of tokens equal to series length), thus solving \emph{challenge-2}. This ensures that no hyperparameter setting tricks are utilized in the model to adapt to potential data domains, even if these domains are completely different from each other in terms of structural features (e.g., channel number, sequence length) are completely different. Finally, we solve \emph{challenge-3} by ensuring that each token contains TS pattern information within a localized window through the finite-length basis functions and bridges the difference in the distribution of pattern information from different domains within the K-dimensional space formed by basis functions, where $K$ is the size of the proposed wavebook.



\section{Implementation Details}\label{appendix:detail}
\begin{table*}[htbp]
  \caption{Description of datasets in forecasting and imputation tasks. The dataset size is organized in (Train, Validation, Test).}\label{tab:dataset_1}
  \vskip 0.05in
  \centering
  \resizebox{0.9\columnwidth}{!}{
  \begin{threeparttable}
  \begin{small}
  \renewcommand{\multirowsetup}{\centering}
  \setlength{\tabcolsep}{3.8pt}
  \begin{tabular}{c|l|c|c|c|c}
    \toprule
    Tasks & Dataset & Dim & Series Length & Dataset Size & Information (Frequency) \\
    
    \midrule
    \multirow{13}{*}{\rotatebox{90}{\scalebox{1}{Forecasting}}}
    & ETTm1, ETTm2 & 7 & \{96, 192, 336, 720\} & (34465, 11521, 11521) & Electricity (15 mins)\\
    \cmidrule{2-6}
    & ETTh1, ETTh2 & 7 & \{96, 192, 336, 720\} & (8545, 2881, 2881) & Electricity (15 mins) \\
    \cmidrule{2-6}
    & Electricity & 321 & \{96, 192, 336, 720\} & (18317, 2633, 5261) & Electricity (Hourly) \\
    \cmidrule{2-6}
    & Traffic & 862 & \{96, 192, 336, 720\} & (12185, 1757, 3509) & Transportation (Hourly) \\
    \cmidrule{2-6}
    & Weather & 21 & \{96, 192, 336, 720\} & (36792, 5271, 10540) & Weather (10 mins) \\
    \cmidrule{2-6}
    & Exchange & 8 & \{96, 192, 336, 720\} & (5120, 665, 1422) & Exchange rate (Daily) \\
    \cmidrule{2-6}
    & Sunspot & 1 & \{96, 192, 336, 720\} & (44354, 14785, 14785) & Nature (Daily) \\
    \cmidrule{2-6}
    & RiverFlow & 1 & \{96, 192, 336, 720\} & (14244, 4748, 4748) & Nature (Daily) \\
    \cmidrule{2-6}
    & SolarPower & 1 & \{96, 192, 336, 720\} & (4438333, 1479444, 1479444) & Energy (4 seconds) \\
     
    \midrule
    \multirow{6}{*}{\rotatebox{90}{\scalebox{1}{Imputation}}}
    & ETTm1, ETTm2 & 7 & 96 & (34465, 11521, 11521) & Electricity (15 mins) \\
    \cmidrule{2-6}
    & ETTh1, ETTh2 & 7 & 96 & (8545, 2881, 2881) & Electricity (15 mins)\\
    \cmidrule{2-6}
    & Electricity & 321 & 96 & (18317, 2633, 5261) & Electricity (15 mins)\\
    \cmidrule{2-6}
    & Weather & 21 & 96 & (36792, 5271, 10540) & Weather (10 mins) \\

    \bottomrule
    \end{tabular}
    \end{small}
  \end{threeparttable}
  }
\end{table*}

\begin{table*}[htbp]
  \caption{Description of datasets in the classification task. The dataset size is organized in (Train, Validation, Test).}\label{tab:dataset_2}
  \vskip 0.05in
  \centering
  \resizebox{0.8\columnwidth}{!}{
  \begin{threeparttable}
  \begin{small}
  \renewcommand{\multirowsetup}{\centering}
  \setlength{\tabcolsep}{3.8pt}
  \begin{tabular}{c|l|c|c|c|c}
    \toprule
    Tasks & Dataset & Class Number & Series Length & Dataset Size & Information (Sample Rate) \\
    
    \midrule
    \multirow{47}{*}{\rotatebox{90}{\scalebox{1}{Classification}}}
    & BeetleFly & 2 & 512 & (20, 20, 20) & Image \\
    \cmidrule{2-6}
    & BME & 3 & 128 & (30, 150, 150) & Simulated \\
    \cmidrule{2-6}
    & CBF & 3 & 128 & (30, 900, 900) & Simulated \\
    \cmidrule{2-6}
    & Chinatown & 2 & 24 & (20, 343, 343) & Traffic \\
    \cmidrule{2-6}
    & ChlorineConcentration & 3 & 166 & (467, 3840, 3840) & Sensor \\
    \cmidrule{2-6}
    & DistalPhalanxTW & 6 & 80 & (400, 139, 139) & Image \\
    \cmidrule{2-6}
    & ECG200 & 2 & 96 & (100, 100, 100) & ECG \\
    \cmidrule{2-6}
    & ECG5000 & 5 & 140 & (500, 4500, 4500) & ECG \\
    \cmidrule{2-6}
    & ElectricDevices & 7 & 96 & (8926, 7711, 7711) & Device \\
    \cmidrule{2-6}
    & FaceAll & 14 & 131 & (560, 1690, 1690) & Image \\
    \cmidrule{2-6}
    & FaceFour & 4 & 350 & (24, 88, 88) & Image \\
    \cmidrule{2-6}
    & FacesUCR & 14 & 131 & (200, 2050, 2050) & Image \\
    \cmidrule{2-6}
    & FiftyWords & 50 & 270 & (450, 455, 455) & Image \\
    \cmidrule{2-6}
    & GunPointAgeSpan & 2 & 150 & (135, 316, 316) & Motion \\
    \cmidrule{2-6}
    & GunPointMaleVersusFemale & 2 & 150 & (135, 316, 316) & Motion \\
    \cmidrule{2-6}
    & GunPointOldVersusYoung & 2 & 150 & (135, 316, 316) & Motion \\
    \cmidrule{2-6}
    & GunPoint & 2 & 150 & (50, 150, 150) & Motion \\
    \cmidrule{2-6}
    & InsectEPGSmallTrain & 3 & 601 & (17, 249, 249) & Sensor \\
    \cmidrule{2-6}
    & InsectWingbeatSound & 11 & 256 & (220, 1980, 1980) & Sensor \\
    \cmidrule{2-6}
    & ItalyPowerDemand & 2 & 24 & (67, 1029, 1029) & Sensor \\
    \cmidrule{2-6}
    & MedicalImages & 10 & 99 & (381, 760, 760) & Image \\
    \cmidrule{2-6}
    & MiddlePhalanxTW & 6 & 80 & (399, 154, 154) & Image \\
    \cmidrule{2-6}
    & MoteStrain & 2 & 84 & (20, 1252, 1252) & Sensor \\
    \cmidrule{2-6}
    & Plane & 7 & 144 & (105, 105, 105) & Sensor \\
    \cmidrule{2-6}
    & ProximalPhalanxTW & 6 & 80 & (400, 205, 205) & Image \\
    \cmidrule{2-6}
    & SonyAIBORobotSurface1 & 2 & 70 & (20, 601, 601) & Sensor \\
    \cmidrule{2-6}
    & SonyAIBORobotSurface2 & 2 & 65 & (27, 953, 953) & Sensor \\
    \cmidrule{2-6}
    & SwedishLeaf & 15 & 128 & (500, 625, 625) & Image \\
    \cmidrule{2-6}
    & SyntheticControl & 6 & 60 & (300, 300, 300) & Simulated \\
    \cmidrule{2-6}
    & ToeSegmentation2 & 2 & 343 & (36, 130, 130) & Motion \\
    \cmidrule{2-6}
    & Trace & 4 & 275 & (100, 100, 100) & Sensor \\
    \cmidrule{2-6}
    & UMD & 3 & 150 & (36, 144, 144) & Simulated \\
    \cmidrule{2-6}
    & UWaveGestureLibraryY & 8 & 315 & (896, 3582, 3582) & Motion \\
    \cmidrule{2-6}
    & Wafer & 2 & 152 & (1000, 6164, 6164) & Sensor \\
    \cmidrule{2-6}
    & WordSynonyms & 25 & 270 & (267, 638, 638) & Image \\
    
    \bottomrule
    \end{tabular}
    \end{small}
  \end{threeparttable}
  }
\end{table*}

We provide the dataset descriptions and experiment configurations in Table \ref{tab:dataset_1}, Table \ref{tab:dataset_2} and Table \ref{tab:model_config}. Besides, The specific description of the symbols involved in the method is shown in Table~\ref{tab:description}. All experiments are repeated three times, implemented in PyTorch \cite{Paszke2019PyTorchAI} and conducted on a single Tesla V100 SXM2 32GB GPU.

Our method is trained with the L2 Loss, using the ADAM optimizer with an initial learning rate of $10^{-4}$, and Batch size is set in $16 \rightarrow 64$. The training process is early stopped after three epochs (patience=3) if there is no loss degradation on the valid set. The mean square error (MSE) and mean absolute error (MAE) are used as metrics in forecasting and imputation tasks. Besides, the accuracy (Acc), precision (Pre), recall (Rec), and f1-score (F1) are used as metrics in the classification task. For a fair comparison, we fix the input length to 336 for all datasets. By default, the proposed \textbf{\myformer} contains $5 \rightarrow 10$ \textbf{EncoderLayers}. All the baselines that we reproduced are implemented based on configurations of the original paper or their official code. For a fair comparison, we design the same input embedding and final prediction layer for all base models. 

Besides, the datasets that were utilized in the forecasting, imputation, and classification tasks are described in detail below:
(1) \textbf{ETT (ETTh1, ETTh2, ETTm1, and ETTm2)}\footnote{https://github.com/zhouhaoyi/Informer2020} \cite{haoyietal-informer-2021} contains six power load features and oil temperature used for monitoring electricity transformers. ETT involves four subsets. ETTm1 and ETTm2 are recorded at 15-minute intervals, while ETTh1 and ETTh2 are recorded hourly. 
(2) \textbf{Weather}\footnote{https://www.bgc-jena.mpg.de/wetter/} contains 21 meteorological indicators, such as temperature, humidity, and precipitation, which are recorded every 10 minutes in the year 2020. 
(3) \textbf{Electricity}\footnote{https://archive.ics.uci.edu/dataset/321/electricity} comprises hourly power consumption of 321 clients from 2012 to 2014. 
(4) \textbf{Traffic}\footnote{http://pems.dot.ca.gov} reports the number of vehicles loaded on all 862 roads at each moment in time. 
(5) \textbf{Sunspot}\footnote{https://www.sidc.be/SILSO/newdataset} records observations of sunspots for long-term monitoring, consisting of $73924$ timesteps.
(6) \textbf{River Flow}\footnote{http://www.jenvstat.org/v04/i11} reports th daily river flow, consisting of $23741$ timesteps.
(7) \textbf{Solar Power}\footnote{https://zenodo.org/records/4656032} contains a single long daily time series representing the wind power production in MW recorded every 4 seconds starting from 2019.
(8) \textbf{UCR archive}\footnote{https://www.cs.ucr.edu/\url{~}eamonn/time\_series\_data\_2018} as the well-known time series classification repository, where representative \textbf{35 datasets} are selected to validate the performance of the proposed approach.

\section{Full experiment results}\label{appendix:exp}

\subsection{Forecasting task}\label{appendix:exp_1}
\begin{table*}[htbp]
  \caption{Comparison of the complete performance with diverse prediction lengths ($\{96,192,336,720\}$) on \textbf{full-data forecasting} task, where \emph{Sunspot}, \emph{RiverFlow} and \emph{SolarPower} are novel datasets that we introduced.}\label{tab:forecasting_full}
  \vskip 0.05in
  \centering
  \resizebox{0.89\columnwidth}{!}{
  \begin{threeparttable}
  \begin{small}
  \renewcommand{\multirowsetup}{\centering}
  \setlength{\tabcolsep}{1pt}
\begin{tabular}{c|c|cc|cc|cc|cc|cc|cc|cc|cc|cc|cc}
    \toprule
    \hline
    
    \multicolumn{2}{c}{\multirow{2}{*}{Models}} & 
    \multicolumn{2}{c}{\rotatebox{0}{\scalebox{0.8}{\textbf{\myformer}}}} &
    \multicolumn{2}{c}{\rotatebox{0}{\scalebox{0.8}{GPT4TS}}} & 
    \multicolumn{2}{c}{\rotatebox{0}{\scalebox{0.8}{DLinear}}} &
    \multicolumn{2}{c}{\rotatebox{0}{\scalebox{0.8}{PatchTST}}} &
    \multicolumn{2}{c}{\rotatebox{0}{\scalebox{0.8}{TimesNet}}} &
    \multicolumn{2}{c}{\rotatebox{0}{\scalebox{0.8}{FEDformer}}} & 
    \multicolumn{2}{c}{\rotatebox{0}{\scalebox{0.8}{Autoformer}}} & 
    \multicolumn{2}{c}{\rotatebox{0}{\scalebox{0.8}{Stationary}}} & 
    \multicolumn{2}{c}{\rotatebox{0}{\scalebox{0.8}{ETSformer}}} &  
    \multicolumn{2}{c}{\rotatebox{0}{\scalebox{0.8}{Informer}}} \\
    
    \multicolumn{2}{c}{} & 
    \multicolumn{2}{c}{\scalebox{0.8}{(\textbf{Ours})}} & 
    \multicolumn{2}{c}{\scalebox{0.8}{\citeyear{Nie2023ATS}}} &
    \multicolumn{2}{c}{\scalebox{0.8}{\citeyear{wu2023timesnet}}} &
    \multicolumn{2}{c}{\scalebox{0.8}{\citeyear{wang2023micn}}} &
    \multicolumn{2}{c}{\scalebox{0.8}{\citeyear{Zeng2022AreTE}}} & 
    \multicolumn{2}{c}{\scalebox{0.8}{\citeyear{zhou2022fedformer}}} & 
    \multicolumn{2}{c}{\scalebox{0.8}{\citeyear{Liu2022NonstationaryTR}}} & 
    \multicolumn{2}{c}{\scalebox{0.8}{\citeyear{wu2021autoformer}}} & 
    \multicolumn{2}{c}{\scalebox{0.8}{\citeyear{liu2021pyraformer}}} &  
    \multicolumn{2}{c}{\scalebox{0.8}{\citeyear{haoyietal-informer-2021}}} \\
\hline
    \bottomrule
  \end{tabular}
    \end{small}
  \end{threeparttable}
}
\end{table*}
\begin{table*}[htpb]
  \caption{Comparison of the complete performance with diverse prediction lengths ($\{96,192,336,720\}$) on \textbf{few-shot forecasting} task, where all samples of trainset are only partially available (5\%) in the training phase, where '-' denotes that the limited length of series cannot constitute a complete training set.}\label{tab:forecasting_few_5}
  \vspace{5pt}
  \centering
  \resizebox{0.75\columnwidth}{!}{
  \begin{small}
  \renewcommand{\multirowsetup}{\centering}
  \setlength{\tabcolsep}{1.5pt}
  \renewcommand\arraystretch{1.2}
  \begin{tabular}{cccc||cccccccccccc}
    \toprule
    \hline
    
    \multicolumn{2}{c}{\multirow{2}{*}{\scalebox{1.0}{Scenarios}}} & 
    \multicolumn{2}{c}{\rotatebox{0}{\scalebox{0.8}{\textbf{\myformer}}}} & 
    \multicolumn{2}{c}{\rotatebox{0}{\scalebox{0.8}{GPT4TS}}} & 
    \multicolumn{2}{c}{\rotatebox{0}{\scalebox{0.8}{DLinear}}} & 
    \multicolumn{2}{c}{\rotatebox{0}{\scalebox{0.8}{PatchTST}}} & 
    \multicolumn{2}{c}{\rotatebox{0}{\scalebox{0.8}{TimesNet}}} & 
    \multicolumn{2}{c}{\rotatebox{0}{\scalebox{0.8}{FEDformer}}} & 
    \multicolumn{2}{c}{\rotatebox{0}{\scalebox{0.8}{Autoformer}}} \\

    
    \cline{3-16} &  & 
    \scalebox{0.78}{MSE} & \scalebox{0.78}{MAE} & 
    \scalebox{0.78}{MSE} & \scalebox{0.78}{MAE} & 
    \scalebox{0.78}{MSE} & \scalebox{0.78}{MAE} & 
    \scalebox{0.78}{MSE} & \scalebox{0.78}{MAE} & 
    \scalebox{0.78}{MSE} & \scalebox{0.78}{MAE} & 
    \scalebox{0.78}{MSE} & \scalebox{0.78}{MAE} & 
    \scalebox{0.78}{MSE} & \scalebox{0.78}{MAE} \\
    \hline
    
    \multirow{5}{*}{\scalebox{0.9}{ETTm1}}
    
    & \scalebox{0.78}{96} & 
    {\scalebox{0.78}{0.357}} & \textbf{\scalebox{0.78}{0.350}} & 
    {\scalebox{0.78}{0.386}} & {\scalebox{0.78}{0.405}} & 
    \textbf{\scalebox{0.78}{0.332}} & {\scalebox{0.78}{0.374}} & 
    {\scalebox{0.78}{0.399}} & {\scalebox{0.78}{0.414}} & 
    {\scalebox{0.78}{0.606}} & {\scalebox{0.78}{0.518}} & 
    {\scalebox{0.78}{0.628}} & {\scalebox{0.78}{0.544}} & 
    {\scalebox{0.78}{0.726}} & {\scalebox{0.78}{0.578}} \\
    
    & \scalebox{0.78}{192} & 
    {\scalebox{0.78}{0.363}} & \textbf{\scalebox{0.78}{0.368}} & 
    {\scalebox{0.78}{0.440}} & {\scalebox{0.78}{0.438}} & 
    \textbf{\scalebox{0.78}{0.358}} & {\scalebox{0.78}{0.390}} & 
    {\scalebox{0.78}{0.441}} & {\scalebox{0.78}{0.436}} & 
    {\scalebox{0.78}{0.681}} & {\scalebox{0.78}{0.539}} & 
    {\scalebox{0.78}{0.666}} & {\scalebox{0.78}{0.566}} & 
    {\scalebox{0.78}{0.750}} & {\scalebox{0.78}{0.591}} \\
    
    & \scalebox{0.78}{336} & 
    \textbf{\scalebox{0.78}{0.401}} & \textbf{\scalebox{0.78}{0.385}} & 
    {\scalebox{0.78}{0.485}} & {\scalebox{0.78}{0.459}} & 
    {\scalebox{0.78}{0.402}} & {\scalebox{0.78}{0.416}} & 
    {\scalebox{0.78}{0.499}} & {\scalebox{0.78}{0.467}} & 
    {\scalebox{0.78}{0.786}} & {\scalebox{0.78}{0.597}} & 
    {\scalebox{0.78}{0.807}} & {\scalebox{0.78}{0.628}} & 
    {\scalebox{0.78}{0.851}} & {\scalebox{0.78}{0.659}} \\

    & \scalebox{0.78}{720} & 
    \textbf{\scalebox{0.78}{0.457}} & \textbf{\scalebox{0.78}{0.424}} & 
    {\scalebox{0.78}{0.577}} & {\scalebox{0.78}{0.499}} & 
    {\scalebox{0.78}{0.511}} & {\scalebox{0.78}{0.489}} & 
    {\scalebox{0.78}{0.767}} & {\scalebox{0.78}{0.587}} & 
    {\scalebox{0.78}{0.796}} & {\scalebox{0.78}{0.593}} & 
    {\scalebox{0.78}{0.822}} & {\scalebox{0.78}{0.633}} & 
    {\scalebox{0.78}{0.857}} & {\scalebox{0.78}{0.655}} \\

    \cline{2-16}

    & \scalebox{0.78}{avg} & 
    \textbf{\scalebox{0.78}{0.394}} & \textbf{\scalebox{0.78}{0.381}} & 
    {\scalebox{0.78}{0.472}} & {\scalebox{0.78}{0.450}} & 
    {\scalebox{0.78}{0.400}} & {\scalebox{0.78}{0.417}} & 
    {\scalebox{0.78}{0.526}} & {\scalebox{0.78}{0.476}} & 
    {\scalebox{0.78}{0.717}} & {\scalebox{0.78}{0.561}} & 
    {\scalebox{0.78}{0.730}} & {\scalebox{0.78}{0.592}} & 
    {\scalebox{0.78}{0.796}} & {\scalebox{0.78}{0.620}} \\

    \hline
    
    \multirow{5}{*}{\scalebox{0.9}{ETTm2}}
    
    & \scalebox{0.78}{96} & 
    \textbf{\scalebox{0.78}{0.175}} & \textbf{\scalebox{0.78}{0.251}} & 
    {\scalebox{0.78}{0.199}} & {\scalebox{0.78}{0.280}} & 
    {\scalebox{0.78}{0.236}} & {\scalebox{0.78}{0.326}} & 
    {\scalebox{0.78}{0.206}} & {\scalebox{0.78}{0.288}} & 
    {\scalebox{0.78}{0.220}} & {\scalebox{0.78}{0.299}} & 
    {\scalebox{0.78}{0.229}} & {\scalebox{0.78}{0.320}} & 
    {\scalebox{0.78}{0.232}} & {\scalebox{0.78}{0.322}} \\
    
    & \scalebox{0.78}{192} & 
    \textbf{\scalebox{0.78}{0.229}} & \textbf{\scalebox{0.78}{0.281}} & 
    {\scalebox{0.78}{0.256}} & {\scalebox{0.78}{0.316}} & 
    {\scalebox{0.78}{0.306}} & {\scalebox{0.78}{0.373}} & 
    {\scalebox{0.78}{0.264}} & {\scalebox{0.78}{0.324}} & 
    {\scalebox{0.78}{0.311}} & {\scalebox{0.78}{0.361}} & 
    {\scalebox{0.78}{0.394}} & {\scalebox{0.78}{0.361}} & 
    {\scalebox{0.78}{0.291}} & {\scalebox{0.78}{0.357}} \\
    
    & \scalebox{0.78}{336} & 
    \textbf{\scalebox{0.78}{0.282}} & \textbf{\scalebox{0.78}{0.314}} & 
    {\scalebox{0.78}{0.318}} & {\scalebox{0.78}{0.353}} & 
    {\scalebox{0.78}{0.380}} & {\scalebox{0.78}{0.423}} & 
    {\scalebox{0.78}{0.334}} & {\scalebox{0.78}{0.367}} & 
    {\scalebox{0.78}{0.338}} & {\scalebox{0.78}{0.366}} & 
    {\scalebox{0.78}{0.378}} & {\scalebox{0.78}{0.427}} & 
    {\scalebox{0.78}{0.478}} & {\scalebox{0.78}{0.517}} \\

    & \scalebox{0.78}{720} & 
    \textbf{\scalebox{0.78}{0.381}} & \textbf{\scalebox{0.78}{0.377}} & 
    {\scalebox{0.78}{0.460}} & {\scalebox{0.78}{0.436}} & 
    {\scalebox{0.78}{0.674}} & {\scalebox{0.78}{0.583}} & 
    {\scalebox{0.78}{0.454}} & {\scalebox{0.78}{0.432}} & 
    {\scalebox{0.78}{0.509}} & {\scalebox{0.78}{0.465}} & 
    {\scalebox{0.78}{0.523}} & {\scalebox{0.78}{0.510}} & 
    {\scalebox{0.78}{0.553}} & {\scalebox{0.78}{0.538}} \\

    \cline{2-16}

    & \scalebox{0.78}{avg} & 
    \textbf{\scalebox{0.78}{0.267}} & \textbf{\scalebox{0.78}{0.306}} & 
    {\scalebox{0.78}{0.308}} & {\scalebox{0.78}{0.346}} & 
    {\scalebox{0.78}{0.399}} & {\scalebox{0.78}{0.426}} & 
    {\scalebox{0.78}{0.314}} & {\scalebox{0.78}{0.352}} & 
    {\scalebox{0.78}{0.344}} & {\scalebox{0.78}{0.372}} & 
    {\scalebox{0.78}{0.381}} & {\scalebox{0.78}{0.404}} & 
    {\scalebox{0.78}{0.388}} & {\scalebox{0.78}{0.433}} \\

    \hline
    
    \multirow{5}{*}{\scalebox{0.9}{ETTh1}}
    
    & \scalebox{0.78}{96} & 
    \textbf{\scalebox{0.78}{0.509}} & \textbf{\scalebox{0.78}{0.440}} & 
    {\scalebox{0.78}{0.543}} & {\scalebox{0.78}{0.506}} & 
    {\scalebox{0.78}{0.547}} & {\scalebox{0.78}{0.503}} & 
    {\scalebox{0.78}{0.557}} & {\scalebox{0.78}{0.519}} & 
    {\scalebox{0.78}{0.892}} & {\scalebox{0.78}{0.625}} & 
    {\scalebox{0.78}{0.593}} & {\scalebox{0.78}{0.529}} & 
    {\scalebox{0.78}{0.681}} & {\scalebox{0.78}{0.570}} \\
    
    & \scalebox{0.78}{192} & 
    \textbf{\scalebox{0.78}{0.527}} & \textbf{\scalebox{0.78}{0.452}} & 
    {\scalebox{0.78}{0.748}} & {\scalebox{0.78}{0.580}} & 
    {\scalebox{0.78}{0.720}} & {\scalebox{0.78}{0.604}} & 
    {\scalebox{0.78}{0.711}} & {\scalebox{0.78}{0.570}} & 
    {\scalebox{0.78}{0.940}} & {\scalebox{0.78}{0.665}} & 
    {\scalebox{0.78}{0.652}} & {\scalebox{0.78}{0.563}} & 
    {\scalebox{0.78}{0.725}} & {\scalebox{0.78}{0.602}} \\
    
    & \scalebox{0.78}{336} & 
    \textbf{\scalebox{0.78}{0.528}} & \textbf{\scalebox{0.78}{0.457}} & 
    {\scalebox{0.78}{0.754}} & {\scalebox{0.78}{0.595}} & 
    {\scalebox{0.78}{0.984}} & {\scalebox{0.78}{0.727}} & 
    {\scalebox{0.78}{0.816}} & {\scalebox{0.78}{0.619}} & 
    {\scalebox{0.78}{0.945}} & {\scalebox{0.78}{0.653}} & 
    {\scalebox{0.78}{0.731}} & {\scalebox{0.78}{0.594}} & 
    {\scalebox{0.78}{0.761}} & {\scalebox{0.78}{0.624}} \\

    & \scalebox{0.78}{720} & 
    {\scalebox{0.78}{-}} & {\scalebox{0.78}{-}} & 
    {\scalebox{0.78}{-}} & {\scalebox{0.78}{-}} & 
    {\scalebox{0.78}{-}} & {\scalebox{0.78}{-}} & 
    {\scalebox{0.78}{-}} & {\scalebox{0.78}{-}} & 
    {\scalebox{0.78}{-}} & {\scalebox{0.78}{-}} & 
    {\scalebox{0.78}{-}} & {\scalebox{0.78}{-}} & 
    {\scalebox{0.78}{-}} & {\scalebox{0.78}{-}} \\

    \cline{2-16}

    & \scalebox{0.78}{avg} & 
    \textbf{\scalebox{0.78}{0.521}} & \textbf{\scalebox{0.78}{0.449}} & 
    {\scalebox{0.78}{0.681}} & {\scalebox{0.78}{0.560}} & 
    {\scalebox{0.78}{0.750}} & {\scalebox{0.78}{0.611}} & 
    {\scalebox{0.78}{0.694}} & {\scalebox{0.78}{0.569}} & 
    {\scalebox{0.78}{0.925}} & {\scalebox{0.78}{0.647}} & 
    {\scalebox{0.78}{0.658}} & {\scalebox{0.78}{0.562}} & 
    {\scalebox{0.78}{0.722}} & {\scalebox{0.78}{0.598}} \\

    \hline
    
    \multirow{5}{*}{\scalebox{0.9}{ETTh2}}
    
    & \scalebox{0.78}{96} & 
    \textbf{\scalebox{0.78}{0.337}} & \textbf{\scalebox{0.78}{0.351}} & 
    {\scalebox{0.78}{0.376}} & {\scalebox{0.78}{0.421}} & 
    {\scalebox{0.78}{0.442}} & {\scalebox{0.78}{0.456}} & 
    {\scalebox{0.78}{0.401}} & {\scalebox{0.78}{0.421}} & 
    {\scalebox{0.78}{0.409}} & {\scalebox{0.78}{0.420}} & 
    {\scalebox{0.78}{0.390}} & {\scalebox{0.78}{0.424}} & 
    {\scalebox{0.78}{0.428}} & {\scalebox{0.78}{0.468}} \\
    
    & \scalebox{0.78}{192} & 
    \textbf{\scalebox{0.78}{0.378}} & \textbf{\scalebox{0.78}{0.381}} & 
    {\scalebox{0.78}{0.418}} & {\scalebox{0.78}{0.441}} & 
    {\scalebox{0.78}{0.617}} & {\scalebox{0.78}{0.610}} & 
    {\scalebox{0.78}{0.452}} & {\scalebox{0.78}{0.455}} & 
    {\scalebox{0.78}{0.483}} & {\scalebox{0.78}{0.464}} & 
    {\scalebox{0.78}{0.457}} & {\scalebox{0.78}{0.465}} & 
    {\scalebox{0.78}{0.496}} & {\scalebox{0.78}{0.504}} \\
    
    & \scalebox{0.78}{336} & 
    \textbf{\scalebox{0.78}{0.404}} & \textbf{\scalebox{0.78}{0.410}} & 
    {\scalebox{0.78}{0.408}} & {\scalebox{0.78}{0.439}} & 
    {\scalebox{0.78}{1.424}} & {\scalebox{0.78}{0.849}} & 
    {\scalebox{0.78}{0.464}} & {\scalebox{0.78}{0.469}} & 
    {\scalebox{0.78}{0.499}} & {\scalebox{0.78}{0.479}} & 
    {\scalebox{0.78}{0.477}} & {\scalebox{0.78}{0.483}} & 
    {\scalebox{0.78}{0.486}} & {\scalebox{0.78}{0.496}} \\

    & \scalebox{0.78}{720} & 
    {\scalebox{0.78}{-}} & {\scalebox{0.78}{-}} & 
    {\scalebox{0.78}{-}} & {\scalebox{0.78}{-}} & 
    {\scalebox{0.78}{-}} & {\scalebox{0.78}{-}} & 
    {\scalebox{0.78}{-}} & {\scalebox{0.78}{-}} & 
    {\scalebox{0.78}{-}} & {\scalebox{0.78}{-}} & 
    {\scalebox{0.78}{-}} & {\scalebox{0.78}{-}} & 
    {\scalebox{0.78}{-}} & {\scalebox{0.78}{-}} \\

    \cline{2-16}

    & \scalebox{0.78}{avg} & 
    \textbf{\scalebox{0.78}{0.373}} & \textbf{\scalebox{0.78}{0.381}} & 
    {\scalebox{0.78}{0.400}} & {\scalebox{0.78}{0.433}} & 
    {\scalebox{0.78}{0.827}} & {\scalebox{0.78}{0.615}} & 
    {\scalebox{0.78}{0.439}} & {\scalebox{0.78}{0.448}} & 
    {\scalebox{0.78}{0.463}} & {\scalebox{0.78}{0.454}} & 
    {\scalebox{0.78}{0.441}} & {\scalebox{0.78}{0.457}} & 
    {\scalebox{0.78}{0.470}} & {\scalebox{0.78}{0.489}} \\

    \hline

    \multirow{5}{*}{\scalebox{0.9}{Weather}}
    
    & \scalebox{0.78}{96} & 
    \textbf{\scalebox{0.78}{0.163}} & \textbf{\scalebox{0.78}{0.209}} & 
    {\scalebox{0.78}{0.175}} & {\scalebox{0.78}{0.230}} & 
    {\scalebox{0.78}{0.184}} & {\scalebox{0.78}{0.242}} & 
    {\scalebox{0.78}{0.171}} & {\scalebox{0.78}{0.224}} & 
    {\scalebox{0.78}{0.207}} & {\scalebox{0.78}{0.253}} & 
    {\scalebox{0.78}{0.229}} & {\scalebox{0.78}{0.309}} & 
    {\scalebox{0.78}{0.227}} & {\scalebox{0.78}{0.299}} \\
    
    & \scalebox{0.78}{192} & 
    \textbf{\scalebox{0.78}{0.200}} & \textbf{\scalebox{0.78}{0.242}} & 
    {\scalebox{0.78}{0.227}} & {\scalebox{0.78}{0.276}} & 
    {\scalebox{0.78}{0.228}} & {\scalebox{0.78}{0.283}} & 
    {\scalebox{0.78}{0.230}} & {\scalebox{0.78}{0.277}} & 
    {\scalebox{0.78}{0.272}} & {\scalebox{0.78}{0.307}} & 
    {\scalebox{0.78}{0.265}} & {\scalebox{0.78}{0.317}} & 
    {\scalebox{0.78}{0.278}} & {\scalebox{0.78}{0.333}} \\
    
    & \scalebox{0.78}{336} & 
    \textbf{\scalebox{0.78}{0.252}} & \textbf{\scalebox{0.78}{0.269}} & 
    {\scalebox{0.78}{0.286}} & {\scalebox{0.78}{0.322}} & 
    {\scalebox{0.78}{0.279}} & {\scalebox{0.78}{0.322}} & 
    {\scalebox{0.78}{0.294}} & {\scalebox{0.78}{0.326}} & 
    {\scalebox{0.78}{0.313}} & {\scalebox{0.78}{0.328}} & 
    {\scalebox{0.78}{0.353}} & {\scalebox{0.78}{0.392}} & 
    {\scalebox{0.78}{0.351}} & {\scalebox{0.78}{0.393}} \\

    & \scalebox{0.78}{720} & 
    \textbf{\scalebox{0.78}{0.315}} & \textbf{\scalebox{0.78}{0.327}} & 
    {\scalebox{0.78}{0.366}} & {\scalebox{0.78}{0.379}} & 
    {\scalebox{0.78}{0.364}} & {\scalebox{0.78}{0.388}} & 
    {\scalebox{0.78}{0.384}} & {\scalebox{0.78}{0.387}} & 
    {\scalebox{0.78}{0.400}} & {\scalebox{0.78}{0.385}} & 
    {\scalebox{0.78}{0.391}} & {\scalebox{0.78}{0.394}} & 
    {\scalebox{0.78}{0.387}} & {\scalebox{0.78}{0.389}} \\

    \cline{2-16}

    & \scalebox{0.78}{avg} & 
    \textbf{\scalebox{0.78}{0.233}} & \textbf{\scalebox{0.78}{0.262}} & 
    {\scalebox{0.78}{0.263}} & {\scalebox{0.78}{0.301}} & 
    {\scalebox{0.78}{0.263}} & {\scalebox{0.78}{0.308}} & 
    {\scalebox{0.78}{0.269}} & {\scalebox{0.78}{0.303}} & 
    {\scalebox{0.78}{0.298}} & {\scalebox{0.78}{0.318}} & 
    {\scalebox{0.78}{0.309}} & {\scalebox{0.78}{0.353}} & 
    {\scalebox{0.78}{0.310}} & {\scalebox{0.78}{0.353}} \\

    \hline
    \bottomrule
  \end{tabular}
  \end{small}
}
\end{table*}
\begin{table*}[htpb]
  \caption{Comparison of the complete performance with diverse prediction lengths ($\{96,192,336,720\}$) on \textbf{few-shot forecasting} task, where all samples of trainset are only partially available (10\%) in the training phase.)}\label{tab:forecasting_few_10}
  \vspace{5pt}
  \centering
  \resizebox{0.75\columnwidth}{!}{
  \begin{small}
  \renewcommand{\multirowsetup}{\centering}
  \setlength{\tabcolsep}{1.5pt}
  \renewcommand\arraystretch{1.2}
  \begin{tabular}{cccc||cccccccccccc}
    \toprule
    \hline
    
    \multicolumn{2}{c}{\multirow{2}{*}{\scalebox{1.0}{Scenarios}}} & 
    \multicolumn{2}{c}{\rotatebox{0}{\scalebox{0.8}{\textbf{\myformer}}}} & 
    \multicolumn{2}{c}{\rotatebox{0}{\scalebox{0.8}{GPT4TS}}} & 
    \multicolumn{2}{c}{\rotatebox{0}{\scalebox{0.8}{DLinear}}} & 
    \multicolumn{2}{c}{\rotatebox{0}{\scalebox{0.8}{PatchTST}}} & 
    \multicolumn{2}{c}{\rotatebox{0}{\scalebox{0.8}{TimesNet}}} & 
    \multicolumn{2}{c}{\rotatebox{0}{\scalebox{0.8}{FEDformer}}} & 
    \multicolumn{2}{c}{\rotatebox{0}{\scalebox{0.8}{Autoformer}}} \\

    
    \cline{3-16} &  & 
    \scalebox{0.78}{MSE} & \scalebox{0.78}{MAE} & 
    \scalebox{0.78}{MSE} & \scalebox{0.78}{MAE} & 
    \scalebox{0.78}{MSE} & \scalebox{0.78}{MAE} & 
    \scalebox{0.78}{MSE} & \scalebox{0.78}{MAE} & 
    \scalebox{0.78}{MSE} & \scalebox{0.78}{MAE} & 
    \scalebox{0.78}{MSE} & \scalebox{0.78}{MAE} & 
    \scalebox{0.78}{MSE} & \scalebox{0.78}{MAE} \\
    \hline
    
    \multirow{5}{*}{\scalebox{0.9}{ETTm1}}
    
    & \scalebox{0.78}{96} & 
    \textbf{\scalebox{0.78}{0.337}} & \textbf{\scalebox{0.78}{0.350}} & 
    {\scalebox{0.78}{0.390}} & {\scalebox{0.78}{0.404}} & 
    {\scalebox{0.78}{0.352}} & {\scalebox{0.78}{0.392}} & 
    {\scalebox{0.78}{0.410}} & {\scalebox{0.78}{0.419}} & 
    {\scalebox{0.78}{0.583}} & {\scalebox{0.78}{0.501}} & 
    {\scalebox{0.78}{0.578}} & {\scalebox{0.78}{0.518}} & 
    {\scalebox{0.78}{0.774}} & {\scalebox{0.78}{0.614}} \\
    
    & \scalebox{0.78}{192} & 
    \textbf{\scalebox{0.78}{0.364}} & \textbf{\scalebox{0.78}{0.363}} & 
    {\scalebox{0.78}{0.429}} & {\scalebox{0.78}{0.423}} & 
    {\scalebox{0.78}{0.382}} & {\scalebox{0.78}{0.412}} & 
    {\scalebox{0.78}{0.437}} & {\scalebox{0.78}{0.434}} & 
    {\scalebox{0.78}{0.630}} & {\scalebox{0.78}{0.528}} & 
    {\scalebox{0.78}{0.617}} & {\scalebox{0.78}{0.546}} & 
    {\scalebox{0.78}{0.754}} & {\scalebox{0.78}{0.592}} \\
    
    & \scalebox{0.78}{336} & 
    \textbf{\scalebox{0.78}{0.389}} & \textbf{\scalebox{0.78}{0.388}} & 
    {\scalebox{0.78}{0.469}} & {\scalebox{0.78}{0.439}} & 
    {\scalebox{0.78}{0.419}} & {\scalebox{0.78}{0.434}} & 
    {\scalebox{0.78}{0.476}} & {\scalebox{0.78}{0.454}} & 
    {\scalebox{0.78}{0.725}} & {\scalebox{0.78}{0.568}} & 
    {\scalebox{0.78}{0.998}} & {\scalebox{0.78}{0.775}} & 
    {\scalebox{0.78}{0.869}} & {\scalebox{0.78}{0.677}} \\

    & \scalebox{0.78}{720} & 
    \textbf{\scalebox{0.78}{0.441}} & \textbf{\scalebox{0.78}{0.417}} & 
    {\scalebox{0.78}{0.569}} & {\scalebox{0.78}{0.498}} & 
    {\scalebox{0.78}{0.490}} & {\scalebox{0.78}{0.477}} & 
    {\scalebox{0.78}{0.681}} & {\scalebox{0.78}{0.556}} & 
    {\scalebox{0.78}{0.769}} & {\scalebox{0.78}{0.549}} & 
    {\scalebox{0.78}{0.693}} & {\scalebox{0.78}{0.579}} & 
    {\scalebox{0.78}{0.810}} & {\scalebox{0.78}{0.630}} \\

    \cline{2-16}

    & \scalebox{0.78}{avg} & 
    \textbf{\scalebox{0.78}{0.383}} & \textbf{\scalebox{0.78}{0.380}} & 
    {\scalebox{0.78}{0.464}} & {\scalebox{0.78}{0.441}} & 
    {\scalebox{0.78}{0.411}} & {\scalebox{0.78}{0.429}} & 
    {\scalebox{0.78}{0.501}} & {\scalebox{0.78}{0.466}} & 
    {\scalebox{0.78}{0.677}} & {\scalebox{0.78}{0.537}} & 
    {\scalebox{0.78}{0.722}} & {\scalebox{0.78}{0.605}} & 
    {\scalebox{0.78}{0.802}} & {\scalebox{0.78}{0.628}} \\

    \hline
    
    \multirow{5}{*}{\scalebox{0.9}{ETTm2}}
    
    & \scalebox{0.78}{96} & 
    \textbf{\scalebox{0.78}{0.170}} & \textbf{\scalebox{0.78}{0.246}} & 
    {\scalebox{0.78}{0.188}} & {\scalebox{0.78}{0.269}} & 
    {\scalebox{0.78}{0.213}} & {\scalebox{0.78}{0.303}} & 
    {\scalebox{0.78}{0.191}} & {\scalebox{0.78}{0.274}} & 
    {\scalebox{0.78}{0.212}} & {\scalebox{0.78}{0.285}} & 
    {\scalebox{0.78}{0.291}} & {\scalebox{0.78}{0.399}} & 
    {\scalebox{0.78}{0.352}} & {\scalebox{0.78}{0.454}} \\
    
    & \scalebox{0.78}{192} & 
    \textbf{\scalebox{0.78}{0.226}} & \textbf{\scalebox{0.78}{0.279}} & 
    {\scalebox{0.78}{0.251}} & {\scalebox{0.78}{0.309}} & 
    {\scalebox{0.78}{0.278}} & {\scalebox{0.78}{0.345}} & 
    {\scalebox{0.78}{0.252}} & {\scalebox{0.78}{0.317}} & 
    {\scalebox{0.78}{0.270}} & {\scalebox{0.78}{0.323}} & 
    {\scalebox{0.78}{0.307}} & {\scalebox{0.78}{0.379}} & 
    {\scalebox{0.78}{0.694}} & {\scalebox{0.78}{0.691}} \\
    
    & \scalebox{0.78}{336} & 
    \textbf{\scalebox{0.78}{0.279}} & \textbf{\scalebox{0.78}{0.312}} & 
    {\scalebox{0.78}{0.307}} & {\scalebox{0.78}{0.346}} & 
    {\scalebox{0.78}{0.338}} & {\scalebox{0.78}{0.385}} & 
    {\scalebox{0.78}{0.306}} & {\scalebox{0.78}{0.353}} & 
    {\scalebox{0.78}{0.323}} & {\scalebox{0.78}{0.353}} & 
    {\scalebox{0.78}{0.543}} & {\scalebox{0.78}{0.559}} & 
    {\scalebox{0.78}{2.408}} & {\scalebox{0.78}{1.407}} \\

    & \scalebox{0.78}{720} & 
    \textbf{\scalebox{0.78}{0.379}} & \textbf{\scalebox{0.78}{0.375}} & 
    {\scalebox{0.78}{0.426}} & {\scalebox{0.78}{0.417}} & 
    {\scalebox{0.78}{0.436}} & {\scalebox{0.78}{0.440}} & 
    {\scalebox{0.78}{0.433}} & {\scalebox{0.78}{0.427}} & 
    {\scalebox{0.78}{0.474}} & {\scalebox{0.78}{0.449}} & 
    {\scalebox{0.78}{0.712}} & {\scalebox{0.78}{0.614}} & 
    {\scalebox{0.78}{1.913}} & {\scalebox{0.78}{1.166}} \\

    \cline{2-16}

    & \scalebox{0.78}{avg} & 
    \textbf{\scalebox{0.78}{0.264}} & \textbf{\scalebox{0.78}{0.303}} & 
    {\scalebox{0.78}{0.293}} & {\scalebox{0.78}{0.335}} & 
    {\scalebox{0.78}{0.316}} & {\scalebox{0.78}{0.368}} & 
    {\scalebox{0.78}{0.296}} & {\scalebox{0.78}{0.343}} & 
    {\scalebox{0.78}{0.320}} & {\scalebox{0.78}{0.353}} & 
    {\scalebox{0.78}{0.463}} & {\scalebox{0.78}{0.488}} & 
    {\scalebox{0.78}{1.342}} & {\scalebox{0.78}{0.930}} \\

    \hline
    
    \multirow{5}{*}{\scalebox{0.9}{ETTh1}}
    
    & \scalebox{0.78}{96} & 
    \textbf{\scalebox{0.78}{0.432}} & \textbf{\scalebox{0.78}{0.409}} & 
    {\scalebox{0.78}{0.458}} & {\scalebox{0.78}{0.456}} & 
    {\scalebox{0.78}{0.492}} & {\scalebox{0.78}{0.495}} & 
    {\scalebox{0.78}{0.516}} & {\scalebox{0.78}{0.485}} & 
    {\scalebox{0.78}{0.861}} & {\scalebox{0.78}{0.628}} & 
    {\scalebox{0.78}{0.512}} & {\scalebox{0.78}{0.499}} & 
    {\scalebox{0.78}{0.613}} & {\scalebox{0.78}{0.552}} \\
    
    & \scalebox{0.78}{192} & 
    \textbf{\scalebox{0.78}{0.456}} & \textbf{\scalebox{0.78}{0.428}} & 
    {\scalebox{0.78}{0.570}} & {\scalebox{0.78}{0.516}} & 
    {\scalebox{0.78}{0.565}} & {\scalebox{0.78}{0.538}} & 
    {\scalebox{0.78}{0.598}} & {\scalebox{0.78}{0.524}} & 
    {\scalebox{0.78}{0.797}} & {\scalebox{0.78}{0.593}} & 
    {\scalebox{0.78}{0.624}} & {\scalebox{0.78}{0.555}} & 
    {\scalebox{0.78}{0.722}} & {\scalebox{0.78}{0.598}} \\
    
    & \scalebox{0.78}{336} & 
    \textbf{\scalebox{0.78}{0.489}} & \textbf{\scalebox{0.78}{0.453}} & 
    {\scalebox{0.78}{0.608}} & {\scalebox{0.78}{0.535}} & 
    {\scalebox{0.78}{0.721}} & {\scalebox{0.78}{0.622}} & 
    {\scalebox{0.78}{0.657}} & {\scalebox{0.78}{0.550}} & 
    {\scalebox{0.78}{0.941}} & {\scalebox{0.78}{0.648}} & 
    {\scalebox{0.78}{0.691}} & {\scalebox{0.78}{0.574}} & 
    {\scalebox{0.78}{0.750}} & {\scalebox{0.78}{0.619}} \\

    & \scalebox{0.78}{720} & 
    \textbf{\scalebox{0.78}{0.531}} & \textbf{\scalebox{0.78}{0.502}} & 
    {\scalebox{0.78}{0.725}} & {\scalebox{0.78}{0.591}} & 
    {\scalebox{0.78}{0.986}} & {\scalebox{0.78}{0.743}} & 
    {\scalebox{0.78}{0.762}} & {\scalebox{0.78}{0.610}} & 
    {\scalebox{0.78}{0.877}} & {\scalebox{0.78}{0.641}} & 
    {\scalebox{0.78}{0.728}} & {\scalebox{0.78}{0.614}} & 
    {\scalebox{0.78}{0.721}} & {\scalebox{0.78}{0.616}} \\

    \cline{2-16}

    & \scalebox{0.78}{avg} & 
    \textbf{\scalebox{0.78}{0.477}} & \textbf{\scalebox{0.78}{0.448}} & 
    {\scalebox{0.78}{0.590}} & {\scalebox{0.78}{0.525}} & 
    {\scalebox{0.78}{0.691}} & {\scalebox{0.78}{0.600}} & 
    {\scalebox{0.78}{0.633}} & {\scalebox{0.78}{0.542}} & 
    {\scalebox{0.78}{0.869}} & {\scalebox{0.78}{0.628}} & 
    {\scalebox{0.78}{0.639}} & {\scalebox{0.78}{0.561}} & 
    {\scalebox{0.78}{0.702}} & {\scalebox{0.78}{0.696}} \\

    \hline
    
    \multirow{5}{*}{\scalebox{0.9}{ETTh2}}
    
    & \scalebox{0.78}{96} & 
    \textbf{\scalebox{0.78}{0.308}} & \textbf{\scalebox{0.78}{0.350}} & 
    {\scalebox{0.78}{0.331}} & {\scalebox{0.78}{0.374}} & 
    {\scalebox{0.78}{0.357}} & {\scalebox{0.78}{0.411}} & 
    {\scalebox{0.78}{0.353}} & {\scalebox{0.78}{0.389}} & 
    {\scalebox{0.78}{0.378}} & {\scalebox{0.78}{0.409}} & 
    {\scalebox{0.78}{0.382}} & {\scalebox{0.78}{0.416}} & 
    {\scalebox{0.78}{0.413}} & {\scalebox{0.78}{0.451}} \\
    
    & \scalebox{0.78}{192} & 
    \textbf{\scalebox{0.78}{0.376}} & \textbf{\scalebox{0.78}{0.385}} & 
    {\scalebox{0.78}{0.402}} & {\scalebox{0.78}{0.411}} & 
    {\scalebox{0.78}{0.569}} & {\scalebox{0.78}{0.519}} & 
    {\scalebox{0.78}{0.403}} & {\scalebox{0.78}{0.414}} & 
    {\scalebox{0.78}{0.490}} & {\scalebox{0.78}{0.467}} & 
    {\scalebox{0.78}{0.478}} & {\scalebox{0.78}{0.474}} & 
    {\scalebox{0.78}{0.474}} & {\scalebox{0.78}{0.477}} \\
    
    & \scalebox{0.78}{336} & 
    {\scalebox{0.78}{0.382}} & \textbf{\scalebox{0.78}{0.406}} & 
    \textbf{\scalebox{0.78}{0.406}} & {\scalebox{0.78}{0.433}} & 
    {\scalebox{0.78}{0.671}} & {\scalebox{0.78}{0.572}} & 
    {\scalebox{0.78}{0.426}} & {\scalebox{0.78}{0.441}} & 
    {\scalebox{0.78}{0.537}} & {\scalebox{0.78}{0.494}} & 
    {\scalebox{0.78}{0.504}} & {\scalebox{0.78}{0.501}} & 
    {\scalebox{0.78}{0.547}} & {\scalebox{0.78}{0.543}} \\

    & \scalebox{0.78}{720} & 
    \textbf{\scalebox{0.78}{0.408}} & \textbf{\scalebox{0.78}{0.413}} & 
    {\scalebox{0.78}{0.449}} & {\scalebox{0.78}{0.464}} & 
    {\scalebox{0.78}{0.824}} & {\scalebox{0.78}{0.648}} & 
    {\scalebox{0.78}{0.477}} & {\scalebox{0.78}{0.480}} & 
    {\scalebox{0.78}{0.510}} & {\scalebox{0.78}{0.491}} & 
    {\scalebox{0.78}{0.499}} & {\scalebox{0.78}{0.509}} & 
    {\scalebox{0.78}{0.516}} & {\scalebox{0.78}{0.523}} \\

    \cline{2-16}

    & \scalebox{0.78}{avg} & 
    \textbf{\scalebox{0.78}{0.369}} & \textbf{\scalebox{0.78}{0.389}} & 
    {\scalebox{0.78}{0.397}} & {\scalebox{0.78}{0.421}} & 
    {\scalebox{0.78}{0.605}} & {\scalebox{0.78}{0.538}} & 
    {\scalebox{0.78}{0.415}} & {\scalebox{0.78}{0.431}} & 
    {\scalebox{0.78}{0.479}} & {\scalebox{0.78}{0.465}} & 
    {\scalebox{0.78}{0.466}} & {\scalebox{0.78}{0.475}} & 
    {\scalebox{0.78}{0.488}} & {\scalebox{0.78}{0.499}} \\

    \hline

    \multirow{5}{*}{\scalebox{0.9}{Weather}}
    
    & \scalebox{0.78}{96} & 
    \textbf{\scalebox{0.78}{0.155}} & \textbf{\scalebox{0.78}{0.205}} & 
    {\scalebox{0.78}{0.163}} & {\scalebox{0.78}{0.215}} & 
    {\scalebox{0.78}{0.171}} & {\scalebox{0.78}{0.224}} & 
    {\scalebox{0.78}{0.165}} & {\scalebox{0.78}{0.215}} & 
    {\scalebox{0.78}{0.184}} & {\scalebox{0.78}{0.230}} & 
    {\scalebox{0.78}{0.188}} & {\scalebox{0.78}{0.253}} & 
    {\scalebox{0.78}{0.221}} & {\scalebox{0.78}{0.297}} \\
    
    & \scalebox{0.78}{192} & 
    \textbf{\scalebox{0.78}{0.194}} & \textbf{\scalebox{0.78}{0.237}} & 
    {\scalebox{0.78}{0.210}} & {\scalebox{0.78}{0.254}} & 
    {\scalebox{0.78}{0.215}} & {\scalebox{0.78}{0.263}} & 
    {\scalebox{0.78}{0.210}} & {\scalebox{0.78}{0.257}} & 
    {\scalebox{0.78}{0.245}} & {\scalebox{0.78}{0.283}} & 
    {\scalebox{0.78}{0.250}} & {\scalebox{0.78}{0.304}} & 
    {\scalebox{0.78}{0.270}} & {\scalebox{0.78}{0.322}} \\
    
    & \scalebox{0.78}{336} & 
    \textbf{\scalebox{0.78}{0.249}} & \textbf{\scalebox{0.78}{0.261}} & 
    {\scalebox{0.78}{0.256}} & {\scalebox{0.78}{0.292}} & 
    {\scalebox{0.78}{0.258}} & {\scalebox{0.78}{0.299}} & 
    {\scalebox{0.78}{0.259}} & {\scalebox{0.78}{0.297}} & 
    {\scalebox{0.78}{0.305}} & {\scalebox{0.78}{0.321}} & 
    {\scalebox{0.78}{0.312}} & {\scalebox{0.78}{0.346}} & 
    {\scalebox{0.78}{0.320}} & {\scalebox{0.78}{0.351}} \\

    & \scalebox{0.78}{720} & 
    \textbf{\scalebox{0.78}{0.310}} & \textbf{\scalebox{0.78}{0.325}} & 
    {\scalebox{0.78}{0.321}} & {\scalebox{0.78}{0.339}} & 
    {\scalebox{0.78}{0.320}} & {\scalebox{0.78}{0.346}} & 
    {\scalebox{0.78}{0.332}} & {\scalebox{0.78}{0.346}} & 
    {\scalebox{0.78}{0.381}} & {\scalebox{0.78}{0.371}} & 
    {\scalebox{0.78}{0.387}} & {\scalebox{0.78}{0.393}} & 
    {\scalebox{0.78}{0.390}} & {\scalebox{0.78}{0.396}} \\

    \cline{2-16}

    & \scalebox{0.78}{avg} & 
    \textbf{\scalebox{0.78}{0.227}} & \textbf{\scalebox{0.78}{0.257}} & 
    {\scalebox{0.78}{0.238}} & {\scalebox{0.78}{0.275}} & 
    {\scalebox{0.78}{0.241}} & {\scalebox{0.78}{0.283}} & 
    {\scalebox{0.78}{0.242}} & {\scalebox{0.78}{0.279}} & 
    {\scalebox{0.78}{0.279}} & {\scalebox{0.78}{0.301}} & 
    {\scalebox{0.78}{0.284}} & {\scalebox{0.78}{0.324}} & 
    {\scalebox{0.78}{0.300}} & {\scalebox{0.78}{0.342}} \\

    \hline
    \bottomrule
  \end{tabular}
  \end{small}
}
\end{table*}

\begin{table*}[htpb]
  \caption{Comparison of the complete performance with diverse prediction lengths ($\{96,192,336,720\}$) on \textbf{zero-shot forecasting} task. Where $Source\rightarrow Traget$ indicates that the model is first pre-trained on the single train set of the \emph{SourceDomain}, subsequently, the model parameters are frozen and predicted on the test set of the \emph{TargetDomain}.}\label{tab:forecasting_zero_cross_single}
  \vspace{5pt}
  \centering
  \resizebox{0.75\columnwidth}{!}{
  \begin{small}
  \renewcommand{\multirowsetup}{\centering}
  \setlength{\tabcolsep}{1.5pt}
  \renewcommand\arraystretch{1.2}
  \begin{tabular}{cccc||cccccccccccc}
    \toprule
    \hline
    
    \multicolumn{2}{c}{\multirow{2}{*}{\scalebox{1.0}{Scenarios}}} & 
    \multicolumn{2}{c}{\rotatebox{0}{\scalebox{0.8}{\textbf{\myformer}}}} & 
    \multicolumn{2}{c}{\rotatebox{0}{\scalebox{0.8}{GPT4TS}}} & 
    \multicolumn{2}{c}{\rotatebox{0}{\scalebox{0.8}{DLinear}}} & 
    \multicolumn{2}{c}{\rotatebox{0}{\scalebox{0.8}{PatchTST}}} & 
    \multicolumn{2}{c}{\rotatebox{0}{\scalebox{0.8}{TimesNet}}} & 
    \multicolumn{2}{c}{\rotatebox{0}{\scalebox{0.8}{FEDformer}}} & 
    \multicolumn{2}{c}{\rotatebox{0}{\scalebox{0.8}{Autoformer}}} \\

    
    \cline{3-16} & &
    \scalebox{0.78}{MSE} & \scalebox{0.78}{MAE} & 
    \scalebox{0.78}{MSE} & \scalebox{0.78}{MAE} & 
    \scalebox{0.78}{MSE} & \scalebox{0.78}{MAE} & 
    \scalebox{0.78}{MSE} & \scalebox{0.78}{MAE} & 
    \scalebox{0.78}{MSE} & \scalebox{0.78}{MAE} & 
    \scalebox{0.78}{MSE} & \scalebox{0.78}{MAE} & 
    \scalebox{0.78}{MSE} & \scalebox{0.78}{MAE} \\
    \hline
    
    \multirow{5}{*}{\scalebox{0.9}{\shortstack{ETTm2\\ $\downarrow$ \\ETTm1}}}
    
    & \scalebox{0.78}{96} & 
    \textbf{\scalebox{0.78}{0.400}} & \textbf{\scalebox{0.78}{0.407}} & 
    {\scalebox{0.78}{0.706}} & {\scalebox{0.78}{0.543}} & 
    {\scalebox{0.78}{0.487}} & {\scalebox{0.78}{0.455}} & 
    {\scalebox{0.78}{0.481}} & {\scalebox{0.78}{0.441}} & 
    {\scalebox{0.78}{0.749}} & {\scalebox{0.78}{0.557}} & 
    {\scalebox{0.78}{0.691}} & {\scalebox{0.78}{0.548}} & 
    {\scalebox{0.78}{0.696}} & {\scalebox{0.78}{0.551}} \\ 
    
    & \scalebox{0.78}{192} & 
    \textbf{\scalebox{0.78}{0.415}} & \textbf{\scalebox{0.78}{0.417}} & 
    {\scalebox{0.78}{0.753}} & {\scalebox{0.78}{0.566}} & 
    {\scalebox{0.78}{0.498}} & {\scalebox{0.78}{0.461}} & 
    {\scalebox{0.78}{0.542}} & {\scalebox{0.78}{0.473}} & 
    {\scalebox{0.78}{0.908}} & {\scalebox{0.78}{0.613}} & 
    {\scalebox{0.78}{0.710}} & {\scalebox{0.78}{0.558}} & 
    {\scalebox{0.78}{0.716}} & {\scalebox{0.78}{0.562}} \\ 
    
    & \scalebox{0.78}{336} & 
    \textbf{\scalebox{0.78}{0.442}} & \textbf{\scalebox{0.78}{0.434}} & 
    {\scalebox{0.78}{0.787}} & {\scalebox{0.78}{0.581}} & 
    {\scalebox{0.78}{0.520}} & {\scalebox{0.78}{0.477}} & 
    {\scalebox{0.78}{0.599}} & {\scalebox{0.78}{0.520}} & 
    {\scalebox{0.78}{0.901}} & {\scalebox{0.78}{0.613}} & 
    {\scalebox{0.78}{0.723}} & {\scalebox{0.78}{0.567}} & 
    {\scalebox{0.78}{0.726}} & {\scalebox{0.78}{0.569}} \\ 

    & \scalebox{0.78}{720} & 
    \textbf{\scalebox{0.78}{0.481}} & \textbf{\scalebox{0.78}{0.459}} & 
    {\scalebox{0.78}{0.914}} & {\scalebox{0.78}{0.627}} & 
    {\scalebox{0.78}{0.558}} & {\scalebox{0.78}{0.501}} & 
    {\scalebox{0.78}{0.760}} & {\scalebox{0.78}{0.597}} & 
    {\scalebox{0.78}{0.871}} & {\scalebox{0.78}{0.615}} & 
    {\scalebox{0.78}{0.747}} & {\scalebox{0.78}{0.581}} & 
    {\scalebox{0.78}{0.748}} & {\scalebox{0.78}{0.582}} \\ 

    \cline{2-16}

    & \scalebox{0.78}{avg} & 
    \textbf{\scalebox{0.78}{0.434}} & \textbf{\scalebox{0.78}{0.429}} & 
    {\scalebox{0.78}{0.790}} & {\scalebox{0.78}{0.579}} & 
    {\scalebox{0.78}{0.516}} & {\scalebox{0.78}{0.473}} & 
    {\scalebox{0.78}{0.596}} & {\scalebox{0.78}{0.508}} & 
    {\scalebox{0.78}{0.857}} & {\scalebox{0.78}{0.599}} & 
    {\scalebox{0.78}{0.718}} & {\scalebox{0.78}{0.564}} & 
    {\scalebox{0.78}{0.722}} & {\scalebox{0.78}{0.566}} \\ 

    \hline
    
    \multirow{5}{*}{\scalebox{0.9}{\shortstack{ETTm1\\ $\downarrow$ \\ETTm2}}}
    
    & \scalebox{0.78}{96} & 
    \textbf{\scalebox{0.78}{0.191}} & \textbf{\scalebox{0.78}{0.263}} & 
    {\scalebox{0.78}{0.231}} & {\scalebox{0.78}{0.308}} & 
    {\scalebox{0.78}{0.260}} & {\scalebox{0.78}{0.346}} & 
    {\scalebox{0.78}{0.224}} & {\scalebox{0.78}{0.299}} & 
    {\scalebox{0.78}{0.259}} & {\scalebox{0.78}{0.327}} & 
    {\scalebox{0.78}{0.225}} & {\scalebox{0.78}{0.305}} & 
    {\scalebox{0.78}{0.233}} & {\scalebox{0.78}{0.315}} \\ 
    
    & \scalebox{0.78}{192} & 
    \textbf{\scalebox{0.78}{0.254}} & \textbf{\scalebox{0.78}{0.303}} & 
    {\scalebox{0.78}{0.307}} & {\scalebox{0.78}{0.349}} & 
    {\scalebox{0.78}{0.309}} & {\scalebox{0.78}{0.380}} & 
    {\scalebox{0.78}{0.286}} & {\scalebox{0.78}{0.338}} & 
    {\scalebox{0.78}{0.307}} & {\scalebox{0.78}{0.358}} & 
    {\scalebox{0.78}{0.285}} & {\scalebox{0.78}{0.339}} & 
    {\scalebox{0.78}{0.290}} & {\scalebox{0.78}{0.348}} \\ 
    
    & \scalebox{0.78}{336} & 
    \textbf{\scalebox{0.78}{0.312}} & \textbf{\scalebox{0.78}{0.340}} & 
    {\scalebox{0.78}{0.372}} & {\scalebox{0.78}{0.388}} & 
    {\scalebox{0.78}{0.378}} & {\scalebox{0.78}{0.426}} & 
    {\scalebox{0.78}{0.345}} & {\scalebox{0.78}{0.375}} & 
    {\scalebox{0.78}{0.372}} & {\scalebox{0.78}{0.394}} & 
    {\scalebox{0.78}{0.339}} & {\scalebox{0.78}{0.371}} & 
    {\scalebox{0.78}{0.340}} & {\scalebox{0.78}{0.374}} \\ 

    & \scalebox{0.78}{720} & 
    \textbf{\scalebox{0.78}{0.413}} & \textbf{\scalebox{0.78}{0.399}} & 
    {\scalebox{0.78}{0.457}} & {\scalebox{0.78}{0.432}} & 
    {\scalebox{0.78}{0.491}} & {\scalebox{0.78}{0.489}} & 
    {\scalebox{0.78}{0.443}} & {\scalebox{0.78}{0.430}} & 
    {\scalebox{0.78}{0.489}} & {\scalebox{0.78}{0.456}} & 
    {\scalebox{0.78}{0.437}} & {\scalebox{0.78}{0.426}} & 
    {\scalebox{0.78}{0.435}} & {\scalebox{0.78}{0.423}} \\ 

    \cline{2-16}

    & \scalebox{0.78}{avg} & 
    \textbf{\scalebox{0.78}{0.293}} & \textbf{\scalebox{0.78}{0.326}} & 
    {\scalebox{0.78}{0.342}} & {\scalebox{0.78}{0.369}} & 
    {\scalebox{0.78}{0.360}} & {\scalebox{0.78}{0.410}} & 
    {\scalebox{0.78}{0.325}} & {\scalebox{0.78}{0.361}} & 
    {\scalebox{0.78}{0.357}} & {\scalebox{0.78}{0.384}} & 
    {\scalebox{0.78}{0.321}} & {\scalebox{0.78}{0.360}} & 
    {\scalebox{0.78}{0.325}} & {\scalebox{0.78}{0.365}} \\ 

    \hline
    
    \multirow{5}{*}{\scalebox{0.9}{\shortstack{ETTh2\\ $\downarrow$ \\ETTh1}}}
    
    & \scalebox{0.78}{96} & 
    \textbf{\scalebox{0.78}{0.466}} & \textbf{\scalebox{0.78}{0.460}} & 
    {\scalebox{0.78}{0.712}} & {\scalebox{0.78}{0.562}} & 
    {\scalebox{0.78}{0.527}} & {\scalebox{0.78}{0.480}} & 
    {\scalebox{0.78}{0.545}} & {\scalebox{0.78}{0.491}} & 
    {\scalebox{0.78}{0.834}} & {\scalebox{0.78}{0.608}} & 
    {\scalebox{0.78}{0.721}} & {\scalebox{0.78}{0.573}} & 
    {\scalebox{0.78}{0.709}} & {\scalebox{0.78}{0.576}} \\ 
    
    & \scalebox{0.78}{192} & 
    \textbf{\scalebox{0.78}{0.505}} & \textbf{\scalebox{0.78}{0.483}} & 
    {\scalebox{0.78}{0.762}} & {\scalebox{0.78}{0.595}} & 
    {\scalebox{0.78}{0.587}} & {\scalebox{0.78}{0.521}} & 
    {\scalebox{0.78}{0.642}} & {\scalebox{0.78}{0.549}} & 
    {\scalebox{0.78}{0.958}} & {\scalebox{0.78}{0.637}} & 
    {\scalebox{0.78}{0.751}} & {\scalebox{0.78}{0.594}} & 
    {\scalebox{0.78}{0.749}} & {\scalebox{0.78}{0.593}} \\ 
    
    & \scalebox{0.78}{336} & 
    \textbf{\scalebox{0.78}{0.535}} & \textbf{\scalebox{0.78}{0.501}} & 
    {\scalebox{0.78}{0.815}} & {\scalebox{0.78}{0.623}} & 
    {\scalebox{0.78}{0.671}} & {\scalebox{0.78}{0.554}} & 
    {\scalebox{0.78}{0.630}} & {\scalebox{0.78}{0.541}} & 
    {\scalebox{0.78}{0.842}} & {\scalebox{0.78}{0.606}} & 
    {\scalebox{0.78}{0.765}} & {\scalebox{0.78}{0.611}} & 
    {\scalebox{0.78}{0.739}} & {\scalebox{0.78}{0.593}} \\ 

    & \scalebox{0.78}{720} & 
    \textbf{\scalebox{0.78}{0.543}} & \textbf{\scalebox{0.78}{0.529}} & 
    {\scalebox{0.78}{0.830}} & {\scalebox{0.78}{0.637}} & 
    {\scalebox{0.78}{0.652}} & {\scalebox{0.78}{0.572}} & 
    {\scalebox{0.78}{0.646}} & {\scalebox{0.78}{0.568}} & 
    {\scalebox{0.78}{1.047}} & {\scalebox{0.78}{0.689}} & 
    {\scalebox{0.78}{0.747}} & {\scalebox{0.78}{0.616}} & 
    {\scalebox{0.78}{0.742}} & {\scalebox{0.78}{0.611}} \\ 

    \cline{2-16}

    & \scalebox{0.78}{avg} & 
    \textbf{\scalebox{0.78}{0.512}} & \textbf{\scalebox{0.78}{0.493}} & 
    {\scalebox{0.78}{0.780}} & {\scalebox{0.78}{0.604}} & 
    {\scalebox{0.78}{0.609}} & {\scalebox{0.78}{0.532}} & 
    {\scalebox{0.78}{0.616}} & {\scalebox{0.78}{0.537}} & 
    {\scalebox{0.78}{0.920}} & {\scalebox{0.78}{0.635}} & 
    {\scalebox{0.78}{0.746}} & {\scalebox{0.78}{0.598}} & 
    {\scalebox{0.78}{0.735}} & {\scalebox{0.78}{0.593}} \\ 

    \hline
    
    \multirow{5}{*}{\scalebox{0.9}{\shortstack{ETTh1\\ $\downarrow$ \\ETTh2}}}
    
    & \scalebox{0.78}{96} & 
    \textbf{\scalebox{0.78}{0.297}} & \textbf{\scalebox{0.78}{0.343}} & 
    {\scalebox{0.78}{0.357}} & {\scalebox{0.78}{0.390}} & 
    {\scalebox{0.78}{0.327}} & {\scalebox{0.78}{0.387}} & 
    {\scalebox{0.78}{0.335}} & {\scalebox{0.78}{0.388}} & 
    {\scalebox{0.78}{0.388}} & {\scalebox{0.78}{0.404}} & 
    {\scalebox{0.78}{0.372}} & {\scalebox{0.78}{0.415}} & 
    {\scalebox{0.78}{0.383}} & {\scalebox{0.78}{0.419}} \\ 
    
    & \scalebox{0.78}{192} & 
    \textbf{\scalebox{0.78}{0.395}} & \textbf{\scalebox{0.78}{0.403}} & 
    {\scalebox{0.78}{0.424}} & {\scalebox{0.78}{0.424}} & 
    {\scalebox{0.78}{0.444}} & {\scalebox{0.78}{0.459}} & 
    {\scalebox{0.78}{0.419}} & {\scalebox{0.78}{0.438}} & 
    {\scalebox{0.78}{0.437}} & {\scalebox{0.78}{0.433}} & 
    {\scalebox{0.78}{0.452}} & {\scalebox{0.78}{0.462}} & 
    {\scalebox{0.78}{0.451}} & {\scalebox{0.78}{0.456}} \\ 
    
    & \scalebox{0.78}{336} & 
    \textbf{\scalebox{0.78}{0.418}} & {\scalebox{0.78}{0.429}} & 
    {\scalebox{0.78}{0.453}} & {\scalebox{0.78}{0.449}} & 
    {\scalebox{0.78}{0.513}} & {\scalebox{0.78}{0.510}} & 
    {\scalebox{0.78}{0.455}} & {\scalebox{0.78}{0.471}} & 
    {\scalebox{0.78}{0.484}} & {\scalebox{0.78}{0.465}} & 
    {\scalebox{0.78}{0.481}} & {\scalebox{0.78}{0.489}} & 
    {\scalebox{0.78}{0.486}} & {\scalebox{0.78}{0.489}} \\ 

    & \scalebox{0.78}{720} & 
    \textbf{\scalebox{0.78}{0.428}} & \textbf{\scalebox{0.78}{0.446}} & 
    {\scalebox{0.78}{0.447}} & {\scalebox{0.78}{0.458}} & 
    {\scalebox{0.78}{0.626}} & {\scalebox{0.78}{0.576}} & 
    {\scalebox{0.78}{0.456}} & {\scalebox{0.78}{0.481}} & 
    {\scalebox{0.78}{0.462}} & {\scalebox{0.78}{0.464}} & 
    {\scalebox{0.78}{0.471}} & {\scalebox{0.78}{0.488}} & 
    {\scalebox{0.78}{0.460}} & {\scalebox{0.78}{0.471}} \\ 

    \cline{2-16}

    & \scalebox{0.78}{avg} & 
    \textbf{\scalebox{0.78}{0.385}} & \textbf{\scalebox{0.78}{0.405}} & 
    {\scalebox{0.78}{0.420}} & {\scalebox{0.78}{0.430}} & 
    {\scalebox{0.78}{0.478}} & {\scalebox{0.78}{0.483}} & 
    {\scalebox{0.78}{0.416}} & {\scalebox{0.78}{0.444}} & 
    {\scalebox{0.78}{0.443}} & {\scalebox{0.78}{0.442}} & 
    {\scalebox{0.78}{0.444}} & {\scalebox{0.78}{0.463}} & 
    {\scalebox{0.78}{0.445}} & {\scalebox{0.78}{0.459}} \\ 

    \hline
    
    \multirow{5}{*}{\scalebox{0.9}{\shortstack{RiverFlow\\ $\downarrow$ \\Exchange}}}
    
    & \scalebox{0.78}{96} & 
    \textbf{\scalebox{0.78}{0.090}} & \textbf{\scalebox{0.78}{0.220}} & 
    {\scalebox{0.78}{0.131}} & {\scalebox{0.78}{0.271}} & 
    {\scalebox{0.78}{0.278}} & {\scalebox{0.78}{0.342}} & 
    {\scalebox{0.78}{0.098}} & {\scalebox{0.78}{0.232}} & 
    {\scalebox{0.78}{0.138}} & {\scalebox{0.78}{0.286}} & 
    {\scalebox{0.78}{0.520}} & {\scalebox{0.78}{0.583}} & 
    {\scalebox{0.78}{0.530}} & {\scalebox{0.78}{0.571}} \\ 
    
    & \scalebox{0.78}{192} & 
    \textbf{\scalebox{0.78}{0.198}} & \textbf{\scalebox{0.78}{0.331}} & 
    {\scalebox{0.78}{0.228}} & {\scalebox{0.78}{0.359}} & 
    {\scalebox{0.78}{0.435}} & {\scalebox{0.78}{0.438}} & 
    {\scalebox{0.78}{0.197}} & {\scalebox{0.78}{0.334}} & 
    {\scalebox{0.78}{0.245}} & {\scalebox{0.78}{0.372}} & 
    {\scalebox{0.78}{0.668}} & {\scalebox{0.78}{0.660}} & 
    {\scalebox{0.78}{0.527}} & {\scalebox{0.78}{0.587}} \\ 
    
    & \scalebox{0.78}{336} & 
    \textbf{\scalebox{0.78}{0.356}} & \textbf{\scalebox{0.78}{0.435}} & 
    {\scalebox{0.78}{0.400}} & {\scalebox{0.78}{0.477}} & 
    {\scalebox{0.78}{0.533}} & {\scalebox{0.78}{0.522}} & 
    {\scalebox{0.78}{0.357}} & {\scalebox{0.78}{0.444}} & 
    {\scalebox{0.78}{0.407}} & {\scalebox{0.78}{0.485}} & 
    {\scalebox{0.78}{0.898}} & {\scalebox{0.78}{0.760}} & 
    {\scalebox{0.78}{0.898}} & {\scalebox{0.78}{0.774}} \\ 

    & \scalebox{0.78}{720} & 
    \textbf{\scalebox{0.78}{0.879}} & \textbf{\scalebox{0.78}{0.711}} & 
    {\scalebox{0.78}{1.099}} & {\scalebox{0.78}{0.855}} & 
    {\scalebox{0.78}{1.092}} & {\scalebox{0.78}{0.846}} & 
    {\scalebox{0.78}{1.032}} & {\scalebox{0.78}{0.822}} & 
    {\scalebox{0.78}{1.195}} & {\scalebox{0.78}{0.888}} & 
    {\scalebox{0.78}{1.683}} & {\scalebox{0.78}{1.057}} & 
    {\scalebox{0.78}{1.426}} & {\scalebox{0.78}{1.022}} \\ 

    \cline{2-16}

    & \scalebox{0.78}{avg} & 
    \textbf{\scalebox{0.78}{0.381}} & \textbf{\scalebox{0.78}{0.424}} & 
    {\scalebox{0.78}{0.464}} & {\scalebox{0.78}{0.491}} & 
    {\scalebox{0.78}{0.585}} & {\scalebox{0.78}{0.537}} & 
    {\scalebox{0.78}{0.421}} & {\scalebox{0.78}{0.458}} & 
    {\scalebox{0.78}{0.497}} & {\scalebox{0.78}{0.508}} & 
    {\scalebox{0.78}{0.942}} & {\scalebox{0.78}{0.765}} & 
    {\scalebox{0.78}{0.845}} & {\scalebox{0.78}{0.739}} \\ 

    \hline
    
    \multirow{5}{*}{\scalebox{0.9}{\shortstack{Sunspot\\ $\downarrow$ \\Weather}}}
    
    & \scalebox{0.78}{96} & 
    \textbf{\scalebox{0.78}{0.181}} & \textbf{\scalebox{0.78}{0.235}} & 
    {\scalebox{0.78}{0.198}} & {\scalebox{0.78}{0.253}} & 
    {\scalebox{0.78}{0.200}} & {\scalebox{0.78}{0.268}} & 
    {\scalebox{0.78}{0.186}} & {\scalebox{0.78}{0.242}} & 
    {\scalebox{0.78}{0.252}} & {\scalebox{0.78}{0.281}} & 
    {\scalebox{0.78}{0.669}} & {\scalebox{0.78}{0.621}} & 
    {\scalebox{0.78}{0.589}} & {\scalebox{0.78}{0.581}} \\ 
    
    & \scalebox{0.78}{192} & 
    \textbf{\scalebox{0.78}{0.226}} & \textbf{\scalebox{0.78}{0.269}} & 
    {\scalebox{0.78}{0.241}} & {\scalebox{0.78}{0.283}} & 
    {\scalebox{0.78}{0.242}} & {\scalebox{0.78}{0.298}} & 
    {\scalebox{0.78}{0.235}} & {\scalebox{0.78}{0.279}} & 
    {\scalebox{0.78}{0.341}} & {\scalebox{0.78}{0.343}} & 
    {\scalebox{0.78}{0.700}} & {\scalebox{0.78}{0.632}} & 
    {\scalebox{0.78}{0.543}} & {\scalebox{0.78}{0.522}} \\ 
    
    & \scalebox{0.78}{336} & 
    \textbf{\scalebox{0.78}{0.275}} & \textbf{\scalebox{0.78}{0.301}} & 
    {\scalebox{0.78}{0.283}} & {\scalebox{0.78}{0.310}} & 
    {\scalebox{0.78}{0.282}} & {\scalebox{0.78}{0.323}} & 
    {\scalebox{0.78}{0.280}} & {\scalebox{0.78}{0.308}} & 
    {\scalebox{0.78}{0.302}} & {\scalebox{0.78}{0.321}} & 
    {\scalebox{0.78}{0.704}} & {\scalebox{0.78}{0.633}} & 
    {\scalebox{0.78}{0.467}} & {\scalebox{0.78}{0.461}} \\ 

    & \scalebox{0.78}{720} & 
    \textbf{\scalebox{0.78}{0.326}} & \textbf{\scalebox{0.78}{0.341}} & 
    {\scalebox{0.78}{0.333}} & {\scalebox{0.78}{0.343}} & 
    {\scalebox{0.78}{0.329}} & {\scalebox{0.78}{0.353}} & 
    {\scalebox{0.78}{0.352}} & {\scalebox{0.78}{0.358}} & 
    {\scalebox{0.78}{0.347}} & {\scalebox{0.78}{0.354}} & 
    {\scalebox{0.78}{0.746}} & {\scalebox{0.78}{0.651}} & 
    {\scalebox{0.78}{0.435}} & {\scalebox{0.78}{0.440}} \\ 

    \cline{2-16}

    & \scalebox{0.78}{avg} & 
    \textbf{\scalebox{0.78}{0.254}} & \textbf{\scalebox{0.78}{0.286}} & 
    {\scalebox{0.78}{0.264}} & {\scalebox{0.78}{0.297}} & 
    {\scalebox{0.78}{0.263}} & {\scalebox{0.78}{0.310}} & 
    {\scalebox{0.78}{0.263}} & {\scalebox{0.78}{0.297}} & 
    {\scalebox{0.78}{0.311}} & {\scalebox{0.78}{0.325}} & 
    {\scalebox{0.78}{0.705}} & {\scalebox{0.78}{0.634}} & 
    {\scalebox{0.78}{0.509}} & {\scalebox{0.78}{0.501}} \\ 

    \hline
    \bottomrule
  \end{tabular}
  \end{small}
}
\end{table*}
\begin{table*}[htpb]
  \caption{Comparison of the complete performance with diverse prediction lengths on \textbf{zero-shot forecasting} task, where $Source\rightarrow Traget$ indicates that the model is first pre-trained uniformly on all train sets from multiple \emph{SourceDomains}, subsequently, the model parameters are frozen and predicted on the test set of the \emph{TargetDomain}.}\label{tab:forecasting_zero_cross_multi}
  \vspace{5pt}
  \centering
  \resizebox{1\columnwidth}{!}{
  \begin{small}
  \renewcommand{\multirowsetup}{\centering}
  \setlength{\tabcolsep}{1.5pt}
  \renewcommand\arraystretch{1.2}
  \begin{tabular}{cc|cc|cccccc||cccc|cccc}
    \toprule
    \hline
    
    \multicolumn{2}{c|}{\rotatebox{0}{\scalebox{0.8}{\textbf{Scenarios}}}} & 
    \multicolumn{8}{c||}{\rotatebox{0}{\scalebox{0.8}{\textbf{Zero-shot}}}} &
    \multicolumn{4}{c|}{\rotatebox{0}{\scalebox{0.8}{\textbf{Full-data}}}} &
    \multicolumn{4}{c}{\rotatebox{0}{\scalebox{0.8}{\textbf{Few-shot}}}} \\

    \hline

    \multicolumn{2}{c|}{\rotatebox{0}{\scalebox{0.8}{\textbf{Models}}}} & 
    \multicolumn{8}{c||}{\rotatebox{0}{\scalebox{0.8}{\textbf{\myformer}}}} &
    \multicolumn{2}{c}{\rotatebox{0}{\scalebox{0.8}{\textbf{OneFitsAll}}}} & 
    \multicolumn{2}{c|}{\rotatebox{0}{\scalebox{0.8}{\textbf{PatchTST}}}} & 
    \multicolumn{2}{c}{\rotatebox{0}{\scalebox{0.8}{\textbf{OneFitsAll}}}} & 
    \multicolumn{2}{c}{\rotatebox{0}{\scalebox{0.8}{\textbf{PatchTST}}}} \\
    
    \hline
    \bottomrule
    
    \multicolumn{2}{c|}{\rotatebox{0}{\scalebox{0.8}{SourceData}}} & 
    \multicolumn{2}{c|}{\rotatebox{0}{\scalebox{0.5}{\textbf{ETT\{m2,h1,h2\}}}}} & 
    \multicolumn{2}{c}{\rotatebox{0}{\scalebox{0.7}{ETTm2}}} & 
    \multicolumn{2}{c}{\rotatebox{0}{\scalebox{0.7}{ETTh1}}} & 
    \multicolumn{2}{c||}{\rotatebox{0}{\scalebox{0.7}{ETTh2}}} & 
    \multicolumn{2}{c}{\rotatebox{0}{\scalebox{0.7}{ETTm1}}} & 
    \multicolumn{2}{c|}{\rotatebox{0}{\scalebox{0.7}{ETTm1}}} & 
    \multicolumn{2}{c}{\rotatebox{0}{\scalebox{0.7}{ETTm1}}} & 
    \multicolumn{2}{c}{\rotatebox{0}{\scalebox{0.7}{ETTm1}}} \\

    \cline{3-18} 
    
    \multicolumn{2}{c|}{\rotatebox{0}{\scalebox{0.8}{Metric}}} & 
    \scalebox{0.7}{MSE} & \scalebox{0.7}{MAE} & 
    \scalebox{0.7}{MSE} & \scalebox{0.7}{MAE} & 
    \scalebox{0.7}{MSE} & \scalebox{0.7}{MAE} & 
    \scalebox{0.7}{MSE} & \scalebox{0.7}{MAE} & 
    \scalebox{0.7}{MSE} & \scalebox{0.7}{MAE} & 
    \scalebox{0.7}{MSE} & \scalebox{0.7}{MAE} & 
    \scalebox{0.7}{MSE} & \scalebox{0.7}{MAE} & 
    \scalebox{0.7}{MSE} & \scalebox{0.7}{MAE} \\
    
    \hline

    \multirow{5}{*}{\scalebox{0.8}{\shortstack{SourceData\\ $\downarrow$ \\ETTm1}}}
    
    & \scalebox{0.78}{96} & 
    \textbf{\scalebox{0.78}{0.379}} & \textbf{\scalebox{0.78}{0.398}} & 
    {\scalebox{0.78}{0.400}} & {\scalebox{0.78}{0.407}} & 
    {\scalebox{0.78}{0.647}} & {\scalebox{0.78}{0.512}} & 
    {\scalebox{0.78}{0.733}} & {\scalebox{0.78}{0.533}} & 
    {\scalebox{0.78}{0.292}} & {\scalebox{0.78}{0.346}} & 
    \textbf{\scalebox{0.78}{0.290}} & \textbf{\scalebox{0.78}{0.342}} & 
    {\scalebox{0.78}{0.386}} & {\scalebox{0.78}{0.405}} & 
    {\scalebox{0.78}{0.399}} & {\scalebox{0.78}{0.414}} \\
    
    & \scalebox{0.78}{192} & 
    \textbf{\scalebox{0.78}{0.389}} & \textbf{\scalebox{0.78}{0.404}} & 
    {\scalebox{0.78}{0.415}} & {\scalebox{0.78}{0.417}} & 
    {\scalebox{0.78}{0.681}} & {\scalebox{0.78}{0.531}} & 
    {\scalebox{0.78}{0.740}} & {\scalebox{0.78}{0.546}} & 
    {\scalebox{0.78}{0.332}} & {\scalebox{0.78}{0.372}} & 
    \textbf{\scalebox{0.78}{0.332}} & \textbf{\scalebox{0.78}{0.369}} & 
    {\scalebox{0.78}{0.440}} & {\scalebox{0.78}{0.438}} & 
    {\scalebox{0.78}{0.441}} & {\scalebox{0.78}{0.436}} \\
    
    & \scalebox{0.78}{336} & 
    \textbf{\scalebox{0.78}{0.414}} & \textbf{\scalebox{0.78}{0.415}} & 
    {\scalebox{0.78}{0.442}} & {\scalebox{0.78}{0.434}} & 
    {\scalebox{0.78}{0.673}} & {\scalebox{0.78}{0.535}} & 
    {\scalebox{0.78}{0.742}} & {\scalebox{0.78}{0.553}} & 
    {\scalebox{0.78}{0.366}} & {\scalebox{0.78}{0.394}} & 
    \textbf{\scalebox{0.78}{0.366}} & \textbf{\scalebox{0.78}{0.392}} & 
    {\scalebox{0.78}{0.485}} & {\scalebox{0.78}{0.459}} & 
    {\scalebox{0.78}{0.499}} & {\scalebox{0.78}{0.467}} \\

    & \scalebox{0.78}{720} & 
    \textbf{\scalebox{0.78}{0.461}} & \textbf{\scalebox{0.78}{0.446}} & 
    {\scalebox{0.78}{0.481}} & {\scalebox{0.78}{0.459}} & 
    {\scalebox{0.78}{0.728}} & {\scalebox{0.78}{0.543}} & 
    {\scalebox{0.78}{0.751}} & {\scalebox{0.78}{0.570}} & 
    {\scalebox{0.78}{0.417}} & {\scalebox{0.78}{0.421}} & 
    \textbf{\scalebox{0.78}{0.416}} & \textbf{\scalebox{0.78}{0.420}} & 
    {\scalebox{0.78}{0.577}} & {\scalebox{0.78}{0.499}} & 
    {\scalebox{0.78}{0.767}} & {\scalebox{0.78}{0.587}} \\

    \cline{2-18}

    & \scalebox{0.78}{avg} & 
    \textbf{\scalebox{0.78}{0.411}} & \textbf{\scalebox{0.78}{0.416}} & 
    {\scalebox{0.78}{0.434}} & {\scalebox{0.78}{0.429}} & 
    {\scalebox{0.78}{0.682}} & {\scalebox{0.78}{0.742}} & 
    {\scalebox{0.78}{0.742}} & {\scalebox{0.78}{0.802}} & 
    {\scalebox{0.78}{0.352}} & {\scalebox{0.78}{0.383}} & 
    \textbf{\scalebox{0.78}{0.351}} & \textbf{\scalebox{0.78}{0.380}} & 
    {\scalebox{0.78}{0.472}} & {\scalebox{0.78}{0.450}} & 
    {\scalebox{0.78}{0.526}} & {\scalebox{0.78}{0.476}} \\
    
    \hline
    \bottomrule

    \multicolumn{2}{c|}{\rotatebox{0}{\scalebox{0.8}{SourceData}}} & 
    \multicolumn{2}{c|}{\rotatebox{0}{\scalebox{0.5}{\textbf{ETT\{m1,h1,h2\}}}}} & 
    \multicolumn{2}{c}{\rotatebox{0}{\scalebox{0.7}{ETTm1}}} & 
    \multicolumn{2}{c}{\rotatebox{0}{\scalebox{0.7}{ETTh1}}} & 
    \multicolumn{2}{c||}{\rotatebox{0}{\scalebox{0.7}{ETTh2}}} & 
    \multicolumn{2}{c}{\rotatebox{0}{\scalebox{0.7}{ETTm2}}} & 
    \multicolumn{2}{c|}{\rotatebox{0}{\scalebox{0.7}{ETTm2}}} & 
    \multicolumn{2}{c}{\rotatebox{0}{\scalebox{0.7}{ETTm2}}} & 
    \multicolumn{2}{c}{\rotatebox{0}{\scalebox{0.7}{ETTm2}}} \\

    \cline{3-18} 
    
    \multicolumn{2}{c|}{\rotatebox{0}{\scalebox{0.8}{Metric}}} & 
    \scalebox{0.7}{MSE} & \scalebox{0.7}{MAE} & 
    \scalebox{0.7}{MSE} & \scalebox{0.7}{MAE} & 
    \scalebox{0.7}{MSE} & \scalebox{0.7}{MAE} & 
    \scalebox{0.7}{MSE} & \scalebox{0.7}{MAE} & 
    \scalebox{0.7}{MSE} & \scalebox{0.7}{MAE} & 
    \scalebox{0.7}{MSE} & \scalebox{0.7}{MAE} & 
    \scalebox{0.7}{MSE} & \scalebox{0.7}{MAE} & 
    \scalebox{0.7}{MAE} & \scalebox{0.7}{MSE} \\
    
    \hline

    \multirow{5}{*}{\scalebox{0.8}{\shortstack{SourceData\\ $\downarrow$ \\ETTm2}}}
    
    & \scalebox{0.78}{96} & 
    \textbf{\scalebox{0.78}{0.177}} & \textbf{\scalebox{0.78}{0.258}} & 
    {\scalebox{0.78}{0.191}} & {\scalebox{0.78}{0.263}} & 
    {\scalebox{0.78}{0.214}} & {\scalebox{0.78}{0.299}} & 
    {\scalebox{0.78}{0.215}} & {\scalebox{0.78}{0.301}} & 
    {\scalebox{0.78}{0.173}} & {\scalebox{0.78}{0.262}} & 
    \textbf{\scalebox{0.78}{0.165}} & \textbf{\scalebox{0.78}{0.255}} & 
    {\scalebox{0.78}{0.199}} & {\scalebox{0.78}{0.280}} & 
    {\scalebox{0.78}{0.206}} & {\scalebox{0.78}{0.288}} \\
    
    & \scalebox{0.78}{192} & 
    \textbf{\scalebox{0.78}{0.242}} & \textbf{\scalebox{0.78}{0.296}} & 
    {\scalebox{0.78}{0.254}} & {\scalebox{0.78}{0.303}} & 
    {\scalebox{0.78}{0.278}} & {\scalebox{0.78}{0.338}} & 
    {\scalebox{0.78}{0.280}} & {\scalebox{0.78}{0.341}} & 
    {\scalebox{0.78}{0.229}} & {\scalebox{0.78}{0.301}} & 
    \textbf{\scalebox{0.78}{0.220}} & \textbf{\scalebox{0.78}{0.292}} & 
    {\scalebox{0.78}{0.256}} & {\scalebox{0.78}{0.316}} & 
    {\scalebox{0.78}{0.264}} & {\scalebox{0.78}{0.324}} \\
    
    & \scalebox{0.78}{336} & 
    \textbf{\scalebox{0.78}{0.300}} & \textbf{\scalebox{0.78}{0.332}} & 
    {\scalebox{0.78}{0.312}} & {\scalebox{0.78}{0.340}} & 
    {\scalebox{0.78}{0.331}} & {\scalebox{0.78}{0.368}} & 
    {\scalebox{0.78}{0.338}} & {\scalebox{0.78}{0.376}} & 
    {\scalebox{0.78}{0.286}} & {\scalebox{0.78}{0.341}} & 
    \textbf{\scalebox{0.78}{0.274}} & \textbf{\scalebox{0.78}{0.329}} & 
    {\scalebox{0.78}{0.318}} & {\scalebox{0.78}{0.353}} & 
    {\scalebox{0.78}{0.334}} & {\scalebox{0.78}{0.367}} \\

    & \scalebox{0.78}{720} & 
    \textbf{\scalebox{0.78}{0.400}} & \textbf{\scalebox{0.78}{0.385}} & 
    {\scalebox{0.78}{0.413}} & {\scalebox{0.78}{0.399}} & 
    {\scalebox{0.78}{0.439}} & {\scalebox{0.78}{0.430}} & 
    {\scalebox{0.78}{0.432}} & {\scalebox{0.78}{0.422}} & 
    {\scalebox{0.78}{0.378}} & {\scalebox{0.78}{0.401}} & 
    \textbf{\scalebox{0.78}{0.362}} & \textbf{\scalebox{0.78}{0.385}} & 
    {\scalebox{0.78}{0.460}} & {\scalebox{0.78}{0.436}} & 
    {\scalebox{0.78}{0.454}} & {\scalebox{0.78}{0.432}} \\

    \cline{2-18}

    & \scalebox{0.78}{avg} & 
    \textbf{\scalebox{0.78}{0.280}} & \textbf{\scalebox{0.78}{0.315}} & 
    {\scalebox{0.78}{0.292}} & {\scalebox{0.78}{0.326}} & 
    {\scalebox{0.78}{0.316}} & {\scalebox{0.78}{0.359}} & 
    {\scalebox{0.78}{0.316}} & {\scalebox{0.78}{0.360}} & 
    {\scalebox{0.78}{0.266}} & {\scalebox{0.78}{0.326}} & 
    \textbf{\scalebox{0.78}{0.255}} & \textbf{\scalebox{0.78}{0.315}} & 
    {\scalebox{0.78}{0.308}} & {\scalebox{0.78}{0.346}} & 
    {\scalebox{0.78}{0.314}} & {\scalebox{0.78}{0.352}} \\
    
    \hline
    \bottomrule

    \multicolumn{2}{c|}{\rotatebox{0}{\scalebox{0.8}{SourceData}}} & 
    \multicolumn{2}{c|}{\rotatebox{0}{\scalebox{0.5}{\textbf{ETT\{m1,m2,h2\}}}}} & 
    \multicolumn{2}{c}{\rotatebox{0}{\scalebox{0.7}{ETTm1}}} & 
    \multicolumn{2}{c}{\rotatebox{0}{\scalebox{0.7}{ETTm2}}} & 
    \multicolumn{2}{c||}{\rotatebox{0}{\scalebox{0.7}{ETTh2}}} & 
    \multicolumn{2}{c}{\rotatebox{0}{\scalebox{0.7}{ETTh1}}} & 
    \multicolumn{2}{c|}{\rotatebox{0}{\scalebox{0.7}{ETTh1}}} & 
    \multicolumn{2}{c}{\rotatebox{0}{\scalebox{0.7}{ETTh1}}} & 
    \multicolumn{2}{c}{\rotatebox{0}{\scalebox{0.7}{ETTh1}}} \\

    \cline{3-18} 
    
    \multicolumn{2}{c|}{\rotatebox{0}{\scalebox{0.8}{Metric}}} & 
    \scalebox{0.7}{MSE} & \scalebox{0.7}{MAE} & 
    \scalebox{0.7}{MSE} & \scalebox{0.7}{MAE} & 
    \scalebox{0.7}{MSE} & \scalebox{0.7}{MAE} & 
    \scalebox{0.7}{MSE} & \scalebox{0.7}{MAE} & 
    \scalebox{0.7}{MSE} & \scalebox{0.7}{MAE} & 
    \scalebox{0.7}{MSE} & \scalebox{0.7}{MAE} & 
    \scalebox{0.7}{MSE} & \scalebox{0.7}{MAE} & 
    \scalebox{0.7}{MAE} & \scalebox{0.7}{MSE} \\
    
    \hline

    \multirow{5}{*}{\scalebox{0.8}{\shortstack{SourceData\\ $\downarrow$ \\ETTh1}}}
    
    & \scalebox{0.78}{96} & 
    \textbf{\scalebox{0.78}{0.438}} & \textbf{\scalebox{0.78}{0.419}} & 
    {\scalebox{0.78}{0.466}} & {\scalebox{0.78}{0.460}} & 
    {\scalebox{0.78}{0.488}} & {\scalebox{0.78}{0.469}} & 
    {\scalebox{0.78}{0.527}} & {\scalebox{0.78}{0.480}} & 
    {\scalebox{0.78}{0.376}} & \textbf{\scalebox{0.78}{0.397}} & 
    \textbf{\scalebox{0.78}{0.370}} & {\scalebox{0.78}{0.399}} & 
    {\scalebox{0.78}{0.543}} & {\scalebox{0.78}{0.506}} & 
    {\scalebox{0.78}{0.557}} & {\scalebox{0.78}{0.519}} \\
    
    & \scalebox{0.78}{192} & 
    \textbf{\scalebox{0.78}{0.449}} & \textbf{\scalebox{0.78}{0.439}} & 
    {\scalebox{0.78}{0.505}} & {\scalebox{0.78}{0.483}} & 
    {\scalebox{0.78}{0.532}} & {\scalebox{0.78}{0.492}} & 
    {\scalebox{0.78}{0.587}} & {\scalebox{0.78}{0.521}} & 
    {\scalebox{0.78}{0.416}} & \textbf{\scalebox{0.78}{0.418}} & 
    \textbf{\scalebox{0.78}{0.413}} & {\scalebox{0.78}{0.421}} & 
    {\scalebox{0.78}{0.748}} & {\scalebox{0.78}{0.580}} & 
    {\scalebox{0.78}{0.711}} & {\scalebox{0.78}{0.570}} \\
    
    & \scalebox{0.78}{336} & 
    \textbf{\scalebox{0.78}{0.471}} & \textbf{\scalebox{0.78}{0.467}} & 
    {\scalebox{0.78}{0.535}} & {\scalebox{0.78}{0.501}} & 
    {\scalebox{0.78}{0.564}} & {\scalebox{0.78}{0.511}} & 
    {\scalebox{0.78}{0.584}} & {\scalebox{0.78}{0.527}} & 
    {\scalebox{0.78}{0.442}} & \textbf{\scalebox{0.78}{0.433}} & 
    \textbf{\scalebox{0.78}{0.422}} & {\scalebox{0.78}{0.436}} & 
    {\scalebox{0.78}{0.754}} & {\scalebox{0.78}{0.595}} & 
    {\scalebox{0.78}{0.816}} & {\scalebox{0.78}{0.619}} \\

    & \scalebox{0.78}{720} & 
    \textbf{\scalebox{0.78}{0.484}} & \textbf{\scalebox{0.78}{0.473}} & 
    {\scalebox{0.78}{0.543}} & {\scalebox{0.78}{0.529}} & 
    {\scalebox{0.78}{0.585}} & {\scalebox{0.78}{0.527}} & 
    {\scalebox{0.78}{0.612}} & {\scalebox{0.78}{0.557}} & 
    {\scalebox{0.78}{0.477}} & \textbf{\scalebox{0.78}{0.456}} & 
    \textbf{\scalebox{0.78}{0.447}} & {\scalebox{0.78}{0.466}} & 
    {\scalebox{0.78}{0.725}} & {\scalebox{0.78}{0.591}} & 
    {\scalebox{0.78}{0.762}} & {\scalebox{0.78}{0.610}} \\

    \cline{2-18}

    & \scalebox{0.78}{avg} & 
    \textbf{\scalebox{0.78}{0.461}} & \textbf{\scalebox{0.78}{0.449}} & 
    {\scalebox{0.78}{0.512}} & {\scalebox{0.78}{0.493}} & 
    {\scalebox{0.78}{0.536}} & {\scalebox{0.78}{0.499}} & 
    {\scalebox{0.78}{0.578}} & {\scalebox{0.78}{0.521}} & 
    {\scalebox{0.78}{0.427}} & \textbf{\scalebox{0.78}{0.426}} & 
    \textbf{\scalebox{0.78}{0.413}} & {\scalebox{0.78}{0.430}} & 
    {\scalebox{0.78}{0.693}} & {\scalebox{0.78}{0.568}} & 
    {\scalebox{0.78}{0.712}} & {\scalebox{0.78}{0.580}} \\
    
    \hline
    \bottomrule
    
    \multicolumn{2}{c|}{\rotatebox{0}{\scalebox{0.8}{SourceData}}} & 
    \multicolumn{2}{c|}{\rotatebox{0}{\scalebox{0.5}{\textbf{ETT\{m1,m2,h1\}}}}} & 
    \multicolumn{2}{c}{\rotatebox{0}{\scalebox{0.7}{ETTm1}}} & 
    \multicolumn{2}{c}{\rotatebox{0}{\scalebox{0.7}{ETTm2}}} & 
    \multicolumn{2}{c||}{\rotatebox{0}{\scalebox{0.7}{ETTh1}}} & 
    \multicolumn{2}{c}{\rotatebox{0}{\scalebox{0.7}{ETTh2}}} & 
    \multicolumn{2}{c|}{\rotatebox{0}{\scalebox{0.7}{ETTh2}}} & 
    \multicolumn{2}{c}{\rotatebox{0}{\scalebox{0.7}{ETTh2}}} & 
    \multicolumn{2}{c}{\rotatebox{0}{\scalebox{0.7}{ETTh2}}} \\

    \cline{3-18} 
    
    \multicolumn{2}{c|}{\rotatebox{0}{\scalebox{0.8}{Metric}}} & 
    \scalebox{0.7}{MSE} & \scalebox{0.7}{MAE} & 
    \scalebox{0.7}{MSE} & \scalebox{0.7}{MAE} & 
    \scalebox{0.7}{MSE} & \scalebox{0.7}{MAE} & 
    \scalebox{0.7}{MSE} & \scalebox{0.7}{MAE} & 
    \scalebox{0.7}{MSE} & \scalebox{0.7}{MAE} & 
    \scalebox{0.7}{MSE} & \scalebox{0.7}{MAE} & 
    \scalebox{0.7}{MSE} & \scalebox{0.7}{MAE} & 
    \scalebox{0.7}{MAE} & \scalebox{0.7}{MSE} \\
    
    \hline

    \multirow{5}{*}{\scalebox{0.8}{\shortstack{SourceData\\ $\downarrow$ \\ETTh2}}}
    
    & \scalebox{0.78}{96} & 
    \textbf{\scalebox{0.78}{0.294}} & \textbf{\scalebox{0.78}{0.357}} & 
    {\scalebox{0.78}{0.344}} & {\scalebox{0.78}{0.380}} & 
    {\scalebox{0.78}{0.326}} & {\scalebox{0.78}{0.370}} & 
    {\scalebox{0.78}{0.297}} & {\scalebox{0.78}{0.343}} & 
    {\scalebox{0.78}{0.285}} & {\scalebox{0.78}{0.342}} & 
    \textbf{\scalebox{0.78}{0.274}} & \textbf{\scalebox{0.78}{0.336}} & 
    {\scalebox{0.78}{0.376}} & {\scalebox{0.78}{0.421}} & 
    {\scalebox{0.78}{0.401}} & {\scalebox{0.78}{0.421}} \\
    
    & \scalebox{0.78}{192} & 
    \textbf{\scalebox{0.78}{0.375}} & \textbf{\scalebox{0.78}{0.397}} & 
    {\scalebox{0.78}{0.430}} & {\scalebox{0.78}{0.424}} & 
    {\scalebox{0.78}{0.411}} & {\scalebox{0.78}{0.415}} & 
    {\scalebox{0.78}{0.395}} & {\scalebox{0.78}{0.403}} & 
    {\scalebox{0.78}{0.354}} & {\scalebox{0.78}{0.389}} & 
    \textbf{\scalebox{0.78}{0.339}} & \textbf{\scalebox{0.78}{0.379}} & 
    {\scalebox{0.78}{0.418}} & {\scalebox{0.78}{0.441}} & 
    {\scalebox{0.78}{0.452}} & {\scalebox{0.78}{0.455}} \\
    
    & \scalebox{0.78}{336} & 
    \textbf{\scalebox{0.78}{0.407}} & \textbf{\scalebox{0.78}{0.407}} & 
    {\scalebox{0.78}{0.460}} & {\scalebox{0.78}{0.452}} & 
    {\scalebox{0.78}{0.450}} & {\scalebox{0.78}{0.448}} & 
    {\scalebox{0.78}{0.418}} & {\scalebox{0.78}{0.429}} & 
    {\scalebox{0.78}{0.373}} & {\scalebox{0.78}{0.407}} & 
    \textbf{\scalebox{0.78}{0.329}} & \textbf{\scalebox{0.78}{0.380}} & 
    {\scalebox{0.78}{0.408}} & {\scalebox{0.78}{0.439}} & 
    {\scalebox{0.78}{0.464}} & {\scalebox{0.78}{0.469}} \\

    & \scalebox{0.78}{720} & 
    \textbf{\scalebox{0.78}{0.408}} & \textbf{\scalebox{0.78}{0.435}} & 
    {\scalebox{0.78}{0.487}} & {\scalebox{0.78}{0.477}} & 
    {\scalebox{0.78}{0.446}} & {\scalebox{0.78}{0.455}} & 
    {\scalebox{0.78}{0.428}} & {\scalebox{0.78}{0.446}} & 
    {\scalebox{0.78}{0.406}} & {\scalebox{0.78}{0.441}} & 
    \textbf{\scalebox{0.78}{0.379}} & \textbf{\scalebox{0.78}{0.422}} & 
    {\scalebox{0.78}{0.449}} & {\scalebox{0.78}{0.464}} & 
    {\scalebox{0.78}{0.477}} & {\scalebox{0.78}{0.480}} \\
    
    \cline{2-18}
    
    & \scalebox{0.78}{avg} & 
    \textbf{\scalebox{0.78}{0.371}} & \textbf{\scalebox{0.78}{0.384}} & 
    {\scalebox{0.78}{0.430}} & {\scalebox{0.78}{0.433}} & 
    {\scalebox{0.78}{0.408}} & {\scalebox{0.78}{0.422}} & 
    {\scalebox{0.78}{0.385}} & {\scalebox{0.78}{0.405}} & 
    {\scalebox{0.78}{0.354}} & {\scalebox{0.78}{0.394}} & 
    \textbf{\scalebox{0.78}{0.330}} & \textbf{\scalebox{0.78}{0.379}} & 
    {\scalebox{0.78}{0.413}} & {\scalebox{0.78}{0.441}} & 
    {\scalebox{0.78}{0.449}} & {\scalebox{0.78}{0.456}} \\
    
    \hline
    \bottomrule

  \end{tabular}
  \end{small}
}
\end{table*}

In this section, to verify that the proposed \myformer has the potential to be a foundational model in time series, we first conduct sufficient experiments under the condition of \textbf{full-data}, as shown in Table \ref{tab:forecasting_full}.

Besides, to demonstrate that the model has excellent data efficiency and powerful cross-domain adaptability, the forecasting performance under the \textbf{few-shot} and \textbf{single-domain zero-shot} conditions are shown in Table \ref{tab:forecasting_few_5}-\ref{tab:forecasting_few_10} and Table \ref{tab:forecasting_zero_cross_single}, respectively.

Notably, to show that the proposed \textbf{\emph{wave as token}} strategy can establish underlying connections between diverse data domains and thus activate the generalization capability of the backbone network, Table \ref{tab:forecasting_zero_cross_multi} compares the performance of \textbf{single-domain adapting} and \textbf{multi-domain adapting} on the zero-shot task. The results show that the proposed strategy can alleviate the negative migration phenomenon in the time series domain.

\subsection{Imputation}\label{appendix:exp_2}
\begin{table*}[htbp]
  \caption{Comparison of the complete performance with diverse mask ratios ($\{12.5\%,25\%,37.5\%,50\%\}$) on \textbf{full-data imputation} task.}\label{tab:imputation_full}
  \vskip 0.05in
  \centering
  \resizebox{1\columnwidth}{!}{
  \begin{threeparttable}
  \begin{small}
  \renewcommand{\multirowsetup}{\centering}
  \setlength{\tabcolsep}{0.8pt}
\begin{tabular}{c|c|cc|cc|cc|cc|cc|cc|cc|cc|cc|cc}
    \toprule
    \hline
    
    \multicolumn{2}{c}{\multirow{2}{*}{Models}} & 
    \multicolumn{2}{c}{\rotatebox{0}{\scalebox{0.8}{\textbf{\myformer}}}} &
    \multicolumn{2}{c}{\rotatebox{0}{\scalebox{0.8}{GPT4TS}}} & 
    \multicolumn{2}{c}{\rotatebox{0}{\scalebox{0.8}{TimesNet}}} &
    \multicolumn{2}{c}{\rotatebox{0}{\scalebox{0.8}{PatchTST}}} & 
    \multicolumn{2}{c}{\rotatebox{0}{\scalebox{0.8}{ETSformer}}} &
    \multicolumn{2}{c}{\rotatebox{0}{\scalebox{0.8}{LightTS}}} & 
    \multicolumn{2}{c}{\rotatebox{0}{\scalebox{0.8}{DLinear}}} & 
    \multicolumn{2}{c}{\rotatebox{0}{\scalebox{0.8}{FEDformer}}} & 
    \multicolumn{2}{c}{\rotatebox{0}{\scalebox{0.8}{Stationary}}} & 
    \multicolumn{2}{c}{\rotatebox{0}{\scalebox{0.8}{Autoformer}}} \\
    
    \multicolumn{2}{c}{} & 
    \multicolumn{2}{c}{\scalebox{0.8}{(\textbf{Ours})}} & 
    \multicolumn{2}{c}{\scalebox{0.8}{\citeyearpar{Nie2023ATS}}} &  
    \multicolumn{2}{c}{\scalebox{0.8}{\citeyearpar{wu2023timesnet}}} &
    \multicolumn{2}{c}{\scalebox{0.8}{\citeyearpar{Nie2023ATS}}} &  
    \multicolumn{2}{c}{\scalebox{0.8}{\citeyearpar{wang2023micn}}} &
    \multicolumn{2}{c}{\scalebox{0.8}{\citeyearpar{Zhang2022LessIM}}} &
    \multicolumn{2}{c}{\scalebox{0.8}{\citeyearpar{Zeng2022AreTE}}} & 
    \multicolumn{2}{c}{\scalebox{0.8}{\citeyearpar{zhou2022fedformer}}} & 
    \multicolumn{2}{c}{\scalebox{0.8}{\citeyearpar{Liu2022NonstationaryTR}}} & 
    \multicolumn{2}{c}{\scalebox{0.8}{\citeyearpar{wu2021autoformer}}} \\
\hline
    \bottomrule
  \end{tabular}
    \end{small}
  \end{threeparttable}
   }
\end{table*}
\begin{table*}[htpb]
  \caption{Comparison of the complete performance with diverse mask ratios ($\{12.5\%,25\%,37.5\%,50\%\}$) on \textbf{zero-shot imputation} task. Where $Source\rightarrow Traget$ indicates that the model is first pre-trained on the single train set of the \emph{SourceDomain}, subsequently, the model parameters are frozen and predicted on the test set of the \emph{TargetDomain}.}\label{tab:imputation_zero_cross_single}
  \label{tab:full_data_limited}
  \vspace{5pt}
  \centering
  \resizebox{0.75\columnwidth}{!}{
  \begin{small}
  \renewcommand{\multirowsetup}{\centering}
  \setlength{\tabcolsep}{1.5pt}
  \renewcommand\arraystretch{1.2}
  \begin{tabular}{cccc||cccccccccccc}
    \toprule
    \hline
    
    \multicolumn{2}{c}{\multirow{2}{*}{\scalebox{1.0}{Scenarios}}} & 
    \multicolumn{2}{c}{\rotatebox{0}{\scalebox{0.8}{\textbf{\myformer}}}} & 
    \multicolumn{2}{c}{\rotatebox{0}{\scalebox{0.8}{GPT4TS}}} & 
    \multicolumn{2}{c}{\rotatebox{0}{\scalebox{0.8}{DLinear}}} & 
    \multicolumn{2}{c}{\rotatebox{0}{\scalebox{0.8}{PatchTST}}} & 
    \multicolumn{2}{c}{\rotatebox{0}{\scalebox{0.8}{TimesNet}}} & 
    \multicolumn{2}{c}{\rotatebox{0}{\scalebox{0.8}{FEDformer}}} & 
    \multicolumn{2}{c}{\rotatebox{0}{\scalebox{0.8}{Autoformer}}} \\

    
    \cline{3-16} &  & 
    \scalebox{0.78}{MSE} & \scalebox{0.78}{MAE} & 
    \scalebox{0.78}{MSE} & \scalebox{0.78}{MAE} & 
    \scalebox{0.78}{MSE} & \scalebox{0.78}{MAE} & 
    \scalebox{0.78}{MSE} & \scalebox{0.78}{MAE} & 
    \scalebox{0.78}{MSE} & \scalebox{0.78}{MAE} & 
    \scalebox{0.78}{MSE} & \scalebox{0.78}{MAE} & 
    \scalebox{0.78}{MSE} & \scalebox{0.78}{MAE} \\
    \hline
    
    \multirow{5}{*}{\scalebox{0.9}{\shortstack{ETTm2\\ $\downarrow$ \\ETTm1}}}
    
    & \scalebox{0.78}{96} & 
    \textbf{\scalebox{0.78}{0.043}} & \textbf{\scalebox{0.78}{0.134}} & 
    {\scalebox{0.78}{0.759}} & {\scalebox{0.78}{0.546}} & 
    {\scalebox{0.78}{0.118}} & {\scalebox{0.78}{0.231}} & 
    {\scalebox{0.78}{0.080}} & {\scalebox{0.78}{0.176}} & 
    {\scalebox{0.78}{0.092}} & {\scalebox{0.78}{0.180}} & 
    {\scalebox{0.78}{0.666}} & {\scalebox{0.78}{0.607}} & 
    {\scalebox{0.78}{0.367}} & {\scalebox{0.78}{0.419}} \\ 
    
    & \scalebox{0.78}{192} & 
    \textbf{\scalebox{0.78}{0.046}} & \textbf{\scalebox{0.78}{0.139}} & 
    {\scalebox{0.78}{0.767}} & {\scalebox{0.78}{0.549}} & 
    {\scalebox{0.78}{0.163}} & {\scalebox{0.78}{0.270}} & 
    {\scalebox{0.78}{0.094}} & {\scalebox{0.78}{0.191}} & 
    {\scalebox{0.78}{0.110}} & {\scalebox{0.78}{0.199}} & 
    {\scalebox{0.78}{0.721}} & {\scalebox{0.78}{0.638}} & 
    {\scalebox{0.78}{0.525}} & {\scalebox{0.78}{0.515}} \\ 
    
    & \scalebox{0.78}{336} & 
    \textbf{\scalebox{0.78}{0.050}} & \textbf{\scalebox{0.78}{0.145}} & 
    {\scalebox{0.78}{0.770}} & {\scalebox{0.78}{0.550}} & 
    {\scalebox{0.78}{0.220}} & {\scalebox{0.78}{0.313}} & 
    {\scalebox{0.78}{0.109}} & {\scalebox{0.78}{0.196}} & 
    {\scalebox{0.78}{0.125}} & {\scalebox{0.78}{0.212}} & 
    {\scalebox{0.78}{0.790}} & {\scalebox{0.78}{0.671}} & 
    {\scalebox{0.78}{0.528}} & {\scalebox{0.78}{0.517}} \\ 

    & \scalebox{0.78}{720} & 
    \textbf{\scalebox{0.78}{0.061}} & \textbf{\scalebox{0.78}{0.159}} & 
    {\scalebox{0.78}{0.772}} & {\scalebox{0.78}{0.551}} & 
    {\scalebox{0.78}{0.309}} & {\scalebox{0.78}{0.368}} & 
    {\scalebox{0.78}{0.116}} & {\scalebox{0.78}{0.202}} & 
    {\scalebox{0.78}{0.146}} & {\scalebox{0.78}{0.228}} & 
    {\scalebox{0.78}{0.871}} & {\scalebox{0.78}{0.706}} & 
    {\scalebox{0.78}{0.607}} & {\scalebox{0.78}{0.553}} \\ 

    \cline{2-16}

    & \scalebox{0.78}{avg} & 
    \textbf{\scalebox{0.78}{0.050}} & \textbf{\scalebox{0.78}{0.144}} & 
    {\scalebox{0.78}{0.767}} & {\scalebox{0.78}{0.549}} & 
    {\scalebox{0.78}{0.203}} & {\scalebox{0.78}{0.295}} & 
    {\scalebox{0.78}{0.099}} & {\scalebox{0.78}{0.191}} & 
    {\scalebox{0.78}{0.118}} & {\scalebox{0.78}{0.205}} & 
    {\scalebox{0.78}{0.762}} & {\scalebox{0.78}{0.655}} & 
    {\scalebox{0.78}{0.507}} & {\scalebox{0.78}{0.501}} \\ 

    \hline
    
    \multirow{5}{*}{\scalebox{0.9}{\shortstack{ETTm1\\ $\downarrow$ \\ETTm2}}}
    
    & \scalebox{0.78}{96} & 
    \textbf{\scalebox{0.78}{0.026}} & \textbf{\scalebox{0.78}{0.094}} & 
    {\scalebox{0.78}{0.121}} & {\scalebox{0.78}{0.236}} & 
    {\scalebox{0.78}{0.067}} & {\scalebox{0.78}{0.173}} & 
    {\scalebox{0.78}{0.055}} & {\scalebox{0.78}{0.142}} & 
    {\scalebox{0.78}{0.087}} & {\scalebox{0.78}{0.206}} & 
    {\scalebox{0.78}{1.676}} & {\scalebox{0.78}{0.989}} & 
    {\scalebox{0.78}{0.927}} & {\scalebox{0.78}{0.708}} \\ 
    
    & \scalebox{0.78}{192} & 
    \textbf{\scalebox{0.78}{0.028}} & \textbf{\scalebox{0.78}{0.095}} & 
    {\scalebox{0.78}{0.152}} & {\scalebox{0.78}{0.262}} & 
    {\scalebox{0.78}{0.100}} & {\scalebox{0.78}{0.211}} & 
    {\scalebox{0.78}{0.056}} & {\scalebox{0.78}{0.145}} & 
    {\scalebox{0.78}{0.091}} & {\scalebox{0.78}{0.213}} & 
    {\scalebox{0.78}{2.019}} & {\scalebox{0.78}{1.086}} & 
    {\scalebox{0.78}{1.148}} & {\scalebox{0.78}{0.782}} \\ 
    
    & \scalebox{0.78}{336} & 
    \textbf{\scalebox{0.78}{0.030}} & \textbf{\scalebox{0.78}{0.100}} & 
    {\scalebox{0.78}{0.153}} & {\scalebox{0.78}{0.262}} & 
    {\scalebox{0.78}{0.131}} & {\scalebox{0.78}{0.242}} & 
    {\scalebox{0.78}{0.059}} & {\scalebox{0.78}{0.151}} & 
    {\scalebox{0.78}{0.094}} & {\scalebox{0.78}{0.219}} & 
    {\scalebox{0.78}{2.309}} & {\scalebox{0.78}{1.159}} & 
    {\scalebox{0.78}{1.484}} & {\scalebox{0.78}{0.890}} \\ 

    & \scalebox{0.78}{720} & 
    \textbf{\scalebox{0.78}{0.033}} & \textbf{\scalebox{0.78}{0.104}} & 
    {\scalebox{0.78}{0.154}} & {\scalebox{0.78}{0.263}} & 
    {\scalebox{0.78}{0.160}} & {\scalebox{0.78}{0.270}} & 
    {\scalebox{0.78}{0.063}} & {\scalebox{0.78}{0.157}} & 
    {\scalebox{0.78}{0.098}} & {\scalebox{0.78}{0.226}} & 
    {\scalebox{0.78}{2.558}} & {\scalebox{0.78}{1.219}} & 
    {\scalebox{0.78}{1.808}} & {\scalebox{0.78}{0.986}} \\ 

    \cline{2-16}

    & \scalebox{0.78}{avg} & 
    \textbf{\scalebox{0.78}{0.029}} & \textbf{\scalebox{0.78}{0.098}} & 
    {\scalebox{0.78}{0.145}} & {\scalebox{0.78}{0.256}} & 
    {\scalebox{0.78}{0.114}} & {\scalebox{0.78}{0.224}} & 
    {\scalebox{0.78}{0.058}} & {\scalebox{0.78}{0.149}} & 
    {\scalebox{0.78}{0.093}} & {\scalebox{0.78}{0.216}} & 
    {\scalebox{0.78}{2.140}} & {\scalebox{0.78}{1.113}} & 
    {\scalebox{0.78}{1.342}} & {\scalebox{0.78}{0.842}} \\ 

    \hline
    
    \multirow{5}{*}{\scalebox{0.9}{\shortstack{ETTm1\\ $\downarrow$ \\ETTh1}}}
    
    & \scalebox{0.78}{96} & 
    \textbf{\scalebox{0.78}{0.128}} & \textbf{\scalebox{0.78}{0.241}} & 
    {\scalebox{0.78}{0.851}} & {\scalebox{0.78}{0.601}} & 
    {\scalebox{0.78}{0.324}} & {\scalebox{0.78}{0.384}} & 
    {\scalebox{0.78}{0.291}} & {\scalebox{0.78}{0.355}} & 
    {\scalebox{0.78}{0.278}} & {\scalebox{0.78}{0.361}} & 
    {\scalebox{0.78}{1.101}} & {\scalebox{0.78}{0.795}} & 
    {\scalebox{0.78}{0.881}} & {\scalebox{0.78}{0.696}} \\ 
    
    & \scalebox{0.78}{192} & 
    \textbf{\scalebox{0.78}{0.148}} & \textbf{\scalebox{0.78}{0.257}} & 
    {\scalebox{0.78}{0.852}} & {\scalebox{0.78}{0.602}} & 
    {\scalebox{0.78}{0.365}} & {\scalebox{0.78}{0.407}} & 
    {\scalebox{0.78}{0.301}} & {\scalebox{0.78}{0.360}} & 
    {\scalebox{0.78}{0.305}} & {\scalebox{0.78}{0.381}} & 
    {\scalebox{0.78}{1.066}} & {\scalebox{0.78}{0.783}} & 
    {\scalebox{0.78}{0.970}} & {\scalebox{0.78}{0.729}} \\ 
    
    & \scalebox{0.78}{336} & 
    \textbf{\scalebox{0.78}{0.183}} & \textbf{\scalebox{0.78}{0.283}} & 
    {\scalebox{0.78}{0.856}} & {\scalebox{0.78}{0.602}} & 
    {\scalebox{0.78}{0.416}} & {\scalebox{0.78}{0.435}} & 
    {\scalebox{0.78}{0.317}} & {\scalebox{0.78}{0.368}} & 
    {\scalebox{0.78}{0.338}} & {\scalebox{0.78}{0.403}} & 
    {\scalebox{0.78}{1.065}} & {\scalebox{0.78}{0.784}} & 
    {\scalebox{0.78}{0.962}} & {\scalebox{0.78}{0.727}} \\ 

    & \scalebox{0.78}{720} & 
    \textbf{\scalebox{0.78}{0.244}} & \textbf{\scalebox{0.78}{0.317}} & 
    {\scalebox{0.78}{0.856}} & {\scalebox{0.78}{0.602}} & 
    {\scalebox{0.78}{0.485}} & {\scalebox{0.78}{0.469}} & 
    {\scalebox{0.78}{0.342}} & {\scalebox{0.78}{0.381}} & 
    {\scalebox{0.78}{0.387}} & {\scalebox{0.78}{0.434}} & 
    {\scalebox{0.78}{1.065}} & {\scalebox{0.78}{0.785}} & 
    {\scalebox{0.78}{1.011}} & {\scalebox{0.78}{0.750}} \\ 

    \cline{2-16}

    & \scalebox{0.78}{avg} & 
    \textbf{\scalebox{0.78}{0.176}} & \textbf{\scalebox{0.78}{0.274}} & 
    {\scalebox{0.78}{0.854}} & {\scalebox{0.78}{0.602}} & 
    {\scalebox{0.78}{0.397}} & {\scalebox{0.78}{0.424}} & 
    {\scalebox{0.78}{0.313}} & {\scalebox{0.78}{0.366}} & 
    {\scalebox{0.78}{0.327}} & {\scalebox{0.78}{0.395}} & 
    {\scalebox{0.78}{1.074}} & {\scalebox{0.78}{0.787}} & 
    {\scalebox{0.78}{0.956}} & {\scalebox{0.78}{0.725}} \\ 

    \hline
    
    \multirow{5}{*}{\scalebox{0.9}{\shortstack{ETTm1\\ $\downarrow$ \\ETTh2}}}
    
    & \scalebox{0.78}{96} & 
    \textbf{\scalebox{0.78}{0.059}} & \textbf{\scalebox{0.78}{0.152}} & 
    {\scalebox{0.78}{0.232}} & {\scalebox{0.78}{0.325}} & 
    {\scalebox{0.78}{0.118}} & {\scalebox{0.78}{0.238}} & 
    {\scalebox{0.78}{0.073}} & {\scalebox{0.78}{0.176}} & 
    {\scalebox{0.78}{0.098}} & {\scalebox{0.78}{0.223}} & 
    {\scalebox{0.78}{2.489}} & {\scalebox{0.78}{1.194}} & 
    {\scalebox{0.78}{2.161}} & {\scalebox{0.78}{1.127}} \\ 
    
    & \scalebox{0.78}{192} & 
    \textbf{\scalebox{0.78}{0.061}} & \textbf{\scalebox{0.78}{0.156}} & 
    {\scalebox{0.78}{0.249}} & {\scalebox{0.78}{0.336}} & 
    {\scalebox{0.78}{0.145}} & {\scalebox{0.78}{0.265}} & 
    {\scalebox{0.78}{0.076}} & {\scalebox{0.78}{0.180}} & 
    {\scalebox{0.78}{0.106}} & {\scalebox{0.78}{0.234}} & 
    {\scalebox{0.78}{2.767}} & {\scalebox{0.78}{1.260}} & 
    {\scalebox{0.78}{2.750}} & {\scalebox{0.78}{1.276}} \\ 
    
    & \scalebox{0.78}{336} & 
    \textbf{\scalebox{0.78}{0.065}} & \textbf{\scalebox{0.78}{0.161}} & 
    {\scalebox{0.78}{0.249}} & {\scalebox{0.78}{0.336}} & 
    {\scalebox{0.78}{0.174}} & {\scalebox{0.78}{0.291}} & 
    {\scalebox{0.78}{0.080}} & {\scalebox{0.78}{0.186}} & 
    {\scalebox{0.78}{0.113}} & {\scalebox{0.78}{0.242}} & 
    {\scalebox{0.78}{2.916}} & {\scalebox{0.78}{1.294}} & 
    {\scalebox{0.78}{2.407}} & {\scalebox{0.78}{1.189}} \\ 

    & \scalebox{0.78}{720} & 
    \textbf{\scalebox{0.78}{0.072}} & \textbf{\scalebox{0.78}{0.172}} & 
    {\scalebox{0.78}{0.249}} & {\scalebox{0.78}{0.336}} & 
    {\scalebox{0.78}{0.204}} & {\scalebox{0.78}{0.316}} & 
    {\scalebox{0.78}{0.087}} & {\scalebox{0.78}{0.193}} & 
    {\scalebox{0.78}{0.120}} & {\scalebox{0.78}{0.252}} & 
    {\scalebox{0.78}{3.014}} & {\scalebox{0.78}{1.316}} & 
    {\scalebox{0.78}{2.576}} & {\scalebox{0.78}{1.230}} \\ 

    \cline{2-16}

    & \scalebox{0.78}{avg} & 
    \textbf{\scalebox{0.78}{0.064}} & \textbf{\scalebox{0.78}{0.160}} & 
    {\scalebox{0.78}{0.245}} & {\scalebox{0.78}{0.333}} & 
    {\scalebox{0.78}{0.160}} & {\scalebox{0.78}{0.277}} & 
    {\scalebox{0.78}{0.079}} & {\scalebox{0.78}{0.184}} & 
    {\scalebox{0.78}{0.109}} & {\scalebox{0.78}{0.238}} & 
    {\scalebox{0.78}{2.796}} & {\scalebox{0.78}{1.266}} & 
    {\scalebox{0.78}{2.473}} & {\scalebox{0.78}{1.206}} \\ 

    \hline
    
    \multirow{5}{*}{\scalebox{0.9}{\shortstack{ETTm1\\ $\downarrow$ \\Exchange}}}
    
    & \scalebox{0.78}{96} & 
    \textbf{\scalebox{0.78}{0.003}} & \textbf{\scalebox{0.78}{0.029}} & 
    {\scalebox{0.78}{0.027}} & {\scalebox{0.78}{0.117}} & 
    {\scalebox{0.78}{0.086}} & {\scalebox{0.78}{0.223}} & 
    {\scalebox{0.78}{0.005}} & {\scalebox{0.78}{0.038}} & 
    {\scalebox{0.78}{0.045}} & {\scalebox{0.78}{0.149}} & 
    {\scalebox{0.78}{3.126}} & {\scalebox{0.78}{1.443}} & 
    {\scalebox{0.78}{2.672}} & {\scalebox{0.78}{1.336}} \\ 
    
    & \scalebox{0.78}{192} & 
    \textbf{\scalebox{0.78}{0.003}} & \textbf{\scalebox{0.78}{0.029}} & 
    {\scalebox{0.78}{0.027}} & {\scalebox{0.78}{0.117}} & 
    {\scalebox{0.78}{0.217}} & {\scalebox{0.78}{0.361}} & 
    {\scalebox{0.78}{0.005}} & {\scalebox{0.78}{0.042}} & 
    {\scalebox{0.78}{0.047}} & {\scalebox{0.78}{0.153}} & 
    {\scalebox{0.78}{3.114}} & {\scalebox{0.78}{1.437}} & 
    {\scalebox{0.78}{3.005}} & {\scalebox{0.78}{1.426}} \\ 
    
    & \scalebox{0.78}{336} & 
    \textbf{\scalebox{0.78}{0.003}} & \textbf{\scalebox{0.78}{0.032}} & 
    {\scalebox{0.78}{0.027}} & {\scalebox{0.78}{0.117}} & 
    {\scalebox{0.78}{0.425}} & {\scalebox{0.78}{0.508}} & 
    {\scalebox{0.78}{0.006}} & {\scalebox{0.78}{0.046}} & 
    {\scalebox{0.78}{0.046}} & {\scalebox{0.78}{0.151}} & 
    {\scalebox{0.78}{3.096}} & {\scalebox{0.78}{1.438}} & 
    {\scalebox{0.78}{2.968}} & {\scalebox{0.78}{1.376}} \\ 

    & \scalebox{0.78}{720} & 
    \textbf{\scalebox{0.78}{0.004}} & \textbf{\scalebox{0.78}{0.033}} & 
    {\scalebox{0.78}{0.027}} & {\scalebox{0.78}{0.117}} & 
    {\scalebox{0.78}{0.705}} & {\scalebox{0.78}{0.656}} & 
    {\scalebox{0.78}{0.007}} & {\scalebox{0.78}{0.051}} & 
    {\scalebox{0.78}{0.044}} & {\scalebox{0.78}{0.149}} & 
    {\scalebox{0.78}{3.092}} & {\scalebox{0.78}{1.441}} & 
    {\scalebox{0.78}{2.973}} & {\scalebox{0.78}{1.391}} \\ 

    \cline{2-16}

    & \scalebox{0.78}{avg} & 
    \textbf{\scalebox{0.78}{0.003}} & \textbf{\scalebox{0.78}{0.031}} & 
    {\scalebox{0.78}{0.027}} & {\scalebox{0.78}{0.117}} & 
    {\scalebox{0.78}{0.358}} & {\scalebox{0.78}{0.437}} & 
    {\scalebox{0.78}{0.006}} & {\scalebox{0.78}{0.044}} & 
    {\scalebox{0.78}{0.045}} & {\scalebox{0.78}{0.150}} & 
    {\scalebox{0.78}{3.107}} & {\scalebox{0.78}{1.440}} & 
    {\scalebox{0.78}{2.904}} & {\scalebox{0.78}{1.382}} \\ 

    \hline
    
    \multirow{5}{*}{\scalebox{0.9}{\shortstack{ETTm1\\ $\downarrow$ \\Weather}}}
    
    & \scalebox{0.78}{96} & 
    \textbf{\scalebox{0.78}{0.027}} & \textbf{\scalebox{0.78}{0.041}} & 
    {\scalebox{0.78}{0.102}} & {\scalebox{0.78}{0.159}} & 
    {\scalebox{0.78}{0.119}} & {\scalebox{0.78}{0.207}} & 
    {\scalebox{0.78}{0.060}} & {\scalebox{0.78}{0.089}} & 
    {\scalebox{0.78}{0.132}} & {\scalebox{0.78}{0.188}} & 
    {\scalebox{0.78}{0.991}} & {\scalebox{0.78}{0.777}} & 
    {\scalebox{0.78}{1.002}} & {\scalebox{0.78}{0.792}} \\ 
    
    & \scalebox{0.78}{192} & 
    \textbf{\scalebox{0.78}{0.029}} & \textbf{\scalebox{0.78}{0.041}} & 
    {\scalebox{0.78}{0.102}} & {\scalebox{0.78}{0.160}} & 
    {\scalebox{0.78}{0.145}} & {\scalebox{0.78}{0.252}} & 
    {\scalebox{0.78}{0.064}} & {\scalebox{0.78}{0.097}} & 
    {\scalebox{0.78}{0.132}} & {\scalebox{0.78}{0.188}} & 
    {\scalebox{0.78}{0.992}} & {\scalebox{0.78}{0.781}} & 
    {\scalebox{0.78}{0.950}} & {\scalebox{0.78}{0.756}} \\ 
    
    & \scalebox{0.78}{336} & 
    \textbf{\scalebox{0.78}{0.031}} & \textbf{\scalebox{0.78}{0.044}} & 
    {\scalebox{0.78}{0.103}} & {\scalebox{0.78}{0.160}} & 
    {\scalebox{0.78}{0.187}} & {\scalebox{0.78}{0.306}} & 
    {\scalebox{0.78}{0.067}} & {\scalebox{0.78}{0.103}} & 
    {\scalebox{0.78}{0.132}} & {\scalebox{0.78}{0.187}} & 
    {\scalebox{0.78}{1.008}} & {\scalebox{0.78}{0.779}} & 
    {\scalebox{0.78}{1.039}} & {\scalebox{0.78}{0.804}} \\ 

    & \scalebox{0.78}{720} & 
    \textbf{\scalebox{0.78}{0.034}} & \textbf{\scalebox{0.78}{0.046}} & 
    {\scalebox{0.78}{0.104}} & {\scalebox{0.78}{0.160}} & 
    {\scalebox{0.78}{0.244}} & {\scalebox{0.78}{0.365}} & 
    {\scalebox{0.78}{0.071}} & {\scalebox{0.78}{0.109}} & 
    {\scalebox{0.78}{0.134}} & {\scalebox{0.78}{0.187}} & 
    {\scalebox{0.78}{1.006}} & {\scalebox{0.78}{0.778}} & 
    {\scalebox{0.78}{1.010}} & {\scalebox{0.78}{0.799}} \\ 

    \cline{2-16}

    & \scalebox{0.78}{avg} & 
    \textbf{\scalebox{0.78}{0.030}} & \textbf{\scalebox{0.78}{0.043}} & 
    {\scalebox{0.78}{0.103}} & {\scalebox{0.78}{0.160}} & 
    {\scalebox{0.78}{0.174}} & {\scalebox{0.78}{0.283}} & 
    {\scalebox{0.78}{0.065}} & {\scalebox{0.78}{0.099}} & 
    {\scalebox{0.78}{0.132}} & {\scalebox{0.78}{0.188}} & 
    {\scalebox{0.78}{0.999}} & {\scalebox{0.78}{0.779}} & 
    {\scalebox{0.78}{1.000}} & {\scalebox{0.78}{0.788}} \\

    \hline
    \bottomrule
  \end{tabular}
  \end{small}
}
\end{table*}

Same as the forecasting task, the performance of the imputation task under \textbf{full-data} and \textbf{zero-shot} conditions are shown in Table \ref{tab:imputation_full} and Table \ref{tab:imputation_zero_cross_single}, respectively.

\subsection{Classification}\label{appendix:exp_3}
\begin{table*}[htbp]
  \caption{Model comparison in classification. The experimental results contain the performance of the three most representative approaches (OneFitsAll, PatchTST, and TimesNet) and proposed \myformer on all 35 classification datasets, which includes four metrics: accuracy (Acc), precision (Pre), recall (Rec), and F1 score (F1).}\label{tab:classification_full}
  \vskip 0.05in
  \centering
  \resizebox{1.0\columnwidth}{!}{
  \begin{threeparttable}
  \begin{small}
  \renewcommand{\multirowsetup}{\centering}
  \setlength{\tabcolsep}{1pt}
\begin{tabular}{c|cccc|c||cccc|c||cccc|c||cccc|c}
    \toprule
    \hline
    
    \multicolumn{1}{c}{\multirow{2}{*}{Models}} & 
    \multicolumn{5}{c}{\rotatebox{0}{\scalebox{0.8}{\myformer}}} &
    \multicolumn{5}{c}{\rotatebox{0}{\scalebox{0.8}{OneFitsAll}}} &
    \multicolumn{5}{c}{\rotatebox{0}{\scalebox{0.8}{PatchTST}}} &
    \multicolumn{5}{c}{\rotatebox{0}{\scalebox{0.8}{TimesNet}}} \\
    
    \multicolumn{1}{c}{} & 
    \multicolumn{5}{c}{\scalebox{0.8}{(\textbf{ours})}} & 
    \multicolumn{5}{c}{\scalebox{0.8}{\citeyearpar{Zhou2023OneFA}}} & 
    \multicolumn{5}{c}{\scalebox{0.8}{\citeyearpar{Nie2023ATS}}} & 
    \multicolumn{5}{c}{\scalebox{0.8}{\citeyearpar{wu2023timesnet}}} \\
    
    \cline{3-21}
    
    \multicolumn{1}{c}{Metric (\%)} & 
    \scalebox{0.78}{Acc} &\scalebox{0.78}{Pre} &\scalebox{0.78}{Rec} &\scalebox{0.78}{F1} &\scalebox{0.78}{Avg} & 
    \scalebox{0.78}{Acc} &\scalebox{0.78}{Pre} &\scalebox{0.78}{Rec} &\scalebox{0.78}{F1} &\scalebox{0.78}{Avg} & 
    \scalebox{0.78}{Acc} &\scalebox{0.78}{Pre} &\scalebox{0.78}{Rec} &\scalebox{0.78}{F1} &\scalebox{0.78}{Avg} & 
    \scalebox{0.78}{Acc} &\scalebox{0.78}{Pre} &\scalebox{0.78}{Rec} &\scalebox{0.78}{F1} &\scalebox{0.78}{Avg} \\
    
    \hline
    
    \scalebox{0.78}{BeetleFly} &
    \scalebox{0.78}{80.00} &\scalebox{0.78}{80.00} &\scalebox{0.78}{80.00} &\scalebox{0.78}{80.00} &\scalebox{0.78}{80.00} &
    \textbf{\scalebox{0.78}{85.00}} &\textbf{\scalebox{0.78}{88.46}} &\textbf{\scalebox{0.78}{85.00}} &\textbf{\scalebox{0.78}{84.65}} &\textbf{\scalebox{0.78}{85.78}} &
    \scalebox{0.78}{80.00} &\scalebox{0.78}{83.46} &\scalebox{0.78}{80.00} &\scalebox{0.78}{79.65} &\scalebox{0.78}{80.78} &
    \scalebox{0.78}{75.00} &\scalebox{0.78}{83.33} &\scalebox{0.78}{75.00} &\scalebox{0.78}{73.33} &\scalebox{0.78}{76.66} \\
\hline
    \bottomrule
  \end{tabular}
    \end{small}
  \end{threeparttable}
  }
\end{table*}
\begin{table*}[htbp]
\caption{Comparison (\textbf{Part-1}) of the complete performance on \textbf{few-shot classification} task. Where \emph{SourceDomain}$\rightarrow$\emph{TragetDomain} indicates that the model is first pre-trained in the \emph{SourceDomain} train set, subsequently, the parameters are fine-tuned in partial (5\%/10\%) samples of the \emph{TargetDomain} train set and finally predicted on the \emph{TargetDomain} test set.}\label{tab:classification_few_1}
\vskip 0.1in
\centering
\resizebox{0.9\columnwidth}{!}{
\begin{threeparttable}
\begin{small}
\renewcommand{\multirowsetup}{\centering}
\setlength{\tabcolsep}{7pt}
\begin{tabular}{c|c|c|cccc|c}
\toprule
\hline

\multicolumn{2}{c}{Scenarios} & Models & Accuracy (\%) & Precision (\%) & Recall (\%) & F1 (\%) & Avg (\%) \\

\midrule
\multirow{16}{*}[-3ex]{\rotatebox{90}{\shortstack{DistalPhalanxTW \\ $\downarrow$ \\ ProximalPhalanxOutlineCorrect}}}

& \multirow{4}{*}[0ex]{\shortstack{Full Data}}
& MICN & 84.88 & 85.82 & \textbf{78.42} & \textbf{80.73} & 82.46 \\
& & TimesNet & \textbf{85.22} & \textbf{88.80} & 77.51 & 80.43 & \textbf{82.99} \\
& & PatchTST & 78.69 & 75.80 & 73.03 & 74.06 & 75.39 \\
& & OneFitsAll & 79.73 & 76.00 & 76.12 & 76.35 & 77.20 \\
\cmidrule(lr){2-8}

& \multirow{6}{*}[-1ex]{\shortstack{Few-shot \\ (10\%)}}
& Random init. & 75.26 & 71.93 & 67.01 & 68.21 & 70.60 \\
\cmidrule(lr){3-8}
& & MICN & 80.38 & 79.46 & 78.00 & 78.67 & 79.13 \\
& & TimesNet & 77.32 & 73.77 & 73.19 & 73.46 & 74.44 \\
& & PatchTST & 78.69 & 81.61 & 68.06 & 69.95 & 74.57 \\
& & OneFitsAll & 77.66 & 74.35 & 72.27 & 73.09 & 74.34 \\
& & \myformer & \textbf{86.25} & \textbf{84.11} & \textbf{84.11} & \textbf{84.11} & \textbf{84.65} \\
\cmidrule(lr){2-8}

& \multirow{6}{*}[-1ex]{\shortstack{Few-shot \\ (5\%)}}
& Random init. & 68.38 & 34.19 & 50.00 & 40.61 & 48.30 \\
\cmidrule(lr){3-8}
& & MICN & 81.79 & 80.32 & 75.87 & 77.42 & 78.85 \\
& & TimesNet & 78.01 & 78.99 & 67.85 & 69.60 & 73.61 \\
& & PatchTST & 78.35 & 76.20 & 71.02 & 72.52 & 74.52 \\
& & OneFitsAll & 76.98 & 74.22 & 69.43 & 70.78 & 72.85 \\
& & \myformer & \textbf{84.54} & \textbf{87.68} & \textbf{76.71} & \textbf{79.52} & \textbf{82.11} \\





\midrule
\multirow{16}{*}[-3ex]{\rotatebox{90}{\shortstack{SonyAIBORobotSurface2 \\ $\downarrow$ \\ Chinatown}}}

& \multirow{4}{*}[0ex]{\shortstack{Full Data}}
& MICN & 63.27 & 55.17 & 55.49 & 55.27 & 57.30 \\
& & TimesNet & \textbf{97.96} & \textbf{96.53} & \textbf{98.59} & \textbf{97.49} & \textbf{97.64} \\
& & PatchTST & 95.63 & 94.38 & 94.67 & 94.52 & 94.80 \\
& & OneFitsAll & 96.79 & 94.94 & 97.46 & 96.08 & 96.32 \\
\cmidrule(lr){2-8}

& \multirow{6}{*}[-1ex]{\shortstack{Few-shot \\ (10\%)}}
& Random init. & 66.79 & 64.76 & 67.79 & 66.11 & 66.36 \\
\cmidrule(lr){3-8}
& & MICN & 27.41 & 13.70 & 50.00 & 21.51 & 28.15 \\
& & TimesNet & 91.25 & 89.65 & 88.02 & 88.78 & 89.43 \\
& & PatchTST & 80.76 & 87.95 & 65.22 & 67.52 & 75.36 \\
& & OneFitsAll & 78.13 & 74.21 & 65.40 & 67.21 & 71.24 \\
& & \myformer & \textbf{94.17} & \textbf{91.23} & \textbf{95.98} & \textbf{93.10} & \textbf{93.62} \\
\cmidrule(lr){2-8}

& \multirow{6}{*}[-1ex]{\shortstack{Few-shot \\ (5\%)}}
& Random init. & 27.41 & 13.70 & 50.00 & 21.51 & 28.16 \\
\cmidrule(lr){3-8}
& & MICN & 27.41 & 13.70 & 50.00 & 21.51 & 28.15 \\
& & TimesNet & 27.41 & 13.70 & 50.00 & 21.51 & 28.15 \\
& & PatchTST & 81.05 & 84.63 & 66.75 & 69.33 & 75.44 \\
& & OneFitsAll & 27.41 & 13.70 & 50.00 & 21.51 & 28.15 \\
& & \myformer & \textbf{91.55} & \textbf{90.16} & \textbf{88.22} & \textbf{89.12} & \textbf{89.76} \\

\midrule
\multirow{16}{*}[-3ex]{\rotatebox{90}{\shortstack{Trace \\ $\downarrow$ \\ DistalPhalanxOutlineCorrect}}}

& \multirow{4}{*}[0ex]{\shortstack{Full Data}}
& MICN & \textbf{80.07} & \textbf{81.39} & \textbf{77.70} & \textbf{78.48} & \textbf{79.41} \\
& & TimesNet & 76.45 & 75.78 & 75.96 & 75.86 & 76.01 \\
& & PatchTST & 71.74 & 72.44 & 68.57 & 68.85 & 70.40 \\
& & OneFitsAll & 74.28 & 77.38 & 70.50 & 70.82 & 73.25 \\
\cmidrule(lr){2-8}

& \multirow{6}{*}[-1ex]{\shortstack{Few-shot \\ (10\%)}}
& Random init. & 58.33 & 29.17 & 50.00 & 36.84 & 43.59 \\
\cmidrule(lr){3-8}
& & MICN & 75.00 & 75.48 & 72.48 & 73.01 & 73.99 \\
& & TimesNet & 71.74 & 72.24 & 68.70 & 69.00 & 70.42 \\
& & PatchTST & 68.48 & 71.72 & 63.66 & 62.65 & 66.62 \\
& & OneFitsAll & 71.74 & 70.97 & 70.19 & 70.43 & 70.83 \\
& & \myformer & \textbf{75.36} & \textbf{76.68} & \textbf{74.42} & \textbf{74.98} & \textbf{75.36} \\
\cmidrule(lr){2-8}

& \multirow{6}{*}[-1ex]{\shortstack{Few-shot \\ (5\%)}}
& Random init. & 58.33 & 29.17 & 50.00 & 36.84 & 43.59 \\
\cmidrule(lr){3-8}
& & MICN & 64.49 & 64.14 & 60.25 & 59.37 & 62.06 \\
& & TimesNet & 67.75 & 77.69 & 61.68 & 58.74 & 66.47 \\
& & PatchTST & 64.13 & 69.63 & 57.83 & 53.70 & 61.32 \\
& & OneFitsAll & 63.04 & 69.96 & 56.27 & 50.72 & 59.99 \\
& & \myformer & \textbf{71.01} & \textbf{70.90} & \textbf{68.32} & \textbf{68.64} & \textbf{69.72} \\

\hline
\bottomrule
\end{tabular}
\end{small}
\end{threeparttable}
}
\end{table*}
\begin{table*}[htbp]
\caption{Comparison (\textbf{Part-2}) of the complete performance on \textbf{few-shot classification} task. Where \emph{SourceDomain}$\rightarrow$\emph{TragetDomain} indicates that the model is first pre-trained in the \emph{SourceDomain} train set, subsequently, the parameters are fine-tuned in partial (5\%/10\%) samples of the \emph{TargetDomain} train set and finally predicted on the \emph{TargetDomain} test set.}\label{tab:classification_few_2}
\vskip 0.1in
\centering
\resizebox{0.9\columnwidth}{!}{
\begin{threeparttable}
\begin{small}
\renewcommand{\multirowsetup}{\centering}
\setlength{\tabcolsep}{7pt}
\begin{tabular}{c|c|c|cccc|c}
\toprule
\hline

\multicolumn{2}{c}{Scenarios} & Models & Accuracy (\%) & Precision (\%) & Recall (\%) & F1 (\%) & Avg (\%) \\

\midrule
\multirow{16}{*}[-3ex]{\rotatebox{90}{\shortstack{SonyAIBORobotSurface2 \\ $\downarrow$ \\ SonyAIBORobotSurface1}}}

& \multirow{4}{*}[0ex]{\shortstack{Full Data}}
& MICN & 74.38 & 80.85 & 77.45 & 74.07 & 76.69 \\
& & TimesNet & 76.37 & 76.62 & 77.14 & 76.30 & 76.61 \\
& & PatchTST & 69.55 & 76.90 & 72.89 & 69.00 & 72.08 \\
& & OneFitsAll & \textbf{79.37} & \textbf{82.30} & \textbf{81.49} & \textbf{79.34} & \textbf{80.62} \\
\cmidrule(lr){2-8}

& \multirow{6}{*}[-1ex]{\shortstack{Few-shot \\ (10\%)}}
& Random init. & 58.51 & 54.26 & 52.51 & 40.88 & 46.53 \\
\cmidrule(lr){3-8}
& & MICN & 62.40 & 61.80 & 61.92 & 61.83 & 61.99 \\
& & TimesNet & 64.23 & 69.14 & 67.07 & 63.82 & 66.07 \\
& & PatchTST & 63.56 & 65.35 & 65.15 & 63.55 & 64.40 \\
& & OneFitsAll & 64.56 & 63.85 & 63.86 & 63.85 & 64.03 \\
& & \myformer & \textbf{91.51} & \textbf{91.36} & \textbf{91.32} & \textbf{91.34} & \textbf{91.38} \\
\cmidrule(lr){2-8}

& \multirow{6}{*}[-1ex]{\shortstack{Few-shot \\ (5\%)}}
& Random init. & 47.09 & 28.55 & 42.45 & 33.61 & 37.94 \\
\cmidrule(lr){3-8}
& & MICN & 56.91 & 52.55 & 50.38 & 39.85 & 49.92 \\
& & TimesNet & 60.23 & 68.97 & 54.02 & 45.82 & 57.26 \\
& & PatchTST & 57.07 & 28.54 & 50.00 & 36.33 & 42.98 \\
& & OneFitsAll & 57.07 & 28.54 & 50.00 & 36.33 & 42.98 \\
& & \myformer & \textbf{85.03} & \textbf{86.45} & \textbf{87.90} & \textbf{85.01} & \textbf{86.09} \\





\midrule
\multirow{16}{*}[-3ex]{\rotatebox{90}{\shortstack{HouseTwenty\\ $\downarrow$ \\ GunPointMaleVersusFemale}}}

& \multirow{4}{*}[0ex]{\shortstack{Full Data}}
& MICN & 98.73 & 98.77 & 98.70 & 98.73 & 98.73 \\
& & TimesNet & \textbf{99.05} & \textbf{99.07} & \textbf{99.03} & \textbf{99.05} & \textbf{99.05} \\
& & PatchTST & 97.47 & 97.44 & 97.49 & 97.46 & 97.47 \\
& & OneFitsAll & 94.62 & 95.00 & 94.37 & 93.24 & 94.31 \\
\cmidrule(lr){2-8}

& \multirow{6}{*}[-1ex]{\shortstack{Few-shot \\ (10\%)}}
& Random init. & 87.66 & 89.45 & 88.22 & 87.60 & 88.23 \\
\cmidrule(lr){3-8}
& & MICN & 75.63 & 78.43 & 76.39 & 75.33 & 76.44 \\
& & TimesNet & 81.96 & 82.73 & 81.55 & 81.68 & 81.98 \\
& & PatchTST & 96.84 & 96.80 & 96.89 & 96.83 & 96.84 \\
& & OneFitsAll & 76.27 & 78.42 & 76.93 & 76.06 & 76.92 \\
& & \myformer & \textbf{98.10} & \textbf{98.26} & \textbf{98.00} & \textbf{98.09} & \textbf{98.11} \\
\cmidrule(lr){2-8}

& \multirow{6}{*}[-1ex]{\shortstack{Few-shot \\ (5\%)}}
& Random init. & 66.77 & 69.86 & 67.67 & 66.08 & 67.59 \\
\cmidrule(lr){3-8}
& & MICN & 69.30 & 72.46 & 70.17 & 68.72 & 70.16 \\
& & TimesNet & 52.53 & 51.64 & 51.16 & 47.92 & 50.81 \\
& & PatchTST & 84.18 & 86.17 & 84.78 & 84.09 & 84.81 \\
& & OneFitsAll & 68.99 & 73.31 & 70.00 & 68.12 & 70.11 \\
& & \myformer & \textbf{85.76} & \textbf{85.79} & \textbf{85.64} & \textbf{85.70} & \textbf{85.72} \\

\midrule
\multirow{16}{*}[-3ex]{\rotatebox{90}{\shortstack{ProximalPhalanxOutlineCorrect \\ $\downarrow$ \\ HouseTwenty}}}

& \multirow{4}{*}[0ex]{\shortstack{Full Data}}
& MICN & 75.63 & 75.00 & 75.13 & 75.06 & 75.21 \\
& & TimesNet & 63.87 & 64.97 & 58.93 & 56.84 & 61.15 \\
& & PatchTST & \textbf{85.71} & \textbf{85.29} & \textbf{85.48} & \textbf{85.38} & \textbf{85.47} \\
& & OneFitsAll & 64.71 & 63.82 & 61.58 & 61.36 & 62.78 \\
\cmidrule(lr){2-8}

& \multirow{6}{*}[-1ex]{\shortstack{Few-shot \\ (10\%)}}
& Random init. & 61.34 & 67.68 & 54.55 & 47.57 & 57.78 \\
\cmidrule(lr){3-8}
& & MICN & 59.66 & 57.71 & 56.68 & 56.30 & 57.59 \\
& & TimesNet & 66.39 & 65.42 & 65.23 & 65.31 & 65.59 \\
& & PatchTST & 63.03 & 61.82 & 61.51 & 61.59 & 61.98 \\
& & OneFitsAll & 60.50 & 69.82 & 53.28 & 44.43 & 57.01 \\
& & \myformer & \textbf{76.47} & \textbf{82.18} & \textbf{72.55} & \textbf{72.94} & \textbf{76.03} \\
\cmidrule(lr){2-8}

& \multirow{6}{*}[-1ex]{\shortstack{Few-shot \\ (5\%)}}
& Random init. & 54.62 & 43.44 & 47.93 & 39.83 & 46.45 \\
\cmidrule(lr){3-8}
& & MICN & 47.06 & 55.12 & 52.70 & 42.93 & 49.45 \\
& & TimesNet & 42.02 & 21.01 & 50.00 & 29.59 & 35.65 \\
& & PatchTST & 42.02 & 21.01 & 50.00 & 29.59 & 35.65 \\
& & OneFitsAll & 42.02 & 21.01 & 50.00 & 29.59 & 35.65 \\
& & \myformer & \textbf{73.95} & \textbf{77.92} & \textbf{70.10} & \textbf{70.24} & \textbf{73.05} \\

\hline
\bottomrule
\end{tabular}
\end{small}
\end{threeparttable}
}
\end{table*}
\begin{table*}[htbp]
\caption{Comparison (\textbf{Part-3}) of the complete performance on \textbf{few-shot classification} task. Where \emph{SourceDomain}$\rightarrow$\emph{TragetDomain} indicates that the model is first pre-trained in the \emph{SourceDomain} train set, subsequently, the parameters are fine-tuned in partial (5\%/10\%) samples of the \emph{TargetDomain} train set and finally predicted on the \emph{TargetDomain} test set.}\label{tab:classification_few_3}
\vskip 0.1in
\centering
\resizebox{0.9\columnwidth}{!}{
\begin{threeparttable}
\begin{small}
\renewcommand{\multirowsetup}{\centering}
\setlength{\tabcolsep}{7pt}
\begin{tabular}{c|c|c|cccc|c}
\toprule
\hline

\multicolumn{2}{c}{Scenarios} & Models & Accuracy (\%) & Precision (\%) & Recall (\%) & F1 (\%) & Avg (\%) \\

\midrule
\multirow{16}{*}[-3ex]{\rotatebox{90}{\shortstack{PigArtPressure\\ $\downarrow$ \\ InsectEPGRegularTrain}}}

& \multirow{4}{*}[0ex]{\shortstack{Full Data}}
& MICN & \textbf{86.78} & 74.94 & \textbf{83.13} & \textbf{79.83} & \textbf{81.17} \\
& & TimesNet & 78.10 & 76.27 & 74.82 & 71.85 & 75.26 \\
& & PatchTST & 83.13 & \textbf{84.74} & 71.74 & 72.78 & 78.10 \\
& & OneFitsAll & 83.03 & 84.46 & 76.61 & 73.89 & 79.50 \\
\cmidrule(lr){2-8}

& \multirow{6}{*}[-1ex]{\shortstack{Few-shot \\ (10\%)}}
& Random init. & 47.39 & 15.80 & 33.33 & 21.44 & 29.49 \\
\cmidrule(lr){3-8}
& & MICN & 35.74 & 11.91 & 33.33 & 17.55 & 24.63 \\
& & TimesNet & 83.13 & 55.98 & 66.67 & 60.30 & 66.52 \\
& & PatchTST & 48.19 & 48.67 & 45.81 & 39.85 & 45.63 \\
& & OneFitsAll & 83.13 & 57.92 & 66.67 & 61.63 & 67.34 \\
& & \myformer & \textbf{93.17} & \textbf{93.74} & \textbf{87.53} & \textbf{89.68} & \textbf{91.03} \\
\cmidrule(lr){2-8}

& \multirow{6}{*}[-1ex]{\shortstack{Few-shot \\ (5\%)}}
& Random init. & 35.74 & 11.91 & 33.33 & 17.55 & 24.63 \\
\cmidrule(lr){3-8}
& & MICN & 35.74 & 11.91 & 33.33 & 17.55 & 24.63 \\
& & TimesNet & 35.74 & 15.86 & 33.33 & 21.50 & 26.61 \\
& & PatchTST & 37.35 & 20.95 & 37.35 & 24.50 & 30.04 \\
& & OneFitsAll & 35.74 & 11.91 & 33.33 & 17.55 & 24.63 \\
& & \myformer & \textbf{83.85} & \textbf{84.37} & \textbf{78.78} & \textbf{80.71} & \textbf{81.93} \\

\midrule
\multirow{16}{*}[-3ex]{\rotatebox{90}{\shortstack{SonyAIBORobotSurface1 \\ $\downarrow$ \\ InsectEPGSmallTrain}}}

& \multirow{4}{*}[0ex]{\shortstack{Full Data}}
& MICN & 75.17 & 75.57 & 70.97 & 72.70 & 73.60 \\
& & TimesNet & 83.03 & 84.46 & 76.61 & 73.89 & 79.50 \\
& & PatchTST & \textbf{83.53} & \textbf{83.97} & \textbf{78.86} & \textbf{80.78} & \textbf{81.78} \\
& & OneFitsAll & 75.17 & 75.57 & 70.97 & 72.70 & 73.60 \\
\cmidrule(lr){2-8}

& \multirow{6}{*}[-1ex]{\shortstack{Few-shot \\ (10\%)}}
& Random init. & 47.39 & 15.80 & 33.33 & 21.44 & 29.49 \\
\cmidrule(lr){3-8}
& & MICN & 35.74 & 11.91 & 33.33 & 17.55 & 24.63 \\
& & TimesNet & 38.55 & 55.98 & 35.31 & 30.70 & 40.14 \\
& & PatchTST & 39.76 & 46.58 & 39.93 & 31.79 & 39.51 \\
& & OneFitsAll & 35.74 & 11.91 & 33.33 & 17.55 & 24.63 \\
& & \myformer & \textbf{83.13} & \textbf{57.92} & \textbf{66.67} & \textbf{61.63} & \textbf{67.34} \\
\cmidrule(lr){2-8}

& \multirow{6}{*}[-1ex]{\shortstack{Few-shot \\ (5\%)}}
& Random init. & 47.39 & 15.80 & 33.33 & 21.44 & 29.49 \\
\cmidrule(lr){3-8}
& & MICN & 35.74 & 11.91 & 33.33 & 17.55 & 24.63 \\
& & TimesNet & 26.98 & 39.18 & 24.72 & 21.49 & 28.09 \\
& & PatchTST & 27.71 & 45.16 & 36.68 & 26.44 & 33.99 \\
& & OneFitsAll & 35.74 & 11.91 & 33.33 & 17.55 & 24.63 \\
& & \myformer & \textbf{83.13} & \textbf{55.16} & \textbf{66.67} & \textbf{60.32} & \textbf{66.32} \\

\midrule
\multirow{16}{*}[-3ex]{\rotatebox{90}{\shortstack{Earthquakes \\ $\downarrow$ \\ ItalyPowerDemand}}}

& \multirow{4}{*}[0ex]{\shortstack{Full Data}}
& MICN & 95.63 & 95.63 & 95.63 & 95.63 & 95.63 \\
& & TimesNet & \textbf{97.18} & \textbf{97.18} & \textbf{97.18} & \textbf{97.18} & \textbf{97.18} \\
& & PatchTST & 96.60 & 96.62 & 96.60 & 96.60 & 96.61 \\
& & OneFitsAll & 96.99 & 96.99 & 96.99 & 96.99 & 96.99 \\
\cmidrule(lr){2-8}

& \multirow{6}{*}[-1ex]{\shortstack{Few-shot \\ (10\%)}}
& Random init. & 78.92 & 79.58 & 78.90 & 78.87 & 79.07 \\
\cmidrule(lr){3-8}
& & MICN & 90.28 & 91.24 & 90.26 & 90.22 & 90.50 \\
& & TimesNet & 86.59 & 87.02 & 86.60 & 86.55 & 86.69 \\
& & PatchTST & 87.37 & 88.40 & 87.34 & 87.28 & 87.59 \\
& & OneFitsAll & 94.66 & 94.75 & 94.65 & 94.65 & 94.68 \\
& & \myformer & \textbf{96.31} & \textbf{96.34} & \textbf{96.30} & \textbf{96.31} & \textbf{96.31} \\
\cmidrule(lr){2-8}

& \multirow{6}{*}[-1ex]{\shortstack{Few-shot \\ (5\%)}}
& Random init. & 70.86 & 72.47 & 70.89 & 70.63 & 71.21 \\
\cmidrule(lr){3-8}
& & MICN & 70.94 & 77.75 & 71.01 & 69.09 & 72.19 \\
& & TimesNet & 73.28 & 77.08 & 73.33 & 72.33 & 74.01 \\
& & PatchTST & 60.93 & 64.96 & 61.01 & 58.20 & 61.27 \\
& & OneFitsAll & 83.19 & 85.25 & 83.22 & 82.95 & 83.65 \\
& & \myformer & \textbf{94.85} & \textbf{94.85} & \textbf{94.85} & \textbf{94.85} & \textbf{94.85} \\

\hline
\bottomrule
\end{tabular}
\end{small}
\end{threeparttable}
}
\end{table*}
\begin{table*}[htbp]
\caption{Comparison (\textbf{Part-4}) of the complete performance on \textbf{few-shot classification} task. Where \emph{SourceDomain}$\rightarrow$\emph{TragetDomain} indicates that the model is first pre-trained in the \emph{SourceDomain} train set, subsequently, the parameters are fine-tuned in partial (5\%/10\%) samples of the \emph{TargetDomain} train set and finally predicted on the \emph{TargetDomain} test set.}\label{tab:classification_few_4}
\vskip 0.1in
\centering
\resizebox{0.9\columnwidth}{!}{
\begin{threeparttable}
\begin{small}
\renewcommand{\multirowsetup}{\centering}
\setlength{\tabcolsep}{7pt}
\begin{tabular}{c|c|c|cccc|c}
\toprule
\hline

\multicolumn{2}{c}{Scenarios} & Models & Accuracy (\%) & Precision (\%) & Recall (\%) & F1 (\%) & Avg (\%) \\

\midrule
\multirow{16}{*}[-3ex]{\rotatebox{90}{\shortstack{DistalPhalanxOutlineAgeGroup \\ $\downarrow$ \\ MiddlePhalanxOutlineAgeGroup}}}

& \multirow{4}{*}[0ex]{\shortstack{Full Data}}
& MICN & 64.94 & 66.29 & 48.22 & 49.50 & 57.24 \\
& & TimesNet & \textbf{66.23} & \textbf{75.40} & 48.98 & 51.02 & \textbf{60.41} \\
& & PatchTST & 65.58 & 72.43 & \textbf{49.12} & \textbf{51.37} & 59.62 \\
& & OneFitsAll & 64.29 & 72.26 & 47.59 & 49.46 & 58.40 \\
\cmidrule(lr){2-8}

& \multirow{6}{*}[-1ex]{\shortstack{Few-shot \\ (10\%)}}
& Random init. & 19.22 & 7.21 & 30.13 & 11.62 & 17.61 \\
\cmidrule(lr){3-8}
& & MICN & 57.14 & 51.64 & \textbf{49.92} & \textbf{50.35} & 52.26 \\
& & TimesNet & 48.70 & 48.76 & 49.42 & 46.54 & 48.36 \\
& & PatchTST & 45.45 & 46.03 & 42.77 & 40.81 & 43.76 \\
& & OneFitsAll & 47.40 & 28.88 & 39.98 & 33.06 & 37.33 \\
& & \myformer & \textbf{60.39} & \textbf{59.20} & 45.05 & 46.23 & \textbf{52.72} \\
\cmidrule(lr){2-8}

& \multirow{6}{*}[-1ex]{\shortstack{Few-shot \\ (5\%)}}
& Random init. & 18.83 & 6.28 & 33.33 & 10.56 & 17.25 \\
\cmidrule(lr){3-8}
& & MICN & 46.10 & 42.96 & 45.04 & 42.89 & 44.25 \\
& & TimesNet & 47.40 & 44.67 & 47.32 & 45.21 & 46.15 \\
& & PatchTST & 40.26 & 43.72 & 47.03 & 40.26 & 42.82 \\
& & OneFitsAll & 35.71 & 34.12 & 37.32 & 32.00 & 34.79 \\
& & \myformer & \textbf{51.95} & \textbf{49.36} & \textbf{48.38} & \textbf{46.92} & \textbf{49.15} \\

\midrule
\multirow{16}{*}[-3ex]{\rotatebox{90}{\shortstack{BeetleFly\\ $\downarrow$ \\ PowerCons}}}

& \multirow{4}{*}[0ex]{\shortstack{Full Data}}
& MICN & 95.56 & 95.65 & 95.56 & 95.55 & 95.58 \\
& & TimesNet & \textbf{100.00} & \textbf{100.00} & \textbf{100.00} & \textbf{100.00} & \textbf{100.00} \\
& & PatchTST & 98.89 & 98.91 & 98.89 & 98.89 & 98.89 \\
& & OneFitsAll & 98.33 & 98.39 & 98.33 & 98.33 & 98.34 \\
\cmidrule(lr){2-8}

& \multirow{6}{*}[-1ex]{\shortstack{Few-shot \\ (10\%)}}
& Random init. & 50.00 & 25.00 & 50.00 & 33.33 & 39.58 \\
\cmidrule(lr){3-8}
& & MICN & 83.89 & 84.10 & 83.89 & 83.86 & 83.94 \\
& & TimesNet & 83.89 & 84.40 & 83.89 & 83.83 & 84.00 \\
& & PatchTST & 90.00 & 90.72 & 90.00 & 89.96 & 90.17 \\
& & OneFitsAll & 97.23 & 97.23 & 97.23 & 97.23 & 97.23 \\
& & \myformer & \textbf{98.89} & \textbf{98.89} & \textbf{98.89} & \textbf{98.89} & \textbf{98.89} \\
\cmidrule(lr){2-8}

& \multirow{6}{*}[-1ex]{\shortstack{Few-shot \\ (5\%)}}
& Random init. & 50.00 & 25.00 & 50.00 & 33.33 & 39.58 \\
\cmidrule(lr){3-8}
& & MICN & 62.78 & 66.56 & 62.78 & 60.53 & 63.16 \\
& & TimesNet & 78.89 & 79.12 & 78.89 & 78.85 & 78.94 \\
& & PatchTST & 82.78 & 83.11 & 82.78 & 82.73 & 82.84 \\
& & OneFitsAll & 90.00 & 90.18 & 90.00 & 89.99 & 90.04 \\
& & \myformer & \textbf{95.00} & \textbf{95.01} & \textbf{95.00} & \textbf{95.00} & \textbf{95.00} \\

\midrule
\multirow{16}{*}[-3ex]{\rotatebox{90}{\shortstack{EOGHorizontalSignal \\ $\downarrow$ \\ ProximalPhalanxOutlineAgeGroup}}}

& \multirow{4}{*}[0ex]{\shortstack{Full Data}}
& MICN & 86.34 & 77.94 & 80.22 & 78.94 & 80.86 \\
& & TimesNet & 86.34 & 80.16 & 69.15 & 71.11 & 76.69 \\
& & PatchTST & \textbf{89.27} & \textbf{86.13} & \textbf{82.49} & \textbf{84.07} & \textbf{85.49} \\
& & OneFitsAll & 88.78 & 84.33 & 82.11 & 83.12 & 84.59 \\
\cmidrule(lr){2-8}

& \multirow{6}{*}[-1ex]{\shortstack{Few-shot \\ (10\%)}}
& Random init. & 41.46 & 13.82 & 28.33 & 18.58 & 25.55 \\
\cmidrule(lr){3-8}
& & MICN & 85.85 & 77.67 & 67.19 & 68.46 & 74.79 \\
& & TimesNet & 85.85 & 76.64 & 70.35 & 72.08 & 76.23 \\
& & PatchTST & 87.80 & 81.74 & 79.77 & 80.67 & 82.50 \\
& & OneFitsAll & 87.32 & 85.20 & 69.90 & 71.62 & 78.51 \\
& & \myformer & \textbf{89.27} & \textbf{91.40} & \textbf{84.07} & \textbf{84.61} & \textbf{87.34} \\
\cmidrule(lr){2-8}

& \multirow{6}{*}[-1ex]{\shortstack{Few-shot \\ (5\%)}}
& Random init. & 36.59 & 12.19 & 24.96 & 16.39 & 22.54 \\
\cmidrule(lr){3-8}
& & MICN & 77.07 & 68.90 & 76.19 & 68.56 & 72.68 \\
& & TimesNet & 85.85 & \textbf{90.68} & 64.02 & 63.18 & 75.93 \\
& & PatchTST & 87.32 & 80.31 & 79.40 & 79.84 & 81.72 \\
& & OneFitsAll & 84.39 & 56.18 & 61.30 & 58.62 & 65.12 \\
& & \myformer & \textbf{87.80} & 80.97 & \textbf{84.52} & \textbf{82.48} & \textbf{83.94} \\

\hline
\bottomrule
\end{tabular}
\end{small}
\end{threeparttable}
}
\end{table*}

To validate the performance of the \textbf{wavebook} strategy on the classification task, complete experimental results contain the three most representative approaches (OneFitsAll, PatchTST, and TimesNet) and proposed \myformer on all 35 classification datasets of UCR archive \cite{UCRArchive2018}, as shown in Table \ref{tab:classification_full}.

In the main body of the experiments on \textbf{few-shot classification} task, \emph{Scenario-i: Source-i}$\rightarrow$\emph{Traget-i} ($i\in[0,..,7]$) are utilized to indicates the cross-domain adaptation scenarios between different TS domains, and the details of scenarios are summarized in Table \ref{tab:seven_scenarios}. Besides, the completed \textbf{few-shot classification} on cross-domain adaptation is demonstrated in Table \ref{tab:classification_few_1}, \ref{tab:classification_few_2}, \ref{tab:classification_few_3}, and \ref{tab:classification_few_4}, which includes four metrics: accuracy (Acc), precision (Pre), recall (Rec), and F1 score (F1).

\subsection{Ablation}\label{appendix:exp_4}
As shown in table \ref{tab:ablation}, when the \emph{wave as token} strategy is removed from the Transformer-based model (\textbf{w/o TP}), there demonstrates a significant performance degradation. Similar results are found for both ETTh1 and ETTm1 datasets, which proves the effectiveness of our approach.

\begin{table*}[htpb]
  \caption{Comparison of the ablation experiment, where \emph{w/o TP} indicates the model without the proposed \textbf{\emph{wave as token}} strategy. To guarantee fairness and effectiveness, \emph{w/o TP} and \emph{\myformer} have the same model structure and parameter size.}\label{tab:ablation}
  \vspace{5pt}
  \centering
  \resizebox{0.8\columnwidth}{!}{
  \begin{small}
  \renewcommand{\multirowsetup}{\centering}
  \setlength{\tabcolsep}{1.pt}
  \renewcommand\arraystretch{1.}
  \begin{tabular}{cc|cccc|cccc|cccc|cccc}
    \hline
    \hline
    \multicolumn{2}{c|}{\multirow{3}{*}{\scalebox{0.8}{Variant}}} & 
    \multicolumn{4}{c|}{\multirow{1}{*}{\rotatebox{0}{\scalebox{0.8}{full datal}}}} & 
    \multicolumn{4}{c|}{\multirow{1}{*}{\rotatebox{0}{\scalebox{0.8}{few shot}}}} & 
    \multicolumn{4}{c|}{\multirow{1}{*}{\rotatebox{0}{\scalebox{0.8}{zero shot (single)}}}} & 
    \multicolumn{4}{c}{\multirow{1}{*}{\rotatebox{0}{\scalebox{0.8}{zero shot (multi)}}}} \\
    \cline{3-18}
    
    & &
    \multicolumn{2}{c}{\multirow{1}{*}{\rotatebox{0}{\scalebox{0.8}{ETTh1}}}} & 
    \multicolumn{2}{c|}{\multirow{1}{*}{\rotatebox{0}{\scalebox{0.8}{ETTm1}}}} & 
    \multicolumn{2}{c}{\multirow{1}{*}{\rotatebox{0}{\scalebox{0.8}{ETTh1}}}} & 
    \multicolumn{2}{c|}{\multirow{1}{*}{\rotatebox{0}{\scalebox{0.8}{ETTm1}}}} & 
    \multicolumn{2}{c}{\multirow{1}{*}{\rotatebox{0}{\scalebox{0.8}{ETTh1}}}} & 
    \multicolumn{2}{c|}{\multirow{1}{*}{\rotatebox{0}{\scalebox{0.8}{ETTm1}}}} & 
    \multicolumn{2}{c}{\multirow{1}{*}{\rotatebox{0}{\scalebox{0.8}{ETTh1}}}} & 
    \multicolumn{2}{c}{\multirow{1}{*}{\rotatebox{0}{\scalebox{0.8}{ETTm1}}}} \\
    \cline{3-18}
    
    & &
    \scalebox{0.78}{MSE} & \scalebox{0.78}{MAE} & 
    \scalebox{0.78}{MSE} & \scalebox{0.78}{MAE} & 
    \scalebox{0.78}{MSE} & \scalebox{0.78}{MAE} & 
    \scalebox{0.78}{MSE} & \scalebox{0.78}{MAE} & 
    \scalebox{0.78}{MSE} & \scalebox{0.78}{MAE} & 
    \scalebox{0.78}{MSE} & \scalebox{0.78}{MAE} & 
    \scalebox{0.78}{MSE} & \scalebox{0.78}{MAE} & 
    \scalebox{0.78}{MSE} & \scalebox{0.78}{MAE} \\

    \hline

    \multirow{5}{*}{\rotatebox{90}{\scalebox{0.95}{w/o TP}}} & 
    \scalebox{0.78}{96} &
    {\scalebox{0.78}{0.543}} & {\scalebox{0.78}{0.515}} &
    {\scalebox{0.78}{0.455}} & {\scalebox{0.78}{0.446}} &
    {\scalebox{0.78}{0.613}} & {\scalebox{0.78}{0.560}} &
    {\scalebox{0.78}{0.465}} & {\scalebox{0.78}{0.455}} &
    {\scalebox{0.78}{0.691}} & {\scalebox{0.78}{0.615}} &
    {\scalebox{0.78}{0.555}} & {\scalebox{0.78}{0.522}} &
    {\scalebox{0.78}{0.665}} & {\scalebox{0.78}{0.593}} &
    {\scalebox{0.78}{0.582}} & {\scalebox{0.78}{0.536}} \\
    
    & \scalebox{0.78}{192} &
    {\scalebox{0.78}{0.577}} & {\scalebox{0.78}{0.534}} &
    {\scalebox{0.78}{0.491}} & {\scalebox{0.78}{0.468}} &
    {\scalebox{0.78}{0.639}} & {\scalebox{0.78}{0.573}} &
    {\scalebox{0.78}{0.501}} & {\scalebox{0.78}{0.477}} &
    {\scalebox{0.78}{0.703}} & {\scalebox{0.78}{0.622}} &
    {\scalebox{0.78}{0.588}} & {\scalebox{0.78}{0.541}} &
    {\scalebox{0.78}{0.677}} & {\scalebox{0.78}{0.602}} &
    {\scalebox{0.78}{0.608}} & {\scalebox{0.78}{0.553}} \\
    
    & \scalebox{0.78}{336} &
    {\scalebox{0.78}{0.597}} & {\scalebox{0.78}{0.546}} &
    {\scalebox{0.78}{0.525}} & {\scalebox{0.78}{0.487}} &
    {\scalebox{0.78}{0.660}} & {\scalebox{0.78}{0.587}} &
    {\scalebox{0.78}{0.540}} & {\scalebox{0.78}{0.499}} &
    {\scalebox{0.78}{0.711}} & {\scalebox{0.78}{0.629}} &
    {\scalebox{0.78}{0.608}} & {\scalebox{0.78}{0.555}} &
    {\scalebox{0.78}{0.682}} & {\scalebox{0.78}{0.608}} &
    {\scalebox{0.78}{0.631}} & {\scalebox{0.78}{0.569}} \\
    
    & \scalebox{0.78}{720} &
    {\scalebox{0.78}{0.609}} & {\scalebox{0.78}{0.571}} &
    {\scalebox{0.78}{0.581}} & {\scalebox{0.78}{0.520}} &
    {\scalebox{0.78}{0.660}} & {\scalebox{0.78}{0.604}} &
    {\scalebox{0.78}{0.588}} & {\scalebox{0.78}{0.524}} &
    {\scalebox{0.78}{0.710}} & {\scalebox{0.78}{0.644}} &
    {\scalebox{0.78}{0.646}} & {\scalebox{0.78}{0.577}} &
    {\scalebox{0.78}{0.683}} & {\scalebox{0.78}{0.619}} &
    {\scalebox{0.78}{0.674}} & {\scalebox{0.78}{0.596}} \\
    
    \cline{2-18}
    
    & \scalebox{0.78}{Avg} &
    {\scalebox{0.78}{0.582}} & {\scalebox{0.78}{0.542}} &
    {\scalebox{0.78}{0.513}} & {\scalebox{0.78}{0.480}} &
    {\scalebox{0.78}{0.643}} & {\scalebox{0.78}{0.581}} &
    {\scalebox{0.78}{0.524}} & {\scalebox{0.78}{0.489}} &
    {\scalebox{0.78}{0.704}} & {\scalebox{0.78}{0.628}} &
    {\scalebox{0.78}{0.599}} & {\scalebox{0.78}{0.548}} &
    {\scalebox{0.78}{0.677}} & {\scalebox{0.78}{0.606}} &
    {\scalebox{0.78}{0.623}} & {\scalebox{0.78}{0.564}} \\

    \hline
    
    \multirow{5}{*}{\rotatebox{90}{\scalebox{0.95}{\myformer}}} & 
    \scalebox{0.78}{96} &
    {\scalebox{0.78}{0.363}} & {\scalebox{0.78}{0.388}} &
    {\scalebox{0.78}{0.324}} & {\scalebox{0.78}{0.342}} &
    {\scalebox{0.78}{0.432}} & {\scalebox{0.78}{0.409}} &
    {\scalebox{0.78}{0.337}} & {\scalebox{0.78}{0.350}} &
    {\scalebox{0.78}{0.466}} & {\scalebox{0.78}{0.460}} &
    {\scalebox{0.78}{0.400}} & {\scalebox{0.78}{0.407}} &
    {\scalebox{0.78}{0.438}} & {\scalebox{0.78}{0.419}} &
    {\scalebox{0.78}{0.379}} & {\scalebox{0.78}{0.398}} \\
    
    & \scalebox{0.78}{192} &
    {\scalebox{0.78}{0.386}} & {\scalebox{0.78}{0.409}} &
    {\scalebox{0.78}{0.349}} & {\scalebox{0.78}{0.357}} &
    {\scalebox{0.78}{0.456}} & {\scalebox{0.78}{0.428}} &
    {\scalebox{0.78}{0.364}} & {\scalebox{0.78}{0.363}} &
    {\scalebox{0.78}{0.505}} & {\scalebox{0.78}{0.483}} &
    {\scalebox{0.78}{0.415}} & {\scalebox{0.78}{0.417}} &
    {\scalebox{0.78}{0.449}} & {\scalebox{0.78}{0.439}} &
    {\scalebox{0.78}{0.389}} & {\scalebox{0.78}{0.404}} \\
    
    & \scalebox{0.78}{336} &
    {\scalebox{0.78}{0.414}} & {\scalebox{0.78}{0.413}} &
    {\scalebox{0.78}{0.385}} & {\scalebox{0.78}{0.375}} &
    {\scalebox{0.78}{0.489}} & {\scalebox{0.78}{0.453}} &
    {\scalebox{0.78}{0.389}} & {\scalebox{0.78}{0.388}} &
    {\scalebox{0.78}{0.535}} & {\scalebox{0.78}{0.501}} &
    {\scalebox{0.78}{0.442}} & {\scalebox{0.78}{0.434}} &
    {\scalebox{0.78}{0.471}} & {\scalebox{0.78}{0.467}} &
    {\scalebox{0.78}{0.414}} & {\scalebox{0.78}{0.415}} \\
    
    & \scalebox{0.78}{720} &
    {\scalebox{0.78}{0.467}} & {\scalebox{0.78}{0.460}} &
    {\scalebox{0.78}{0.412}} & {\scalebox{0.78}{0.402}} &
    {\scalebox{0.78}{0.531}} & {\scalebox{0.78}{0.502}} &
    {\scalebox{0.78}{0.441}} & {\scalebox{0.78}{0.417}} &
    {\scalebox{0.78}{0.543}} & {\scalebox{0.78}{0.529}} &
    {\scalebox{0.78}{0.481}} & {\scalebox{0.78}{0.459}} &
    {\scalebox{0.78}{0.484}} & {\scalebox{0.78}{0.473}} &
    {\scalebox{0.78}{0.461}} & {\scalebox{0.78}{0.446}} \\
    
    \cline{2-18}
    
    & \scalebox{0.78}{Avg} &
    {\scalebox{0.78}{0.407}} & {\scalebox{0.78}{0.418}} &
    {\scalebox{0.78}{0.368}} & {\scalebox{0.78}{0.369}} &
    {\scalebox{0.78}{0.477}} & {\scalebox{0.78}{0.448}} &
    {\scalebox{0.78}{0.383}} & {\scalebox{0.78}{0.380}} &
    {\scalebox{0.78}{0.512}} & {\scalebox{0.78}{0.493}} &
    {\scalebox{0.78}{0.434}} & {\scalebox{0.78}{0.429}} &
    {\scalebox{0.78}{0.461}} & {\scalebox{0.78}{0.449}} &
    {\scalebox{0.78}{0.411}} & {\scalebox{0.78}{0.416}} \\

    
    
    
    \hline
    \hline
  \end{tabular}
  \end{small}
}
\end{table*}

\end{document}